\documentclass[twoside]{article}

\usepackage[accepted]{aistats2026}
\usepackage[utf8]{inputenc} 
\usepackage[T1]{fontenc}    
\usepackage{hyperref}       
\usepackage{booktabs}       
\usepackage{amsfonts}       
\usepackage{nicefrac}       
\usepackage{microtype}      
\usepackage{xcolor}         
\usepackage{natbib}  
\usepackage{wrapfig}
\usepackage{graphicx}
\usepackage{subfig} 
\usepackage{adjustbox}   
\usepackage{stfloats}  
\usepackage{url}
\usepackage{amsmath}
\usepackage{amsthm}
\usepackage{subcaption}
\usepackage{algorithmic}
\usepackage{algorithm} 
\usepackage{bm}
\usepackage{array}
\usepackage{multirow}
\usepackage{booktabs} 
\usepackage{comment}
\usepackage{multirow}
\usepackage{amssymb} 
\usepackage{amsfonts} 
\usepackage{svg}
\usepackage{booktabs,adjustbox,makecell,siunitx}
\usepackage{xurl}   

\newtheorem{theorem}{Theorem}[section]
\newtheorem{proposition}[theorem]{Proposition}
\newtheorem{lemma}[theorem]{Lemma}
\newtheorem{corollary}[theorem]{Corollary}
\usepackage{enumitem}  

\newtheorem{assumption}[theorem]{Assumption}

\newcommand\norm[1]{\lVert#1\rVert}

\begin{document}

\renewcommand{\paragraph}[1]{\noindent\textbf{#1}}
\twocolumn[

\aistatstitle{Slithering Through Gaps: Capturing Discrete Isolated Modes via Logistic Bridging}

\aistatsauthor{ Pinaki Mohanty \And Ruqi Zhang }

\aistatsaddress{Department of Computer Science\\
College of Science \& College of Engineering\\
Purdue University, West Lafayette, IN, USA
\And
Department of Computer Science\\
College of Science \& College of Engineering\\
Purdue University, West Lafayette, IN, USA
} ]

\begin{abstract}
 High-dimensional and complex discrete distributions often exhibit multimodal behavior due to inherent discontinuities, posing significant challenges for sampling. Gradient-based discrete samplers, while effective, frequently become trapped in local modes when confronted with rugged or disconnected energy landscapes. This limits their ability to achieve adequate mixing and convergence in high-dimensional multimodal discrete spaces. To address these challenges, we propose \emph{Hyperbolic Secant-squared Gibbs-Sampling (HiSS)}, a novel family of sampling algorithms that integrates a \emph{Metropolis-within-Gibbs} framework to enhance mixing efficiency. HiSS leverages a logistic convolution kernel to couple the discrete sampling variable with the continuous auxiliary variable in a joint distribution. This design allows the auxiliary variable to encapsulate the true target distribution while facilitating easy transitions between distant and disconnected modes. We provide theoretical guarantees of convergence and demonstrate empirically that HiSS outperforms many popular alternatives on a wide variety of tasks, including Ising models, binary neural networks, and combinatorial optimization.
\end{abstract}

\section{Introduction}
\begin{figure}[t]
    \centering
    \includegraphics[width=0.38\textwidth, height=0.325\textwidth]{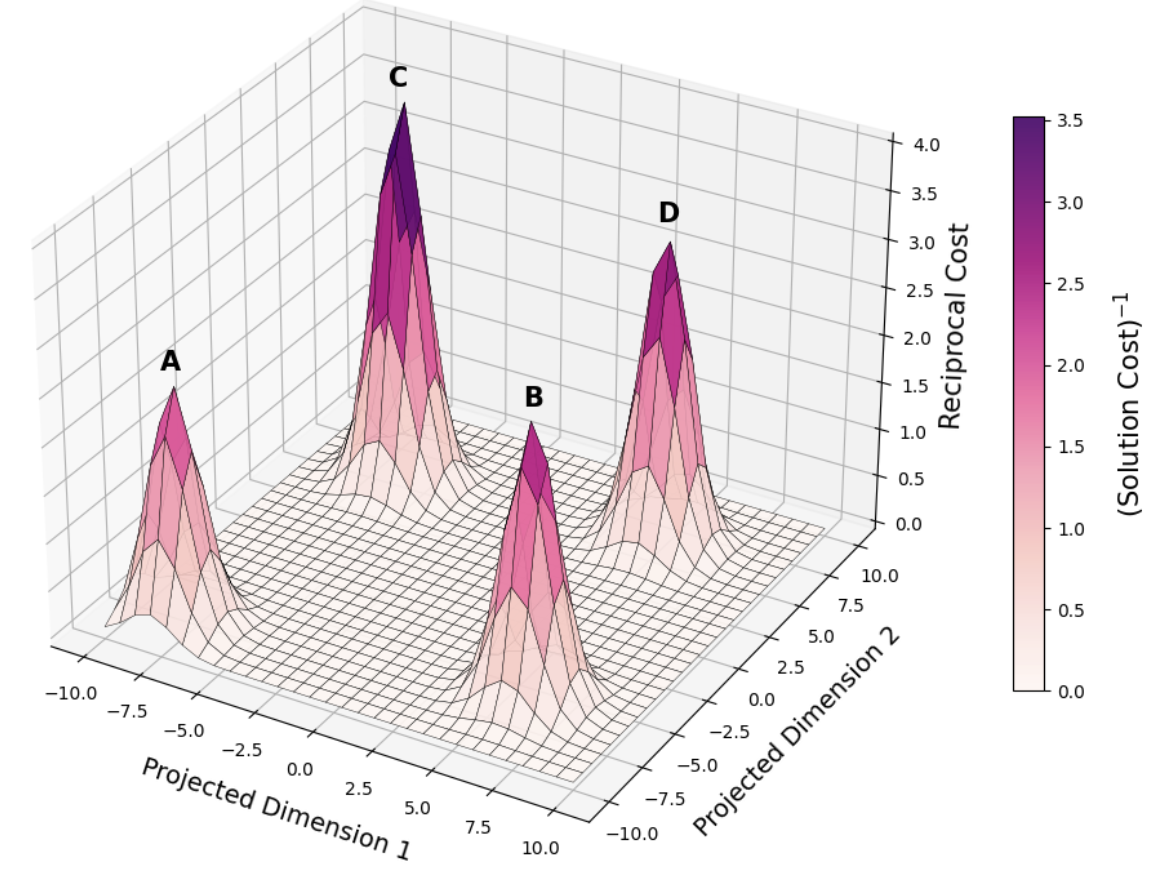}
    \caption{A visualization of the multimodal cost landscape in the Traveling Salesman Problem, showcasing solutions separated by low energy barriers.}
    \label{fig:motivation}
\end{figure}

Gradient-based sampling methods, such as Langevin and Hamiltonian Monte Carlo (HMC) ~\citep{roberts2002langevin, neal2011mcmc}, have achieved remarkable success in continuous spaces by using gradient information to guide transitions between states. However, these methods face significant limitations when sampling from multimodal distributions with low-density barriers separating distinct modes. In such scenarios, gradients often fail to provide the global context for effective navigation across disconnected regions~\citep{betancourt2017conceptual, livingstone2019geometric, pompe2020framework}. Techniques like cyclical step size~\citep{zhang2019cyclical}, PT(parallel tempering)~\citep{swendsen1986replica}, and flat-histogram approaches~\citep{berg1991multicanonical} have been developed to address this issue, improving mixing efficiency in continuous spaces by facilitating mode transitions.

In discrete spaces, these challenges are even more pronounced due to the inherent discontinuities in the landscape, which exacerbate ruggedness and multimodality. These challenges are particularly pervasive in applications such as text generation in natural language processing (NLP)~\citep{gu2018non, devlin2019bert, lewis2020bart}, protein coupling prediction~\citep{lapedes1999correlated}, low-precision neural networks~\citep{courbariaux2016binarized}, and combinatorial optimization~\citep{appelgate2006traveling}. For instance, in Figure \ref{fig:motivation}  mode A is thoroughly isolated making it difficult for gradient-based samplers to escape and discover lower-cost alternatives such as modes B, C, or D. 

While gradient-based discrete samplers, such as the Discrete Langevin Proposal (DLP)~\citep{zhang2022langevin} and Gibbs-With-Gradient (GWG)~\citep{grathwohl2021gwg}, have improved sampling efficiency, they face limitations similar to their continuous counterparts. These methods rely heavily on local gradient information, which renders them \emph{myopic} in their exploration, often failing to identify paths to promising yet disconnected regions in the landscape. Recent advancement using cyclical schedules~\citep{pynadath2024gradientbased} has attempted to address these limitations. However, the challenge of effectively sampling from \emph{disconnected} modes in discrete spaces separated by near-zero-probability regions remains unresolved.

In this paper, we propose \emph{\textbf{H}yperbol\textbf{i}c \textbf{S}ecant-squared Gibbs-\textbf{S}ampling} (HiSS), a novel Metropolis-within-Gibbs sampling framework. Our method introduces a continuous auxiliary variable, $\bm{\theta}_a$, alongside the discrete primary variable, $\bm{\theta}$, modeled under a joint distribution. $\bm{\theta}_a$ encapsulates the multimodal discrete target distribution, allowing $\bm{\theta}$ to escape local modes and transition to more promising states while ensuring detailed balance with respect to the target distribution. This makes exploring the state space effective and improves mixing between states. 
We summarize our contributions as follows:
\begin{itemize}[nosep,leftmargin=*]
\item We propose Hyperbolic Secant-Squared Gibbs Sampling (HiSS), a gradient-based algorithm for sampling multimodal discrete distributions. The core of HiSS is a logistic convolution kernel that bridges isolated modes, enabling the discrete sampler to efficiently traverse disconnected regions.

\item We prove that HiSS satisfies detailed balance with respect to the discrete target distribution, ensuring asymptotic correctness. Furthermore, we establish a non-asymptotic convergence guarantee for HiSS with Metropolis-Hastings correction in locally log-concave discrete distributions.

\item We present extensive experimental results, demonstrating the superiority of HiSS over standard gradient-based methods and other sampling techniques in multimodal settings. Our evaluations span diverse tasks, including Bernoulli distributions, Ising models,  binary Bayesian neural networks, and combinatorial optimization. In addition to convergence to the true target distributions, we also analyze runtime performance and sample diversity. To promote reproducibility, we release the code at \url{https://github.com/pinakirm/HiSS}.
\end{itemize}

\section{Related Works}
\paragraph{Gradient-Based Discrete Sampling.}
Gradient-based discrete sampling methods have significantly advanced, enhancing efficiency and applicability. Locally informed proposals~\citep{zanella2020informed}  proposals using probability ratios improved convergence, while later gradient-based approaches~\citep{grathwohl2021oops} to approximate the probability ratio, boosted efficiency . Subsequent works proposed various gradient-based techniques~\citep{rhodes2022enhanced, sun2021path, sun2022path, sun2023discrete, xiang2023efficient}. Discrete Langevin Proposal (DLP)~\citep{zhang2022langevin} extended the Langevin algorithm from continuous to discrete spaces, enabling parallel updates for all coordinates.

\paragraph{Multimodal Sampling in Continuous Spaces.} In continuous spaces, techniques like simulated tempering~\citep{marinari1992simulated}, cyclical step sizes~\citep{zhang2019cyclical}, parallel tempering~\citep{swendsen1986replica}, flat histograms~\citep{berg1991multicanonical,deng2020contour}, and Wolff algorithm~\citep{wolff1989collective} have been widely employed to facilitate mode transitions.

More recently, Diffusive Gibbs (DiGs)~\citep{2024diffusive} used Gaussian Convolution on continuous multimodal distributions. HiSS and DiGS enhance Markov chain mixing efficiency through auxiliary structures, but extending DiGS to discrete distributions is not straightforward: DiGS utilizes Gaussian proposals($q({x}|\tilde{x}^{i-1})$) centered around scaled noisy samples and a denoising posterior($p(x|\tilde{x}^{i-1})$). To tackle multimodality in discrete spaces, one may think of combating these problems using DLP, but because the former does not involve gradient information, DLP in its true form cannot be directly applied. HiSS accelerates convergence in discrete distributions using a joint-hybrid distribution (Lemma \ref{lemma:joint_posterior}), while DiGS’s joint distribution is purely continuous. In our hybrid setting, noising and denoising steps involve state-switching strategies, which are non-trivial and require careful design. HiSS introduces a logistic convolutional kernel, unlike DiGS’s Gaussian convolutional kernel. This structural shift is detailed in Section \ref{sec:app_eta}. HiSS also distinguishes itself by simplicity. Unlike DiGS, which relies on a complex hyperparameter schedule (the VP schedule), HiSS uses a static hyperparameter design. This makes HiSS easier to tune and offers competitive performance across various applications. Unlike DiGS, we provide convergence guarantees(Theorem \ref{thm:hiss}) to substantiate the reliability of our proposed approach. 

\paragraph{Multimodal Sampling in Discrete Spaces.} Most prior works on multimodal discrete sampling rely on combinatorial or swap-based proposals like parallel tempering~\citep{swendsen1986replica}, Wang-Landau Sampling~\citep{wang2001efficient}, Swedson-Wang Algorithm ~\citep{swendsen1987nonuniversal}, and early Mode-Jumping MCMC~\citep{madras2003markov}. Based on gradient-based discrete samplers, \citet{pynadath2024gradientbased} introduced Automatic Cyclical Scheduling (ACS), combining automatic tuning of cyclical step sizes and balancing schedules to encourage dynamic transitions between global exploration and localized moves within each cycle.

\paragraph{Hybrid Coupling MCMC Samplers.} While both EDLP~\citep{mohanty2025entropy} and HiSS use a coupling mechanism between a discrete primary variable and an auxiliary continuous variable, while retaining only the discrete variable their objectives and principles differ: In EDMALA, $\bm{\theta}_a$ guides $\bm{\theta}$ to target and enhance sampling from flat modes while ensuring detailed balance. In HiSS, $\bm{\theta_a}$’s role is to fling $\bm{\theta}$ when stuck at a local mode while ensuring detailed balance with respect to the marginal. Another fundamental distinction lies in their theoretical guarantees. EDLP’s convergence bounds pertain to the joint distribution (Theorems 5.5 and 5.6 in~\citet{mohanty2025entropy}), while Theorem \ref{thm:hiss} establishes convergence strictly for the marginal distribution.

\paragraph{Diffusion Models.} Diffusion models, a powerful class of generative models, operate by iteratively injecting Gaussian noise into data and gradually denoising it to generate new samples~\citep{sohl2015deep}. This core paradigm has been extended to discrete spaces, often requiring a trained neural network to approximate the reverse “scoring” process~\citep{pmlr-v235-lou24a, pmlr-v202-avdeyev23a, NEURIPS2024_eb0b13cc}. Our work adapts the noise-denoise mechanism as a means of dislodging and exploring the neighborhood under a pure MCMC framework.

Further, HiSS as a pure MCMC sampler, does not require training or a model-specific score network. It samples from a known (unnormalized) target distribution for probabilistic inference. In contrast, diffusion models learn an unknown data distribution for generative tasks or optimization. HiSS uses a heavy-tailed logistic kernel to bridge disconnected modes, unlike the standard Gaussian transition kernels in discrete diffusion models~\citep{NEURIPS2024_eb0b13cc}. HiSS uses a simple single noise-denoise cycle per Gibbs sweep, unlike the multi-step denoising schedules in many diffusion models.

\section{Preliminaries \label{sec:prelim}}
\textbf{Target Distribution.}
We define a target distribution over a discrete space using an energy function. The target distribution is given by $\pi(\bm{\theta}) = \frac{1}{Z} \exp(U(\bm{\theta}))$,
where $\bm{\theta}$ is a \(d\)-dimensional discrete variable within domain $\bm{\Theta}$, \(U(\bm{\theta})\) represents the energy function, and \(Z\) is the normalizing constant ensuring \(\pi(\bm{\theta})\) is a proper probability distribution. We make the following assumptions consistent with the literature on gradient-based discrete sampling~\citep{grathwohl2021oops, sun2021path, zhang2022langevin}:
1. \textit{Coordinate-wise Factorization}: The domain \(\bm{\Theta}\) is factorized such that \(\bm{\Theta} = \Pi_{i=1}^d \Theta_i\).
2. \textit{Differentiable Energy Function}: The energy function \(U\) can be extended to a differentiable function in \(\mathbb{R}^d\). This extension is crucial for applying gradient-based sampling methods, as it allows the use of gradient information.

\textbf{Gradient Computation in Discrete Spaces. }
To enable gradient-based sampling in discrete spaces, we rely on the Functional Extension framework, discussed in Section 3 of ~\citet{grathwohl2021oops}. Per this framework, discrete distributions defined by energy functions sometimes possess a `natural differentiable extension': $U: \mathbb{R}^d \to \mathbb{R}$ to compute $\nabla_{\bm\theta} U(\bm\theta)$ at discrete points via standard automatic differentiation. This extension provides the necessary local geometry to guide various discrete gradient-based samplers like GWG~\citep{grathwohl2021oops}, DMALA~\citep{zhang2022langevin}, and ACS~\citep{pynadath2024gradientbased}.

\textbf{Discrete Langevin Proposal.}
The DLP is an extension of the Langevin algorithm tailored for discrete spaces~\citep{zhang2022langevin}. At a given position $\bm{\theta}$, the proposal distribution $q(\cdot|\bm{\theta})$ determines the next position. The proposal distribution is formulated as:
\begin{align}
    q(\bm{\theta}'|\bm{\theta}) &= \frac{\exp\left(-\frac{1}{2\alpha}\|\bm{\theta}' - \bm{\theta} - \frac{\alpha}{2}\nabla U(\bm{\theta})\|^2\right)}{Z_{\bm{\Theta}}(\bm{\theta})},
    \label{eq:proposal}
\end{align}
where $Z_{\bm{\Theta}}(\bm{\theta})$ is the normalizing constant. 

\textbf{Logistic Convolution.}
 In density estimation, the logistic convolutional kernel(sometimes referred to as the sech-squared convolutional kernel) smoothly facilitates with optimal bandwidth selection and enhances the modeling of complex data distributions~\citep{aboelhadid2018logistic}. Thus, if the conditional distribution is of the form
$p(\bm{\theta}_a \mid \bm{\theta})\sim \textit{Logistic}(\bm{\theta}, \eta) $, where $\eta>0$ is the scaling parameter. Then, for original distribution $p(\bm{\theta})$ , the convolved distribution \( {p}(\bm{\theta}_a) \) can be expressed as:
\begin{equation}\label{eq:auxillary}
{p}(\bm{\theta}_a) = \sum_{\bm{\theta} \in \bm{\Theta}} p({\bm{\theta}_a} \mid \bm{\theta})p(\bm{\theta}) 
\end{equation}
By ensuring non-negligible density paths between distant modes, the kernel improves connectivity in $p(\bm{\theta}_a)$ relative to $p(\bm{\theta})$, capturing both sharp transitions and smooth variations in the probability landscape.

\section{Hyperbolic Secant-squared Gibbs Sampling \label{sec:method}}

In this section, we present our sampling framework to address the limitations of gradient-based methods in traversing disconnected modes. To achieve this, we introduce an auxiliary variable $\bm{\theta}_a$ and explain how it smooths the disconnected energy landscape through a well-defined joint probability distribution. We then detail our sampling strategy, outlining each sub-step of the algorithm.

\subsection{Joint Distribution}\label{sec:bridging}
To facilitate efficient exploration between modes, we aim to derive a smoothed version of the target discrete distribution, $p(\bm{\theta}) \propto \exp{(U(\bm{\theta}))}$. Inspired by the logistic kernel in \eqref{eq:auxillary}, we define the smoothed target distribution:
\begin{equation}
    p(\bm{\theta}_a) \propto \sum_{\bm{\theta} \in \bm{\Theta}} \exp{\left\{ U(\bm{\theta}) - 2\ln\left(\cosh\left(\frac{\bm{\theta}_a - \bm{\theta}}{2\eta}\right)\right)\right\}}
    \label{eq:theta_a_posterior}
\end{equation}

$p(\bm{\theta}_a)$ represents a continuous version of \( p(\bm{\theta}) \), connecting otherwise isolated modes. 

\textbf{Why Logistic Convolutional Kernel?}
Our choice of a logistic convolutional kernel is a core design decision driven by both theoretical and practical considerations. Unlike Gaussian kernels, which rapidly decay and concentrate mass near the current mode, the logistic kernel’s slower tail decay and broader spread facilitate better bridging across disconnected regions. Compared to a Gaussian kernel, the logistic kernel is also less sensitive to hyperparameter tuning, ensuring stable $p(\bm{\theta}_a)$. Notably, it retains support over $\mathbb{R}^d$, similar to the Gaussian kernel, but avoids pathologies observed in other heavy-tailed distributions (e.g., undefined moments in Cauchy, non-differentiability in Laplace). Furthermore, generating logistic noise is computationally efficient, due to closed-form inverse CDF sampling. These attributes make it a robust and scalable choice for discrete MCMC~\citep{kingma2013auto, maddison2017concrete}. We empirically demonstrate the advantages of logistic kernel in Section \ref{sec:exp:as} and Appendix~\ref{sec:app_eta}. 

Inspired by the coupling method used in \citet{mohanty2025entropy}, we couple $\bm{\theta}$ and $\bm{\theta}_a$ as follows:
\begin{lemma}\label{lemma:joint_posterior}
Given \(\widetilde{\bm{\theta}} = [\bm{\theta}^T, \bm{\theta}_a^T]^T\), the joint distribution p(\(\widetilde{\bm{\theta}}\)) is:
\begin{equation}\label{eq:joint_posterior}
\resizebox{0.98\hsize}{!}{$
    p(\widetilde{\bm{\theta}}) = p(\bm{\theta}, \bm{\theta}_a) \propto \exp\left\{ U(\bm{\theta}) - 2\ln\left(\cosh\left(\frac{\bm{\theta}_a - \bm{\theta}}{2\eta} \right) \right)\right\}
    $}
\end{equation}
By construction, the marginal distributions of \(\bm{\theta}\) and \(\bm{\theta}_a\) are the original distribution \( p(\bm{\theta}) \) and the smoothed distribution \( p(\bm{\theta}_a) \)(Eq.~\ref{eq:theta_a_posterior}).
\end{lemma}
One notices, the joint hybrid variable \(\widetilde{\bm{\theta}}\)  lies in a product space where first $d$ coordinates are discrete-valued and the remaining $d$ coordinates lie in $\mathbb{R}^d$ i.e. $\widetilde{\bm{\theta}}\in \bm{\Theta} \times \mathbb{R}^d$ and its energy function can be expressed as 
$$ U(\widetilde{\bm{\theta}}) = U(\bm{\theta}) - 2\ln\left(\cosh\left(\frac{\bm{\theta}_a - \bm{\theta}}{2\eta} \right) \right).$$

\subsection{Gibbs-like Update Procedure}
Since $p(\bm{\theta}_a)$ is a smoothed version of the target distribution, samples from $p(\bm{\theta}_a)$ can effectively serve as launchpads for exploring the target discrete distribution $p(\bm{\theta})$. However, sampling directly from $p(\bm{\theta}_a)$ is intractable due to its summation form.
In contrast, the conditional distribution $p(\bm{\theta}|\bm{\theta}_a)$ remains tractable and is given by:
\begin{equation}\label{eq:tractable}
p(\bm{\theta}|\bm{\theta}_a) \propto \frac{1}{Z_{\bm{\theta}_a}} \exp\left\{ U(\bm{\theta}) - 2\ln\left(\cosh\left(\frac{\bm{\theta}_a - \bm{\theta}}{2\eta}\right)\right)  \right\},
\end{equation}
where \(Z_{\bm{\theta}_a}\) is the associated normalization constant. Moreover, its gradient also becomes easy to compute, i.e. $$\nabla_{\bm{\theta}}\log p(\bm{\theta}|\bm{\theta}_a) = \nabla_{\bm{\theta}} U(\bm{\theta}) + \frac{1}{\eta}\tanh\left(\frac{\bm{\theta}_a - \bm{\theta}}{2\eta}\right)$$. This allows us to use gradient-based discrete samplers like DMALA~\citep{zhang2022langevin}. Similarly, the conditional distribution for the auxiliary variable $\bm{\theta}_a$ is, $p(\bm{\theta}_a|\bm{\theta}) \sim Logistic(\bm{\theta}, \eta)$, 
where we can perform sampling easily as follows,
\begin{equation}\label{eq:noise}
 \bm{\theta}_a = \bm{\theta} + \eta\xi, \quad \xi \sim \text{Logistic}(0,1)^d. 
 \end{equation}

\subsection{Sampling Algorithm}
HiSS utilizes a Gibbs sampler to sample from the joint distribution specified in Equation \eqref{eq:joint_posterior}. It alternates between (1) perturbing the discrete state \(\bm{\theta}\) via a logistic noising step to encourage exploration, (2) recovering a candidate discrete state \(\bm{\theta}'_{\text{init}}\) through a denoising step grounded in a coordinate-wise logistic energy model, (3) accepting the proposal using a Metropolis-Hastings correction, and (4) refining the sample with a gradient-based sampler. This combination enables both mode hopping and local exploitation, facilitating efficient sampling in disconnected and multimodal discrete landscapes. We provide an ablation study for HiSS in Section \ref{sec:exp:as}.

\subsubsection{Noising}
To facilitate exploration of new cluster of modes beyond the current discrete mode, we perturb $\bm{\theta}$ by adding noise scaled by a factor $\eta$ to get $\bm{\theta}_a$, encouraging the sampler to escape local modes and explore diverse regions of the state space. Mathematically, the noising process is defined in Equation \eqref{eq:noise}.

\subsubsection{Denoising}
To recover a discrete proposal state from the previously noised sample $\bm{\theta}_a$, we will map $\bm{\theta}_a$ back to a discrete state. To ensure that this mapping favors discrete points close to $\bm{\theta}_a$ in a probabilistically consistent way, we define an energy function based on the negative log-density of a logistic distribution. For each coordinate $i$, we compute the energy for every candidate discrete state $\theta_i' \in \Theta_i$. To see this, we write $q_{\text{denoise}}(\bm{\theta}'|\bm{\theta}_a) =\prod_{i=1}^{d} {q_{\text{denoise}}}_{i}({\theta}_i'|\bm{\theta}_a)$, where ${q_{\text{denoise}}}_{i}({\theta}_i'|\bm{\theta}_a)$ has the energy: $$U(\theta_i') = -2 \ln\left(\cosh\left(\frac{{{\theta}_a}_i - \theta_i'}{2\eta}\right)\right)$$
 Using these energies, we compute the probabilities for each candidate discrete state via the softmax function.
These probabilities form a categorical distribution over $\bm{\Theta}$, and thus $\bm{\theta}'_{init}$ is sampled  as:
\begin{equation}\label{eq:denoise}
 \bm{\theta}'_{init} \sim \text{Categorical}(\text{Softmax}(U(\theta_i'))).
\end{equation}
It is worth noting that the denoising is independent of the original sample $\bm{\theta}$. 

\subsubsection{Metropolis-Hastings Acceptance for Denoised Proposal }
We apply the  Metropolis-Hastings (MH) step for the newly proposed discrete state:
\begin{equation}\label{eq:MH:acc:rate}
\resizebox{0.95\hsize}{!}{$
a_{\text{init}}(\bm{{\theta}}'_{\text{init}}|\bm{{\theta}}^{(i-1)})=
\min\left(1, \frac{\pi(\bm{{\theta}}'_{\text{init}})
q_{\text{noise}}(\bm{{\theta}}_a^{(i-1)}|\bm{{\theta}}'_{\text{init}})
q_{\text{denoise}}(\bm{{\theta}}^{(i-1)} \mid \bm{{\theta}}_a^{(i-1)})}
{\pi(\bm{{\theta}}^{(i-1)})
q_{\text{noise}}(\bm{{\theta}}_a^{(i-1)}|\bm{{\theta}}^{(i-1)})
q_{\text{denoise}}(\bm{{\theta}}'_{\text{init}} \mid \bm{{\theta}}_a^{(i-1)})}\right)
$}
\end{equation}

The MH step ensures the proposed state is from a high-density region $p(\bm{\theta})$, thereby providing a good initialization for the next step. We refer to this step as the MwG (Metropolis-within-Gibbs) step.
\subsubsection{Gradient-based Denoising}
After the MH step, HiSS runs a discrete gradient-based sampler ${\Phi}$ (e.g., GWG, DMALA), initialized from the accepted state $\bm{\theta}_{\text{init}}^{(i)}$, to sample from $p(\bm{\theta}|\bm{\theta}_a)$ in \eqref{eq:tractable}. For this work, we use DMALA for theoretical results (Section \ref{sec:theory}) and experiments (Section \ref{sec:exp}). Detailed implementation specifics are provided in Appendix \ref{sec:app_DLP}.
We will collect samples of \(\bm{\theta}\), as the marginal distribution of \(p(\widetilde{\bm{\theta}})\) over \(\bm{\theta}\) yields our desired target distribution.

\textbf{Comparison to DiGS.} HiSS is specifically designed for discrete multimodal distributions with isolated modes. Unlike DiGS~\citep{2024diffusive}, which operates solely in continuous spaces, HiSS targets a hybrid joint distribution where the primary variable $\bm{\theta}$ is discrete and the auxiliary variable $\bm{\theta}_a$ is continuous. To couple the two, HiSS replaces DiGS’s Gaussian kernel and sensitive VP schedule with a logistic convolution kernel with static tuning, leading to a simpler formulation with theoretical guarantees.

\begin{algorithm}[t]
   \caption{Hyperbolic Secant-squared Gibbs-Sampling (HiSS)}
   \label{alg:1}
\begin{algorithmic}
    \STATE \textbf{Inputs:} 
    \STATE \hspace{1em} Main variable $\bm{\theta} \in \bm{\Theta}$ , Auxiliary variable $\bm{\theta}_a \in \mathbb{R}^d$,  Main stepsize $\alpha$,  Scale parameter $\eta$, Number of Gibbs Sweeps $G$, Discrete gradient-based Sampler ${\Phi}$,  Number of denoising steps $L$
    \STATE \textbf{Initialize}: $\mathcal{S} \gets \emptyset$ 
    \LOOP
   \FOR{$i \gets 1$ to $G$}
   \STATE 1. Sample $\bm{\theta}_a^{i-1} \sim q_{\text{noise}}(\bm{\theta}_a|\bm{\theta})$ using proposal in Equation~\ref{eq:noise}. 
    \STATE 2. Propose $\bm{\theta}_{init}'^{}\sim q_{\text{denoise}}(\bm{\theta}|\bm{\theta}_a^{i-1})$ using proposal in Equation~\ref{eq:denoise}.
    \STATE 3. Accept ${\bm{\theta}}_{init}^{(i)}\gets \bm{\theta}_{init}'^{}$ with probability in \eqref{eq:MH:acc:rate}.
    Otherwise, set ${\bm{\theta}}_{init}^{(i)}\gets \bm{\theta}_{}^{i-1}$.
   \STATE 4. Sample $\bm{\theta}^{(i)}$ using $\Phi$ for $L$ steps from the initial point $\bm{\theta}^{(i)}_{init}$ conditioned on $\bm{\theta}_a^{i-1}$.
   \ENDFOR
      \STATE \textbf{Update} 
        $\mathcal{S} \gets \mathcal{S}\cup\{\bm{\theta}^{(G)}\}$
   \ENDLOOP
   \STATE {\bfseries Output:} $\mathcal{S}$
\end{algorithmic}
\end{algorithm}

\section{Theoretical Analysis\label{sec:theory}}
In this section, we provide asymptotic and non-asymptotic convergence guarantees for HiSS. We make similar assumptions as in \citet{pynadath2024gradientbased}. Those are as follows,

\begin{assumption} \label{assumption:5.1}
The function $U(\cdot) \in C^2(\mathbb{R}^d)$ has $M$-Lipschitz gradient. Note that it implicitly assumes that the set in domain $\bm{\Theta}$ is finite. We define $conv(\bm{\Theta})$ as the convex hull of the set $\bm{\Theta}$. 
\end{assumption}
\begin{assumption} \label{assumption:5.2}
For each $\bm{\theta} \in \mathbb{R}^d$, there exists an open ball containing $\bm{\theta}$ of some radius $r_{\bm{\theta}}$, denoted by $B(\bm{\theta}, r_{\bm{\theta}})$, such that the function $U(\cdot)$ is $m_{\bm{\theta}}$-strongly concave in $B(\bm{\theta}, r_{\bm{\theta}})$ for some $m_{\bm{\theta}} > 0$. 
\end{assumption}

We define $\text{diam}(\bm{\Theta}) = \sup_{\bm{\theta}, \bm{\theta}' \in \bm{\Theta}} \|\bm{\theta} - \bm{\theta}'\|$, and $\Delta(\bm{\Theta})=\sup_{\bm{\theta} , \bm{\theta}' \in \bm{\Theta}} \|\bm{\theta}'-\bm{\theta}\|_1 $.
Finally, we define $ a \in \arg\min_{\bm{\theta} \in \bm{\Theta}} \|\nabla U(\bm{\theta})\|$ as the set of values which minimizes the energy function in $\bm{\Theta}$. Assumptions \ref{assumption:5.1} ,\ref{assumption:5.2}  are standard in optimization and sampling literature ~\citep{bottou2018optimization,dalalyan2017further, durmus2017nonasymptotic}. Under Assumption \ref{assumption:5.2}, $U(\cdot)$ is $m$-strongly concave on $\text{conv}(\bm{\Theta})$, following Lemma C.3 from \citet{pynadath2024gradientbased}. The total variation distance between two probability measures $\mu$ and $\nu$, defined on some space $\bm{\theta} \subset \mathbb{R}^d$ is $\|\mu - \nu\|_{TV} = \sup_{A \subseteq B(\bm{\theta})} |\mu(A) - \nu(A)|,$ where $B(\bm{\theta})$ is the set of all measurable sets in $\bm{\theta}$.
\begin{proposition}\label{prop:gibbs-valid}
HiSS (Algorithm~\ref{alg:1}) generates an irreducible and recurrent Markov chain for the joint distribution \(\pi(\bm{\widetilde{\theta}})\). 
\end{proposition}
This theoretical guarantee for HiSS holds true for any choice of ${\Phi}$.
\subsection{Convergence Analysis for HiSS (with DMALA)}
The following results hold true for DMALA as ${\Phi}$ in Algorithm~\ref{alg:1}.
\begin{proposition}\label{prop:db}
HiSS (Algorithm~\ref{alg:1}) ensures detailed balance with respect to the target distribution \(\pi(\bm{\theta})\).
\end{proposition}

Based on Proposition \ref{prop:db}, we focus on the marginal distribution \(\pi(\bm{\theta})\), we establish a non-asymptotic convergence guarantee for HiSS using a uniform minorization argument.

\begin{theorem}\label{thm:hiss}
Under Assumptions \ref{assumption:5.1} ,\ref{assumption:5.2} and $\alpha < \frac{2}{M}$ in Algorithm \ref{alg:1}, the marginal Markov chain P is uniformly ergodic under,
\[
\norm{P^k (x,\cdot) -{\pi}}_{TV}\le (1-\epsilon_{\alpha})^k
\]
where,\\
\resizebox{0.98\hsize}{!}{$
\epsilon_{\alpha} = 
\exp \left\{
\begin{aligned}
    &\Bigg( 
       -M\!\left(\tfrac{LG}{2} + G\right) 
       - \tfrac{LG}{\alpha} 
       + \tfrac{mLG}{4} 
    \Bigg) \, \mathrm{diam}(\bm{\Theta})^2 
    \quad\\
    &+ \Bigg( 
       \tfrac{G\sqrt{d}\,(3L-2) + LG}{\eta} 
       - \Big(\tfrac{LG}{2} + G\Big) \|\nabla U(a)\|
    \Bigg) \, \mathrm{diam}(\bm{\Theta})
\end{aligned}
\right\}
$}
\end{theorem}

As $\alpha \rightarrow 0$, $\epsilon_{\alpha} \rightarrow 0$, causing the convergence factor $1 - \epsilon_{\alpha}$ to approach 1. This slows the convergence rate, as the chain takes longer to approach the stationary distribution. As dimension $d$ grows, we may need to increase $\eta$ (propose even broader jumps) to maintain  adequate level of mixing and faster convergence. Refer to Appendix \ref{sec:app_proof} for detailed proofs.

\begin{theorem}\label{thm:hiss2}
Under Assumptions \ref{assumption:5.1} ,\ref{assumption:5.2} and $\alpha < \frac{2}{M}$ in Algorithm \ref{alg:1},  for Markov chain P,
for any real-valued function $f$ and samples $X_1,X_2,X_3,\cdots,X_n$ from $P$, one has
    \begin{align*}
    \sqrt{n}\left(\frac{1}{n} \sum_{i=1}^{n} f(X_i)-\sum_{\theta \in \Theta} f(\theta) \pi(\theta)  \right) \overset{d}{\to} N(0, \tilde{\sigma}^2_{*})
    \end{align*}
for some $\tilde{\sigma}_*>0$ as $n\to \infty$.
\end{theorem}
\begin{proof}
Theorem~\ref{thm:hiss2} is true due to direct consequence of using Theorem \ref{thm:hiss} and ~\citet{jones2004markov}[Corollary 5].
\end{proof}

\section{Experiments\label{sec:exp}}
We evaluated HiSS empirically, showing it mixes faster and more efficiently than existing methods in various discrete multimodal settings. Our setups are inspired by \citet{zhang2022langevin, mohanty2025entropy}. We compare HiSS against some popular pure gradient-based baselines like  Gibbs with Gradient (GWG)~\citep{grathwohl2021oops}, Discrete Metropolis-Adjusted Langevin Algorithm (DMALA)~\citep{zhang2022langevin}. Other baselines targeting discrete multimodal distributions such as, Automatic Cyclical Sampler (ACS)~\citep{pynadath2024gradientbased}, and Parallel Tempering(PT)~\citep{swendsen1986replica} are also included. Being consistent with \citet{2024diffusive}, HiSS and PT both employ DMALA as their base sampler. For PT, for most tasks, for every original chain we use 5 temperature chains with geometric temperature scaling, employing consecutive swaps between adjacent chains \citep{earl2005parallel, kone2005selection}. We use the same stepsize $\alpha$ for DMALA, ACS, PT+DMALA, and HiSS. We ensure the number of iterations per sample for all samplers across all tasks is the same for fair runtime comparison.
\subsection{Motivational Synthetic Example}
\begin{figure}[t]
    \centering
    \begin{minipage}[t]{0.50\columnwidth}
        \centering
        \includegraphics[width=1\linewidth, height=0.155\textheight]{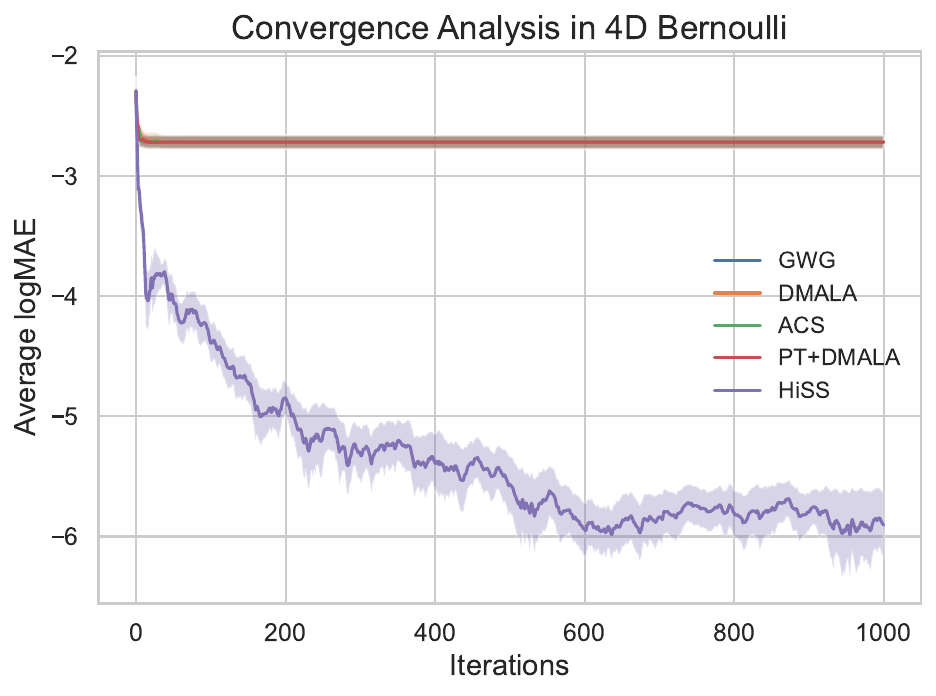}
    \end{minipage}\hfill
    \begin{minipage}[t]{0.50\columnwidth}
        \centering
        \includegraphics[width=1\linewidth, height=0.155\textheight]{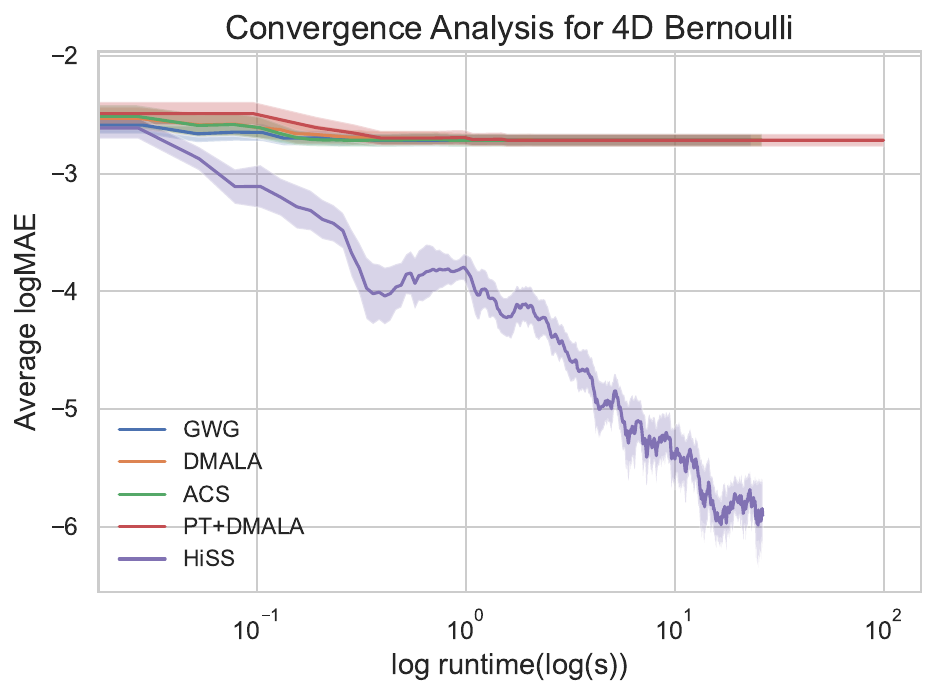}
    \end{minipage}
    \caption{4D Joint Bernoulli}
    \label{fig:berneval}
\end{figure}
We consider sampling from a Joint Quadrivariate Bernoulli Distribution, a multivariate distribution where each of the four binary random variables can take on the value 0 or 1.  Let \(\bm{\theta} = (\theta_1, \theta_2, \theta_3, \theta_4)\) be a 4-dimensional binary random vector. The joint probability distribution is specified by \(p_{\bm{\theta}}\), which represents the probability of the vector (\( \theta_1, \theta_2, \theta_3, \theta_4 \)). For a given state \(\bm{\theta}\) the energy function is given by :
\[U(\bm{\theta}) = \sum_{{a} \in \{0,1\}^4} \left( \prod_{n=1}^4 \theta_n^{a_n} (1 - \theta_n)^{1 - a_n} \right) \ln p_{{a}},\]

Details of the target distribution, featuring isolated modes separated by low-density energy, and its visualization can be found in Figure~\ref{fig:bern_target} of the Appendix. For each sampler, we ran 10 parallel chains for 1000 iterations. Even in this simple setup, gradient-based samplers like GWG and DMALA struggle to converge effectively. ACS and PT, methods designed to handle multimodal distributions, also perform poorly. In contrast, HiSS, demonstrates impressive performance, achieving the lowest log Mean Absolute Error (MAE). PT, as a high-resource sampling method, takes the longest due to inter-temperature chain communication overhead. HiSS, however, achieves remarkable convergence with respect to runtime, demonstrating efficiency and accuracy (Figure \ref{fig:berneval}). We provide hyperparameter settings, additional results, and diagnostics in the Appendix \ref{sec:app_bern}.

\textbf{Why HiSS Outperforms PT?}
HiSS significantly outperforms PT because PT adjusts the inverse temperature $\beta=\frac{1}{T}$ to enhance exploration, but disconnected modes still remain inaccessible since $p(\bm{\theta})^\beta = 0 \quad \forall\beta>0$ when $p(\bm{\theta}) = 0$. In contrast, the Logistic Convolution kernel assigns strictly positive mass across the target distribution (see Figure \ref{fig:temp}).
\begin{figure}[t]
    \centering
    \includegraphics[width=0.48\textwidth, height=0.145\textwidth]{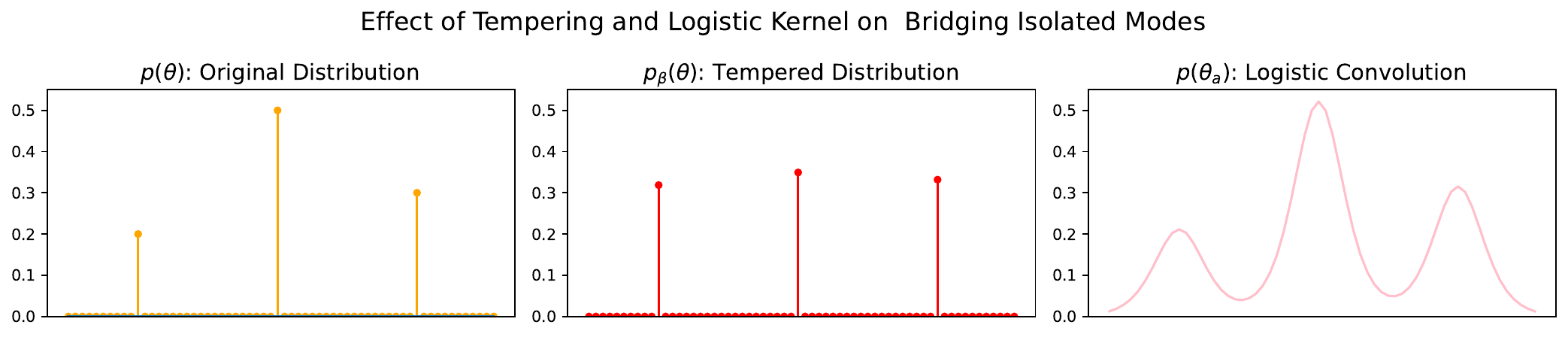}
    \caption{Comparison of the target distribution, tempered distribution, and $p(\bm{\theta}_a)$ in HiSS.}
    \label{fig:temp}
\end{figure}

\subsection{Sampling from Ising Models}
We consider a $3$ by $3$ lattice Ising model with random variable $\bm{\theta} \in \{-1,1\}^d$, and $d=3 \times 3 = 9$.  The energy function is,
\[
U(\bm{\theta}) = a\bm{\theta}^{\top}\mathbf{W}\bm{\theta} + b\bm{\theta},
\]
where $\mathbf{W}$ is the interaction matrix, $a=0.5$ is the connectivity strength and $b=0.1$ is the bias.

\begin{figure}[t]
    \centering
    \begin{minipage}[t]{0.50\columnwidth}
        \centering
        \includegraphics[width=1\linewidth, height=0.155\textheight]{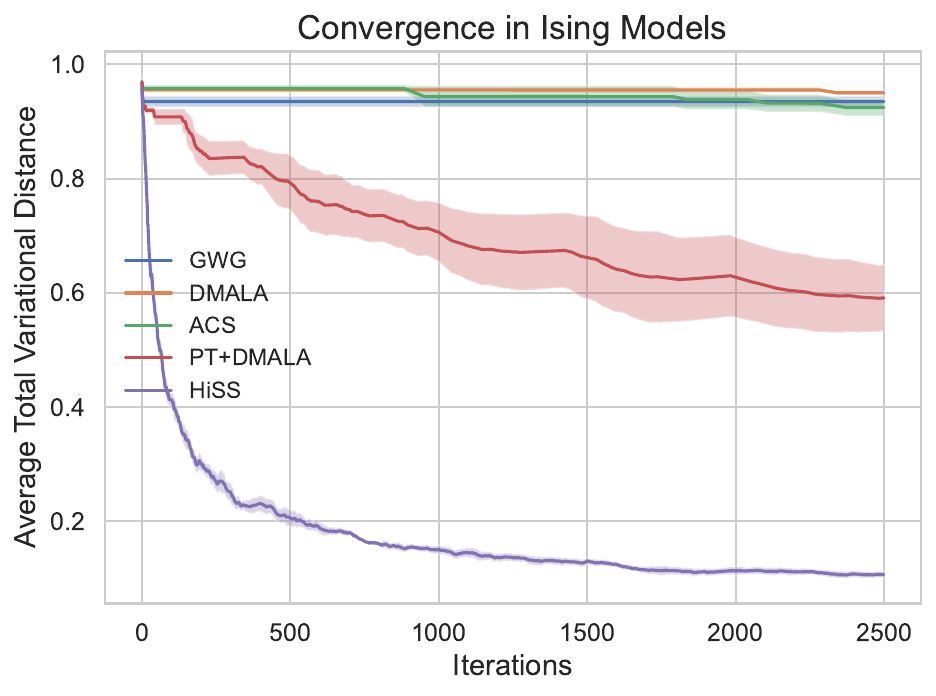}
    \end{minipage}\hfill
    \begin{minipage}[t]{0.50\columnwidth}
        \centering
        \includegraphics[width=1\linewidth, height=0.155\textheight]{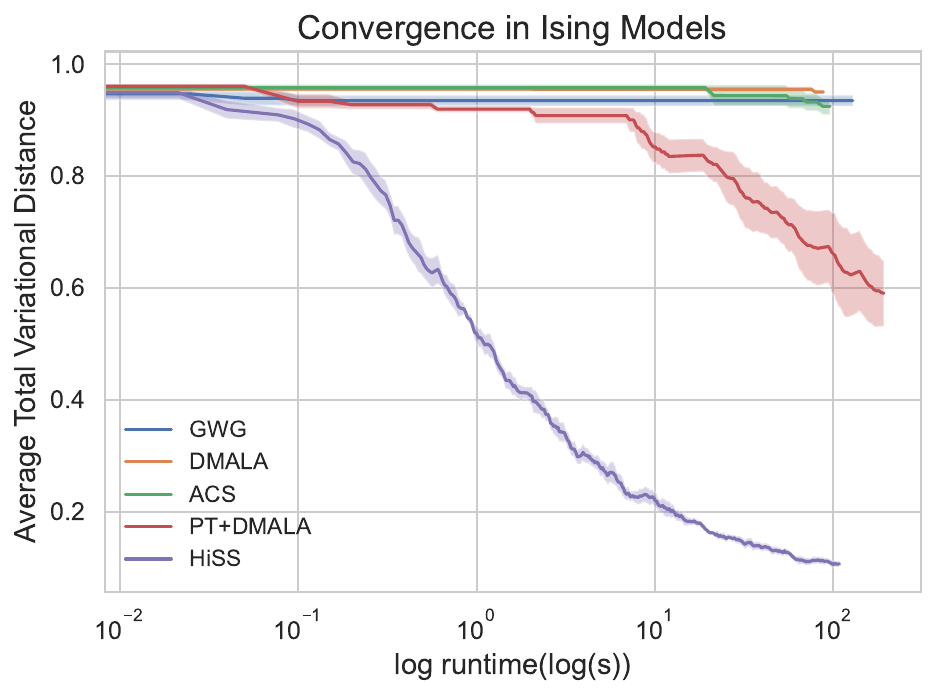}
    \end{minipage}
    \caption{Ising Model}
    \label{fig:isingeval}
\end{figure}

 We run 5 chains independently for each sampler for 2500 iterations and report the average total variational distance (TVD) between the estimated and true distribution with standard error against iterations and runtime in Figure~\ref{fig:isingeval}. The convergence analysis reveals that HiSS significantly outperforms other samplers, achieving the lowest TVD ($\approx0.15$) with consistent and rapid convergence. In contrast, PT+DMALA shows delayed convergence, while GWG, DMALA, and ACS fail to effectively navigate the distribution, remaining stuck at higher TVD levels. HiSS achieves convergence faster than other methods, while PT+DMALA has slower runtime performance due to high communication overhead. GWG, DMALA, and ACS show limited improvement. We provide hyperparameter settings, additional results, and diagnostics in the Appendix \ref{sec:app_ising}.

\subsection{Traveling Salesman Problem}
In the Traveling Salesman Problem (TSP), the objective is to determine the shortest route that visits $n$ cities exactly once before returning to the starting location. \( U(\bm{\theta}) \) is designed to capture the total cost of a particular route configuration \( \bm{\theta} = (\theta_1, \theta_2, \dots, \theta_n) \). The expression of \( U(\bm{\theta}) \) is:
\begingroup
\setlength{\abovedisplayskip}{6pt}
\setlength{\belowdisplayskip}{6pt}
\begin{equation*}
\resizebox{\linewidth}{!}{%
$U(\bm{\theta}) =
-\Big(\sum_{i=1}^{n-1}\sqrt{(x_i-x_{i+1})^2+(y_i-y_{i+1})^2}
+\sqrt{(x_n-x_1)^2+(y_n-y_1)^2}\Big)$%
}
\end{equation*}
\endgroup
where each city location is given by \( \theta_i = (x_i, y_i) \in \mathbb{R}^2 \). The final term, \( \sqrt{(x_n - x_1)^2 + (y_n - y_1)^2} \), ensures that the route forms a closed loop, thereby completing the tour.

We use the \textit{eil14} dataset, a 14-city problem extracted from the 51-city Christofides/Eilon instance, available in the TSPLIB95 benchmark repository, first introduced in \citet{reinelt1991tsplib}. The dataset consists of 2D Euclidean coordinates of cities, and the objective is to find the shortest tour visiting all cities exactly once. The dataset is publicly available.\footnote{GitHub Repository: \url{https://github.com/jam7/tsp/tree/master}; TSPLIB Archive: \url{http://comopt.ifi.uni-heidelberg.de/software/TSPLIB95/}.} Each solution is represented as a binary square matrix of size $n$, where each row corresponds to a position in the tour and each column corresponds to a city. If a proposed solution violates the uniqueness of city visits, we reject the sample and retain the current solution. In TSP, the solution space grows exponentially with the number of cities, reaching $\mathcal{O}(n!)$. We execute each sampler for 10000 iterations.

\begin{table}[t]
\centering
\small
\setlength{\tabcolsep}{5pt}        
\renewcommand{\arraystretch}{1.15} 
\caption{Performance comparison of different samplers on the \texttt{eil14} dataset.}
\label{tab:sampler_comparison}
\begin{adjustbox}{max width=\columnwidth, height=0.03\textheight}
\begin{tabular}{|l|r|r|r|r|}
\hline
\textbf{Sampler} &
\textbf{Cost} $\downarrow$ &
\textbf{PMC} $\uparrow$ &
\textbf{Jaccard} $\downarrow$ &
\textbf{Unique Solns.} $\uparrow$ \\
\hline
GWG      & 370.7086 $\pm$ 19.7404 &  27.0000 $\pm$ 0.0000 & 0.5555 $\pm$ 0.0000 &  2 \\
DMALA    & 339.0111 $\pm$ 43.5156 &  84.2444 $\pm$ 23.4995 & 0.1483 $\pm$ 0.1481 & 10 \\
ACS      & 382.8639 $\pm$  2.5501 & \textbf{105.0000} $\pm$ 0.0000 & \textbf{0.0370} $\pm$ 0.0000 &  2 \\
PT+DMALA & 337.1584 $\pm$ 44.9881 &  80.4167 $\pm$ 19.7419 & 0.1402 $\pm$ 0.1357 &  9 \\
HiSS     & \textbf{277.9008} $\pm$ 19.8467 & 103.0727 $\pm$ 9.1807 & 0.0990 $\pm$ 0.0964 & \textbf{11} \\
\hline
\end{tabular}
\end{adjustbox}
\end{table}
Table \ref{tab:sampler_comparison} highlights HiSS’s superior balance between solution quality, consistency, and solution exploration quality in combinatorial optimization. HiSS achieves the lowest cost with minimal variance, ensuring both optimality and stability, unlike PT+DMALA and DMALA, which generate competitive solution counts but at significantly higher costs. HiSS produces structurally diverse solutions, as measured by Pairwise Mismatch Count (PMC), inspired by Pairwise Order Discrepancy ~\citep{bossek2020comprehensive,helgemo2021evaluating,vinyals2015pointer} and Kendall’s Tau ~\citep{fagin2003comparing}, and Jaccard similarity ~\citep{li2022circular} while ensuring samples discovered are low-cost solutions. We provide additional insights and hyperparameter settings in Appendix \ref{sec:app_tsp}. 
\subsection{Binary Bayesian Neural Networks}
\begin{table*}[!t]  
\centering
\small                                
\setlength{\tabcolsep}{6pt}
\renewcommand\arraystretch{1.5}
\caption{Experiment results with binary Bayesian neural networks on four datasets.}
\label{tab:bayesbnn}
\begin{adjustbox}{max width=\textwidth}
\begin{tabular}{c|ccccc|ccccc}
\toprule
\multirow{2}{*}{Dataset} & \multicolumn{5}{c|}{Average Test Log-likelihood ($\uparrow$)} &
\multicolumn{5}{c}{Average Test Root-Mean Square Error ($\downarrow$)} \\
\cmidrule(lr){2-6}\cmidrule(l){7-11}
 & GWG & DMALA  & ACS & PT+DMALA & HiSS & GWG & DMALA & ACS & PT+DMALA & HiSS \\
\midrule
Breast Cancer & $-0.0241 \pm 0.0030$ & $-0.0240 \pm 0.0014$ & $-0.0280 \pm 0.0022$ & $-0.0246 \pm 0.0018$ & $\mathbf{-0.0237} \pm 0.0014$ & $0.1553 \pm 0.0082$ & $0.1550 \pm 0.0047$ & $0.1673 \pm 0.0065$ & $0.1568 \pm 0.0060$ & $\mathbf{0.1541 \pm 0.0047}$ \\
COMPAS        & $-0.2265 \pm 0.0025$ & $-0.2271 \pm 0.0025$ & $-0.2284 \pm 0.0025$ & $-0.2265 \pm 0.0026$ & $\mathbf{-0.2237} \pm 0.0039$ & $0.4759 \pm 0.0027$ & $0.4766 \pm 0.0026$ & $0.4779 \pm 0.0026$ & $0.4759 \pm 0.0027$ & $\mathbf{0.4729 \pm 0.0042}$ \\
HIV           & $-0.7025 \pm 0.1127$ & $-0.7446 \pm 0.0000$ & $-0.7446 \pm 0.0000$ & $-0.7446 \pm 0.0000$ & $\mathbf{-0.2551} \pm 0.0009$ & $0.8341 \pm 0.0082$ & $0.8629 \pm 0.0000$ & $0.8629 \pm 0.0000$ & $0.8629 \pm 0.0000$ & $\mathbf{0.5050 \pm 0.0089}$ \\
Blog          & $-0.2799 \pm 0.0351$ & $-0.2919 \pm 0.0000$ & $-0.2919 \pm 0.0000$ & $-0.2919 \pm 0.0000$ & $\mathbf{-0.2136} \pm 0.0129$ & $0.5277 \pm 0.0376$ & $0.5403 \pm 0.0000$ & $0.5403 \pm 0.0000$ & $0.5403 \pm 0.0000$ & $\mathbf{0.4620 \pm 0.0137}$ \\
\bottomrule
\end{tabular}
\end{adjustbox}
\end{table*}
The posterior distribution of binary neural networks (BNNs)~\citep{courbariaux2016binarized, rastegari2016xnor, liu2021post} is highly multimodal, characterized by disconnected or isolated modes~\citep{zhangcyclical,izmailov2021bayesian}. To investigate this, we perform regression tasks on four UCI datasets~\citep{Dua:2019}, defining the energy function as:
\[ U(\bm{\theta}) = -\sum_{i=1}^N ||f_{\bm{\theta}}(x_i) - y_i||^2,\]
where \(D=\{x_i, y_i\}_{i=1}^N\) is the training dataset, and \(f_\theta\) represents a two-layer neural network with \texttt{Tanh} activation and 100 hidden neurons. To ensure a robust evaluation, we train 50 networks in parallel. We report the log-likelihood and RMSE on the test set along with their standard deviation (see Table \ref{tab:bayesbnn}). Notably, HiSS outperforms other baseline methods in terms of generalization to unseen data. Consistent with the findings of \citet{2024diffusive} on their synthetic Bayesian Neural Networks Setup (Section 4.2), we hypothesize that this performance gain arises from HiSS's ability to capture a broader range of modes in the posterior distribution. We provide the dataset details in the Appendix \ref{bnnn}.

\subsection{Ablation Study}\label{sec:exp:as}

\begin{figure}[t]
    \centering
    \begin{minipage}[t]{0.48\columnwidth}
        \centering
        \includegraphics[width=\linewidth,height=0.155\textheight]{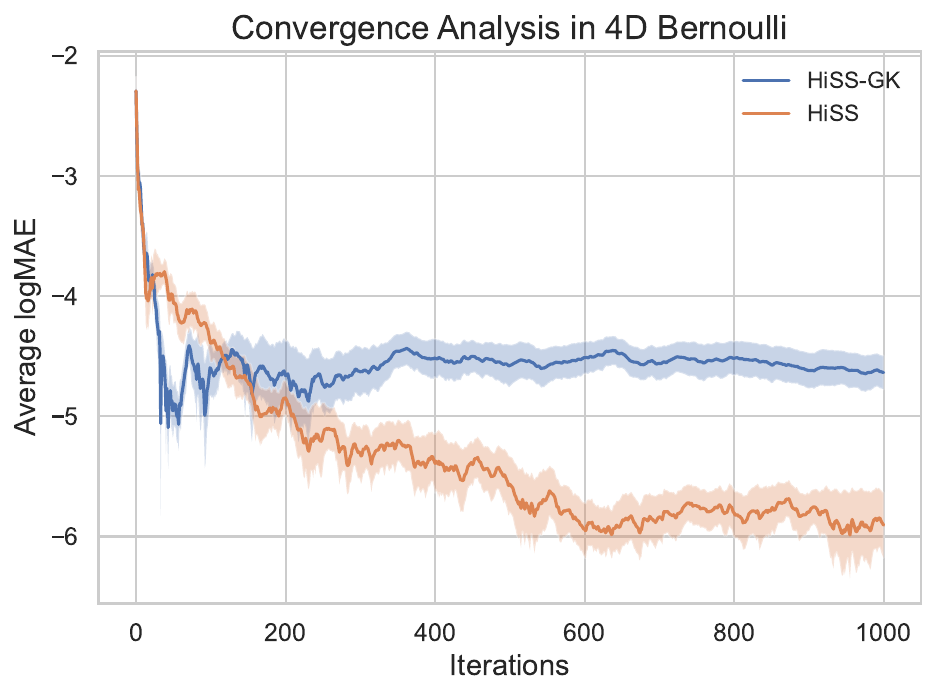}
    \end{minipage}\hfill
    \begin{minipage}[t]{0.48\columnwidth}
        \centering
        \includegraphics[width=\linewidth,height=0.155\textheight]{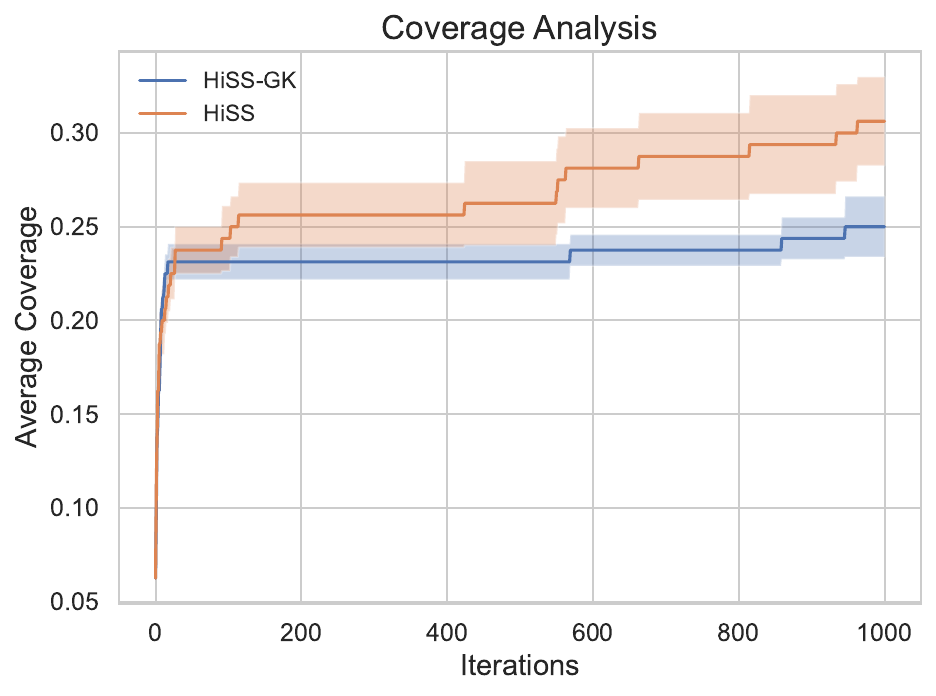}
    \end{minipage}
    \caption{Gaussian Kernel vs Logistic Kernel}
    \label{fig:gaussvslog}
\end{figure}
We conduct an Ablation Study for HiSS using 4D Joint Bernoulli task under the same conditions as Section~\ref{sec:exp}.
\textbf{Gaussian Kernel vs Logistic Kernel.}
To empirically assess the benefits of the Logistic Kernel over the Gaussian Kernel, we replace HiSS’s logistic convolution with Gaussian convolution and call it ‘HiSS with Gaussian Kernel’ (Hiss-GK). Setting $\sigma^2 = 0.9$ for Hiss-GK, HiSS converges faster than HiSS-GK, as shown by the rapid decrease in log mean absolute error (logMAE) in Figure \ref{fig:gaussvslog}. HiSS also achieves higher coverage of the target distribution across iterations. The random walk acceptance ratios for HiSS and HiSS-GK are similar, with values of $0.136 \pm 0.109$ and $0.144 \pm 0.111$, respectively. This improvement is due to the logistic kernel’s superior mode-bridging capability, which facilitates better mixing and exploration of disconnected modes.

\begin{figure}[t]
    \centering
    \begin{minipage}[t]{0.48\columnwidth}
        \centering
        \includegraphics[width=\linewidth,height=0.155\textheight]{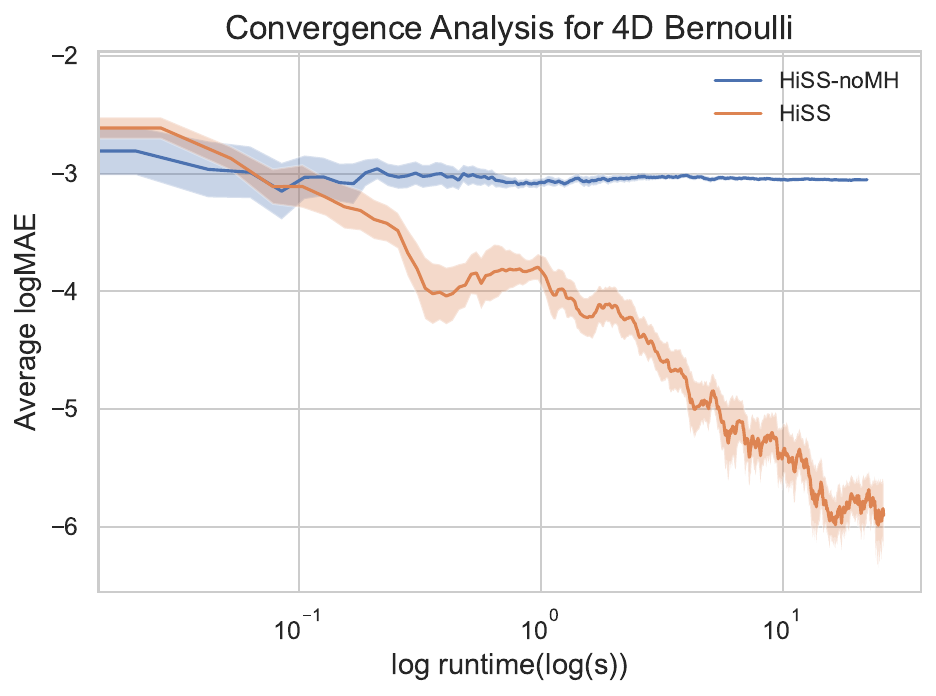}
    \end{minipage}\hfill
    \begin{minipage}[t]{0.48\columnwidth}
        \centering
        \includegraphics[width=\linewidth,height=0.155\textheight]{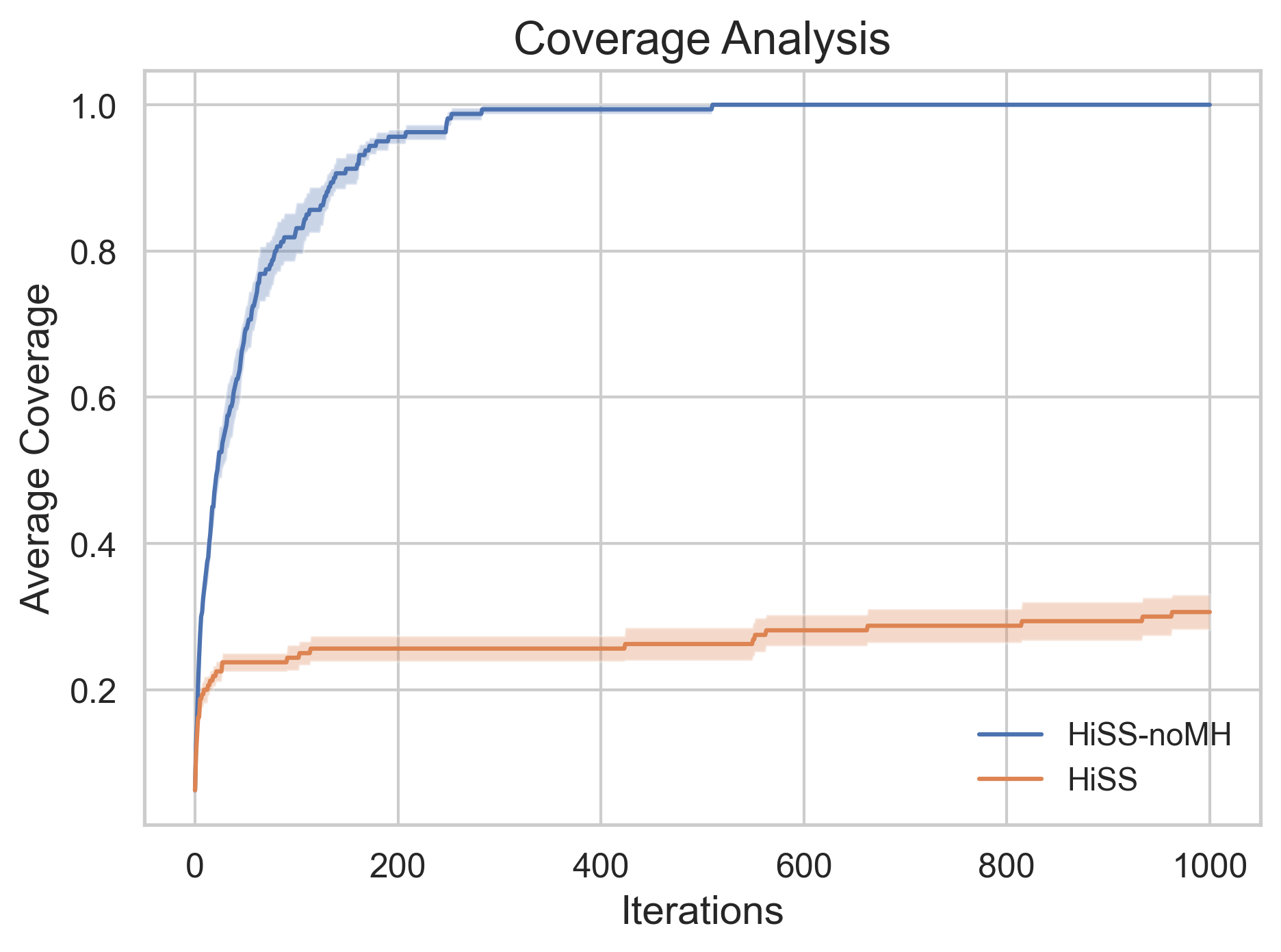}
    \end{minipage}
    \caption{MH vs no-MH}
    \label{fig:no-MH}
\end{figure}

\textbf{MH Step for Denoised Proposal. }

From Figure \ref{fig:no-MH}, skipping this step reduces overall execution time, but it causes markedly early poor convergence to the marginal distribution. The MH step for the denoised proposal ensures convergence to the marginal distribution (Proposition \ref{prop:db}). Since we accept the proposed state without MH correction, coverage quickly reaches 1, unlike the sampler that uses MH correction to select states and evaluate exploration potential. We use DMALA as the base sampler for both methods.

\textbf{Effect of $\eta$. } The scaling parameter $\eta$ is arguably the most crucial hyperparameter for HiSS. We report the average logMAE, MwG acceptance probability, and coverage (defined in Appendix \ref{def:cov}) across various $\eta$ values (Figure \ref{fig:sa}). We observe that for smaller $\eta$, HiSS exhibits high acceptance rates, poor coverage, and higher logMAE, as the sampler tends to retain the current state after the noise-denoise procedure, limiting exploration. As $\eta$ increases, HiSS explores the state space more effectively, leading to improved coverage and convergence (lower logMAE), albeit at the cost of reduced acceptance probability due to more targeted, aggressive proposals. Notably, the standard deviation in logMAE also increases with larger $\eta$, reflecting more dynamic movement through the state space rather than stagnation. Beyond a certain threshold of $\eta$ (approx $1$ here), the performance across all metrics saturates. We provide insights into tuning $G$ and $L$ in the Appendix \ref{sec:add:gl_tune}.
\begin{wrapfigure}{r}{0.32\textwidth}
    \vspace{0.2em}
  \centering
    \includegraphics[width=0.32\textwidth,height=0.155\textheight]{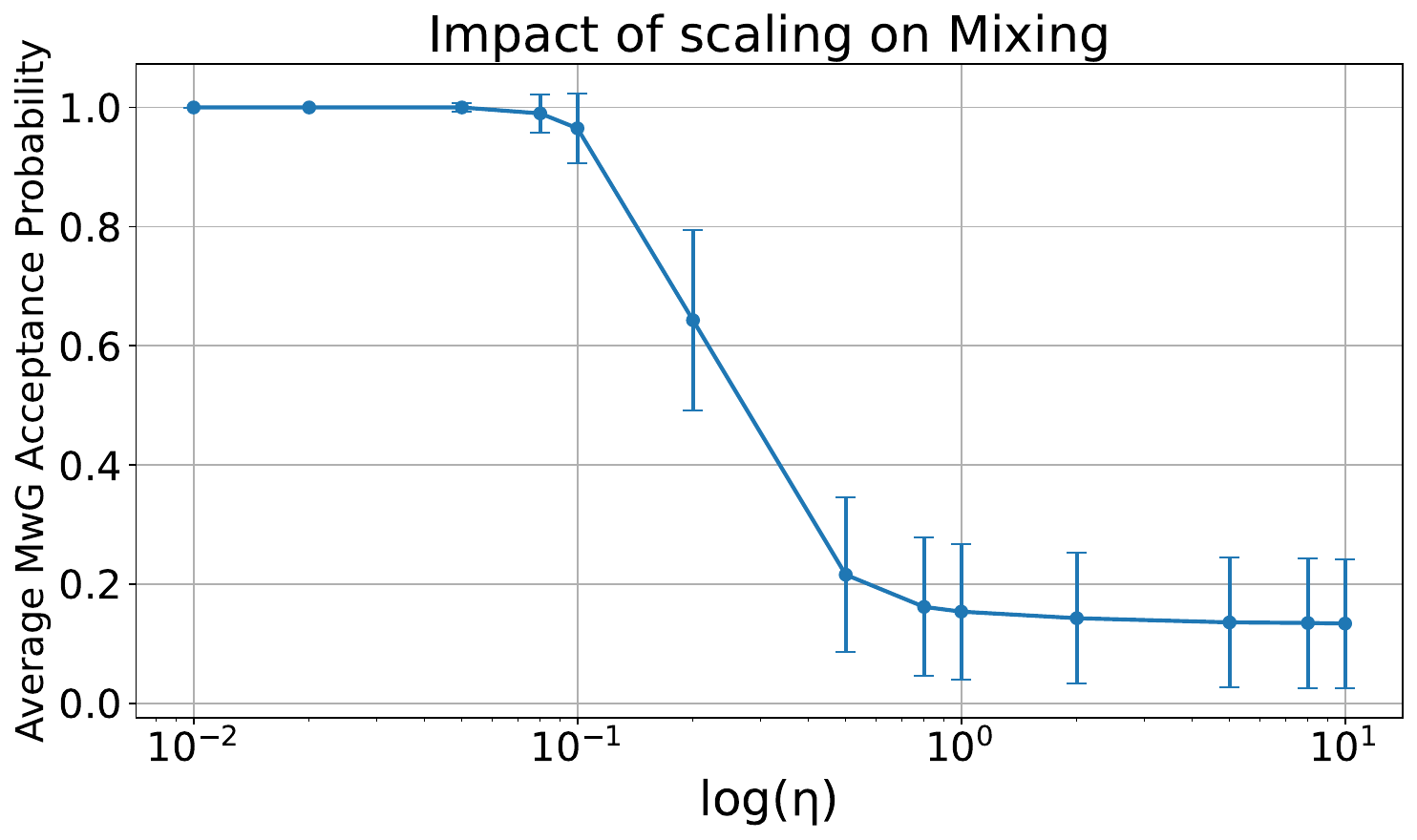}
    \hfill
    \includegraphics[width=0.32\textwidth,height=0.155\textheight]{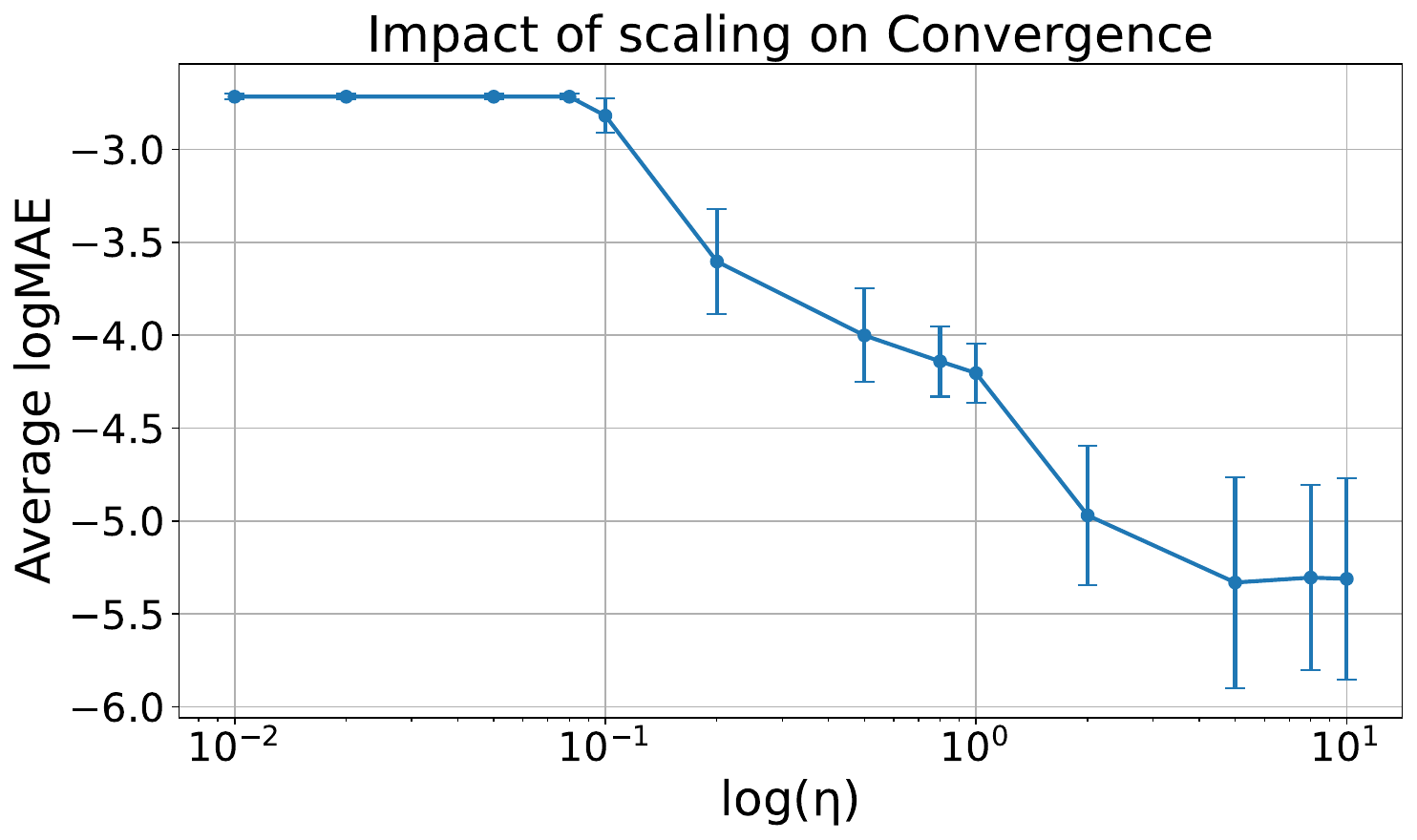}
    \hfill
    \includegraphics[width=0.32\textwidth,height=0.155\textheight]{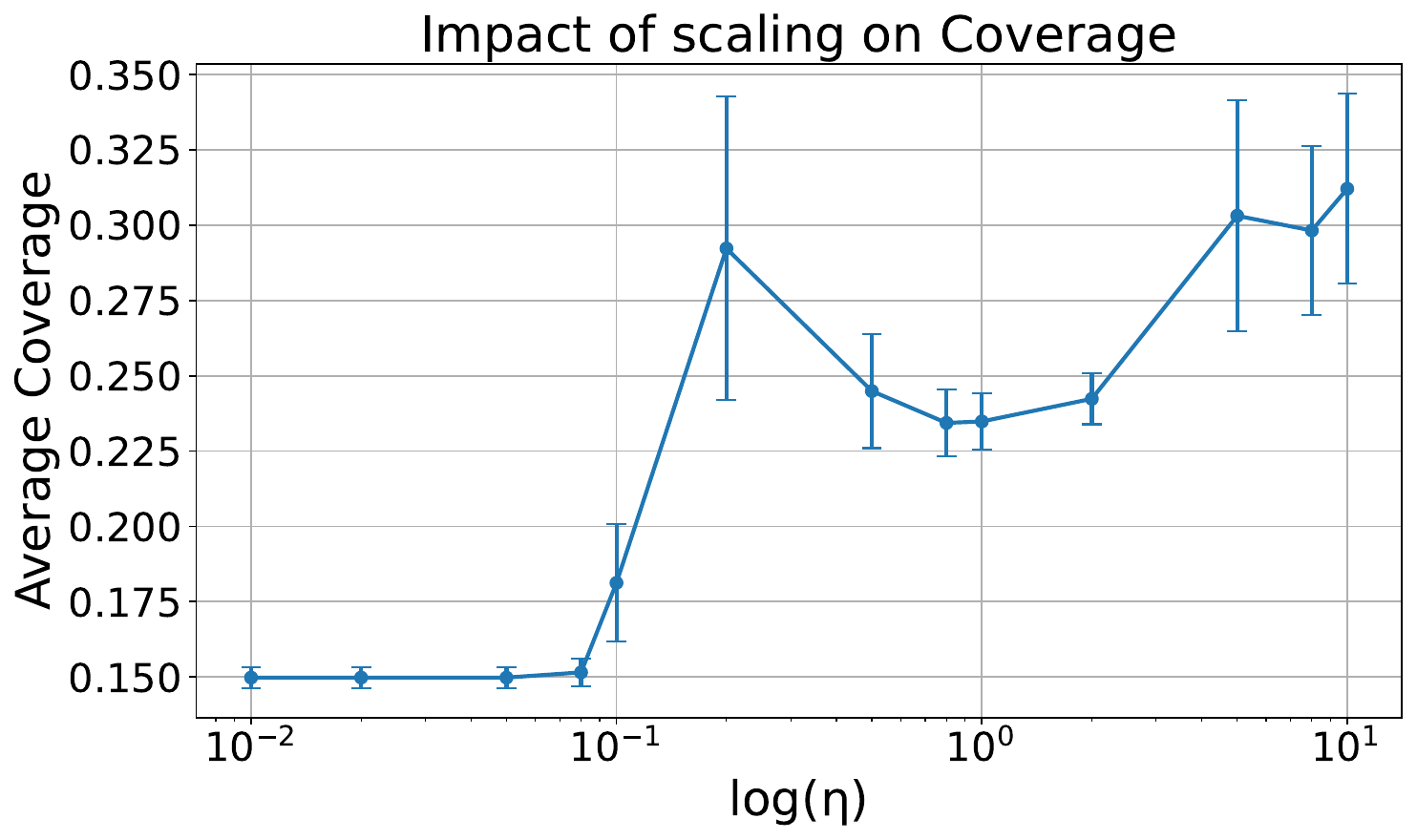}
  \caption{Sensitivity analysis of HiSS.}
  \label{fig:sa}
 \vspace{-1em}
\end{wrapfigure}

\textbf{Impact of gradient refinement.}

\begin{wrapfigure}{r}{0.32\textwidth}
 \vspace{-0.1em}
    \centering
    \includegraphics[width=0.32\textwidth]{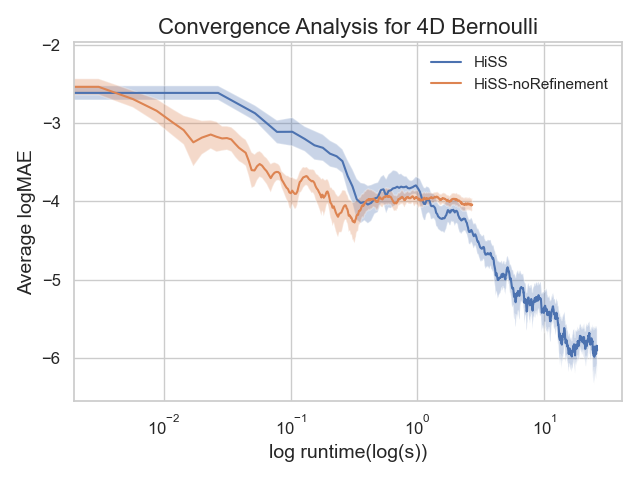}
    \caption{No Gradient Refinement for 4D Bernoulli}
    \label{fig:bern_noL}
 \vspace{-0.8em}
\end{wrapfigure}

We investigate the behavior of HiSS when the gradient-based refinement step is omitted (i.e., setting $L=0$). Theoretically, the resulting sampler, now consisting solely of the Noising, Denoising, and MH correction remains a valid MCMC kernel that satisfies detailed balance with respect to the marginal distribution $\pi(\bm\theta)$. However, removing the gradient step fundamentally alters the sampler's exploration dynamics. Without gradient-informed updates (e.g., DMALA), the sampler loses its ability to exploit the local geometry of the energy landscape. The exploration now reduces to a random-walk behavior governed solely by the logistic kernel. While the kernel facilitates inter-mode jumps (global exploration), it is inefficient at intra-mode mixing (local exploration), particularly in high-dimensional spaces where the volume of the mode is large. We observed that the $L=0$ variant exhibits rapid early convergence due to the reduced computational overhead (no gradient evaluations). However, this advantage is quickly negated as the sampler struggles to thoroughly explore the discovered modes as seen in Figure \ref{fig:bern_noL}.

\section{Conclusion}
\textbf{Limitations. }While effective, HiSS faces several challenges. First, compared to any gradient-based sampler ran for $LG$ steps, HiSS requires additional $G$ MH steps, thereby slightly increasing runtime. Second, the denoised sample's MH acceptance rate, after tuning, remains low 13-14\%. Designing asymmetric (e.g., Gumbel or skewed distributions), mode-aware intelligent proposals through landscape scouting can enhance efficiency by directing proposals to promising regions, potentially bypassing the additional MH step and reducing runtime. Third, our convergence analysis and theoretical guarantees are based solely on DMALA as the base refinement step. Performing a comparative analysis using other gradient-informed samplers (e.g., ACS, DULA, GWG) could provide additional insights. 

Most importantly, HiSS is designed on the premise that gradient-based methods are effective for local refinement, not global exploration. By introducing an auxiliary variable, we decouple these roles: the Logistic Kernel facilitates global mode-hopping, while the inner gradient sampler (e.g., DMALA) handles intra-mode mixing. This allows the system to escape isolated modes while maintaining high acceptance rates within basins. However, this hierarchical structure creates a dependency on the inner sampler. As analyzed by \citet{chehab2024practicaldiffusionpathsampling}, if the local basins exhibit extreme ruggedness or lack local log-concavity, making the local problem as difficult as the global one, the difficulty is shifted to the inner sampler, failing our two-step construction. HiSS assumes that the target landscape consists of basins that are locally amenable to gradient-based sampling.

\textbf{Discussion. } We introduce Hyperbolic Secant-Squared Gibbs Sampling (HiSS), a novel approach for exploring multimodal distributions in discrete spaces, especially those with distant modes. HiSS improves mixing efficiency through a logistic convolution kernel, enhancing the characterization of complex distributions. We provide both asymptotic and non-asymptotic convergence guarantees. Extensive experiments across spin glass systems, Bayesian inference, and combinatorial optimization show HiSS consistently outperforms existing sampling techniques. These findings highlight HiSS’s potential for studying complex discrete distributions and its applications in various scientific fields.

\newpage
\section*{Acknowledgment}
The authors thank Dr. Rajiv Khanna for his helpful discussions, constructive feedback, and suggestions that improved the clarity and presentation of this work. RZ acknowledges support from NSF IIS-2508145 and Amazon Research Award.

\bibliography{main}

@InProceedings{pmlr-v161-khanna21a,
  title = 	 {Geometric rates of convergence for kernel-based sampling algorithms},
  author =       {Khanna, Rajiv and Hodgkinson, Liam and Mahoney, Michael W.},
  booktitle = 	 {Proceedings of the Thirty-Seventh Conference on Uncertainty in Artificial Intelligence},
  pages = 	 {2156--2164},
  year = 	 {2021},
  editor = 	 {de Campos, Cassio and Maathuis, Marloes H.},
  volume = 	 {161},
  series = 	 {Proceedings of Machine Learning Research},
  month = 	 {27--30 Jul},
  publisher =    {PMLR},
  url = 	 {https://proceedings.mlr.press/v161/khanna21a.html}
}

@inproceedings{koyejo2014prior,
  title={On Prior Distributions and Approximate Inference for Structured Variables},
  author={Koyejo, Oluwasanmi O and Khanna, Rajiv and Ghosh, Joydeep and Poldrack, Russell},
  booktitle={Advances in Neural Information Processing Systems 27 (NIPS 2014)},
  pages={172--180},
  year={2014},
  editor={Z. Ghahramani and M. Welling and C. Cortes and N. D. Lawrence and K. Q. Weinberger},
  publisher={Curran Associates, Inc.}
}

@inproceedings{zhangcyclical,
  title={Cyclical Stochastic Gradient MCMC for Bayesian Deep Learning},
  author={Zhang, Ruqi and Li, Chunyuan and Zhang, Jianyi and Chen, Changyou and Wilson, Andrew Gordon},
  booktitle={International Conference on Learning Representations},
  year={2020}
}

@article{kone2005selection,
  title={Selection of temperature intervals for parallel-tempering simulations},
  author={Kone, Alexandre and Kofke, David A},
  journal={The Journal of Chemical Physics},
  volume={122},
  number={20},
  pages={206101},
  year={2005},
  publisher={AIP Publishing},
  doi={10.1063/1.1926935}
}

@article{jones2004markov,
author = {Galin L. Jones},
title = {{On the Markov chain central limit theorem}},
volume = {1},
journal = {Probability Surveys},
number = {none},
publisher = {Institute of Mathematical Statistics and Bernoulli Society},
pages = {299 -- 320},
year = {2004},
doi = {10.1214/154957804100000051},
URL = {https://doi.org/10.1214/154957804100000051}
}

@article{marinari1992simulated,
  title={Simulated tempering: a new Monte Carlo scheme},
  author={Marinari, Enzo and Parisi, Giorgio},
  journal={Europhysics letters},
  volume={19},
  number={6},
  pages={451},
  year={1992},
  publisher={IOP Publishing}
}

@article{wang2001efficient,
  title={Efficient, multiple-range random walk algorithm to calculate the density of states},
  author={Wang, Fugao and Landau, David P},
  journal={Physical review letters},
  volume={86},
  number={10},
  pages={2050},
  year={2001},
  publisher={APS}
}

@article{madras2003markov,
  title={Markov chain importance sampling},
  author={Madras, Neal and Zheng, Dana},
  journal={Journal of Statistical Physics},
  volume={112},
  number={1--2},
  pages={293--312},
  year={2003},
  publisher={Springer}
}

@article{swendsen1987nonuniversal,
  title={Nonuniversal critical dynamics in Monte Carlo simulations},
  author={Swendsen, Robert H and Wang, Jian-Sheng},
  journal={Physical review letters},
  volume={58},
  number={2},
  pages={86},
  year={1987},
  publisher={APS}
}

@article{wolff1989collective,
  title={Collective Monte Carlo updating for spin systems},
  author={Wolff, Ulli},
  journal={Physical Review Letters},
  volume={62},
  number={4},
  pages={361},
  year={1989},
  publisher={APS}
}

@book{meyn2009markov,
  title={Markov Chains and Stochastic Stability},
  author={Meyn, Sean and Tweedie, Richard L.},
  year={2009},
  publisher={Cambridge University Press},
  edition={2nd}
}

@incollection{lapedes1999correlated,
  title={Correlated mutations in models of protein sequences: phylogenetic and structural effects},
  author={Lapedes, A. S. and Giraud, B. G. and Liu, L. and Stormo, G. D.},
  booktitle={Lecture Notes--Monograph Series},
  pages={236--256},
  year={1999},
  publisher={Institute of Mathematical Statistics}
}

@book{nummelin1984general,
  title={General Irreducible Markov Chains and Non-negative Operators},
  author={Nummelin, Esa},
  year={1984},
  publisher={Cambridge University Press}
}

@inproceedings{izmailov2021bayesian,
  title={What are Bayesian neural network posteriors really like?},
  author={Izmailov, Pavel and Vikram, Sharad and Hoffman, Matthew D and Wilson, Andrew Gordon Gordon},
  booktitle={International conference on machine learning},
  pages={4629--4640},
  year={2021},
  organization={PMLR}
}

@book{bishop2006pattern,
  title={Pattern recognition and machine learning},
  author={Bishop, Christopher M},
  volume={4},
  year={2006},
  publisher={Springer}
}

@article{deng2020contour,
  title={A contour stochastic gradient langevin dynamics algorithm for simulations of multi-modal distributions},
  author={Deng, Wei and Lin, Guang and Liang, Faming},
  journal={Advances in neural information processing systems},
  volume={33},
  pages={15725--15736},
  year={2020}
}

@inproceedings{sohl2015deep,
  title={Deep unsupervised learning using nonequilibrium thermodynamics},
  author={Sohl-Dickstein, Jascha and Weiss, Eric and Maheswaranathan, Niru and Ganguli, Surya},
  booktitle={International conference on machine learning},
  pages={2256--2265},
  year={2015},
  organization={PMLR}
}

@article{zanella2020informed,
  title={Informed proposals for local MCMC in discrete spaces},
  author={Zanella, Giacomo},
  journal={Journal of the American Statistical Association},
  volume={115},
  number={530},
  pages={852--865},
  year={2020},
  publisher={Taylor \& Francis}
}

@article{neal2011mcmc,
  title={MCMC using Hamiltonian dynamics},
  author={Neal, Radford M and others},
  journal={Handbook of markov chain monte carlo},
  volume={2},
  number={11},
  pages={2},
  year={2011}
}

@article{grathwohl2021oops,
  title={Oops I Took A Gradient: Scalable Sampling for Discrete Distributions},
  author={Grathwohl, Will and Swersky, Kevin and Hashemi, Milad and Duvenaud, David and Maddison, Chris J},
  journal={International Conference on Machine Learning},
  year={2021}
}

@article{song2019generative,
  title={Generative modeling by estimating gradients of the data distribution},
  author={Song, Yang and Ermon, Stefano},
  journal={Advances in Neural Information Processing Systems},
  year={2019}
}

@inproceedings{zhang2019cyclical,
  title={Cyclical Stochastic Gradient MCMC for Bayesian Deep Learning},
  author={Zhang, Ruqi and Li, Chunyuan and Zhang, Jianyi and Chen, Changyou and Wilson, Andrew Gordon},
  booktitle={International Conference on Learning Representations},
  year={2020}
}

@article{courbariaux2016binarized,
  title={Binarized neural networks: Training deep neural networks with weights and activations constrained to+ 1 or-1},
  author={Courbariaux, Matthieu and Hubara, Itay and Soudry, Daniel and El-Yaniv, Ran and Bengio, Yoshua},
  journal={arXiv preprint arXiv:1602.02830},
  year={2016}
}

@inproceedings{liu2021post,
  title     = {Post-training Quantization with Multiple Points: Mixed Precision without Mixed Precision},
  author    = {Liu, Xingchao and Ye, Mao and Zhou, Dengyong and Liu, Qiang},
  booktitle = {Proceedings of the AAAI Conference on Artificial Intelligence},
  pages     = {8697--8705},
  year      = {2021}
}

@inproceedings{rastegari2016xnor,
  title={Xnor-net: Imagenet classification using binary convolutional neural networks},
  author={Rastegari, Mohammad and Ordonez, Vicente and Redmon, Joseph and Farhadi, Ali},
  booktitle={European conference on computer vision},
  pages={525--542},
  year={2016},
  organization={Springer}
}

@article{kingma2013auto,
  title={Auto-encoding variational bayes},
  author={Kingma, Diederik P and Welling, Max},
  journal={arXiv preprint arXiv:1312.6114},
  year={2013}
}

@inproceedings{gu2018non,
  title={Non-Autoregressive Neural Machine Translation},
  author={Gu, Jiatao and Bradbury, James and Xiong, Caiming and Li, Victor OK and Socher, Richard},
  booktitle={International Conference on Learning Representations},
  year={2018}
}

@article{compas,
  title={Machine bias: There’s software used across the country to predict future criminals. And it’s biased against blacks},
  author={J. Angwin, J. Larson, S. Mattu and L. Kirchner},
  journal={ProPublica},
  year={2016}
}

@misc{breastcancer,
  author       = {William H. Wolberg and Olvi L. Mangasarian and Nick Street and W. Street},
  title        = {{Breast Cancer Wisconsin (Diagnostic)}},
  year         = {1993},
  howpublished = {UCI Machine Learning Repository},
  note         = {{DOI}: \url{https://doi.org/10.24432/C5DW2B}}
}

@misc{Dua:2019 ,
author = "Dua, Dheeru and Graff, Casey",
year = "2017",
title = "{UCI} Machine Learning Repository",
url = "http://archive.ics.uci.edu/ml",
institution = "University of California, Irvine, School of Information and Computer Sciences" }

@inproceedings{sun2021path,
  title={Path Auxiliary Proposal for MCMC in Discrete Space},
  author={Sun, Haoran and Dai, Hanjun and Xia, Wei and Ramamurthy, Arun},
  booktitle={International Conference on Learning Representations},
  year={2022}
}

@article{lewis2020bart,
  title={BART: Denoising sequence-to-sequence pre-training for natural language generation, translation, and comprehension},
  author={Lewis, Mike and Liu, Yinhan and Goyal, Naman and Ghazvininejad, Marjan and Mohamed, Abdelrahman and Levy, Omer and Stoyanov, Veselin and Zettlemoyer, Luke},
  journal={arXiv preprint arXiv:1910.13461},
  year={2020}
}

@article{devlin2019bert,
  title={BERT: Pre-training of deep bidirectional transformers for language understanding},
  author={Devlin, Jacob and Chang, Ming-Wei and Lee, Kenton and Toutanova, Kristina},
  journal={arXiv preprint arXiv:1810.04805},
  year={2019}
}

@inproceedings{zhang2022langevin,
  title={A Langevin-like sampler for discrete distributions},
  author={Zhang, Ruqi and Liu, Xingchao and Liu, Qiang},
  booktitle={International Conference on Machine Learning},
  pages={26375--26396},
  year={2022},
  organization={PMLR}
}

@article{
rhodes2022enhanced,
title={Enhanced gradient-based {MCMC} in discrete spaces},
author={Benjamin Rhodes and Michael U. Gutmann},
journal={Transactions on Machine Learning Research},
issn={2835-8856},
year={2022},
note={}
}

@inproceedings{
  song2021scorebased,
  title={Score-Based Generative Modeling through Stochastic Differential Equations},
  author={Yang Song and Jascha Sohl-Dickstein and Diederik P Kingma and Abhishek Kumar and Stefano Ermon and Ben Poole},
  booktitle={International Conference on Learning Representations},
  year={2021},
  url={https://openreview.net/forum?id=PxTIG12RRHS}
}

@article{
doi:10.1126/science.aaw1147,
author = {Frank Noé  and Simon Olsson  and Jonas Köhler  and Hao Wu },
title = {Boltzmann generators: Sampling equilibrium states of many-body systems with deep learning},
journal = {Science},
volume = {365},
number = {6457},
pages = {eaaw1147},
year = {2019},
doi = {10.1126/science.aaw1147},
URL = {https://www.science.org/doi/abs/10.1126/science.aaw1147},
eprint = {https://www.science.org/doi/pdf/10.1126/science.aaw1147}}

@article{banterle2015acceleratingmetropolishastingsalgorithmsdelayed,
title = {Accelerating Metropolis-Hastings algorithms by Delayed Acceptance},
journal = {Foundations of Data Science},
volume = {1},
number = {2},
pages = {103-128},
year = {2019},
issn = {},
doi = {10.3934/fods.2019005},
url = {https://www.aimsciences.org/article/id/1de5e91b-73f5-46cb-a8f1-b92f874250ac},
author = {Marco Banterle and Clara Grazian and Anthony Lee and Christian P. Robert}
}

@inproceedings{sun2023discrete,
  title={Discrete Langevin Samplers via Wasserstein Gradient Flow},
  author={Sun, Haoran and Dai, Hanjun and Dai, Bo and Zhou, Haomin and Schuurmans, Dale},
  booktitle={International Conference on Artificial Intelligence and Statistics},
  pages={6290--6313},
  year={2023},
  organization={PMLR}
}

@misc{2024diffusive,
      title={Diffusive Gibbs Sampling}, 
      author={Wenlin Chen and Mingtian Zhang and Brooks Paige and José Miguel Hernández-Lobato and David Barber},
      year={2024},
      archivePrefix={arXiv},
}

@inproceedings{ho2020denoising,
  title={Denoising diffusion probabilistic models},
  author={Ho, Jonathan and Jain, Ajay and Abbeel, Pieter},
  booktitle={Advances in Neural Information Processing Systems},
  volume={33},
  pages={6840--6851},
  year={2020}
}

@article{roberts2002langevin,
  title={Langevin diffusions and Metropolis-Hastings algorithms},
  author={Roberts, Gareth O. and Rosenthal, Jeffrey S.},
  journal={Methodology and Computing in Applied Probability},
  volume={4},
  number={4},
  pages={337--357},
  year={2002},
  publisher={Springer}
}

@inproceedings{
sun2022path,
title={Path Auxiliary Proposal for {MCMC} in Discrete Space},
author={Haoran Sun and Hanjun Dai and Wei Xia and Arun Ramamurthy},
booktitle={International Conference on Learning Representations},
year={2022}
}

@article{pompe2020framework,
  title={A framework for adaptive MCMC targeting multimodal distributions},
  author={Emilia Pompe and Chris C. Holmes and Krzysztof Latuszy'nski},
  journal={The Annals of Statistics},
  year={2020}
}

@article{betancourt2017conceptual,
  title={A conceptual introduction to Hamiltonian Monte Carlo},
  author={Betancourt, Michael},
  journal={arXiv preprint arXiv:1701.02434},
  year={2017}
}

@article{livingstone2019geometric,
  title={Geometric foundations of Hamiltonian Monte Carlo},
  author={Livingstone, Samuel and Betancourt, Michael and Byrne, Simon and Girolami, Mark},
  journal={Bernoulli},
  volume={25},
  number={4A},
  pages={2257--2298},
  year={2019},
  publisher={Bernoulli Society for Mathematical Statistics and Probability}
}

@inproceedings{
pynadath2024gradientbased,
title={Gradient-based Discrete Sampling with Automatic Cyclical Scheduling},
author={Patrick Pynadath and Riddhiman Bhattacharya and ARUN NARAYANAN HARIHARAN and Ruqi Zhang},
booktitle={ICML 2024 Workshop on Structured Probabilistic Inference {\&} Generative Modeling},
year={2024},
url={https://openreview.net/forum?id=aTDId2TrtL}
}

@article{swendsen1986replica,
  title={Replica Monte Carlo simulation of spin-glasses},
  author={Swendsen, Robert H and Wang, Jian-Sheng},
  journal={Physical Review Letters},
  volume={57},
  number={21},
  pages={2607},
  year={1986},
  publisher={APS}
}

@article{berg1991multicanonical,
  title={Multicanonical ensemble: A new approach to simulate first-order phase transitions},
  author={Berg, Bernd A and Neuhaus, Thomas},
  journal={Physical Review Letters},
  volume={68},
  number={1},
  pages={9--12},
  year={1992},
  publisher={APS}
}

@inproceedings{grathwohl2021gwg,
  title={Oops I took a gradient: Scalable sampling for discrete distributions},
  author={Grathwohl, Will and Wang, Milad Hashemi and Liao, Renjie and Swersky, Kevin and Duvenaud, David},
  booktitle={Proceedings of the 38th International Conference on Machine Learning (ICML)},
  year={2021}
}

@misc{xiang2023efficient,
      title={Efficient Informed Proposals for Discrete Distributions via Newton's Series Approximation}, 
      author={Yue Xiang and Dongyao Zhu and Bowen Lei and Dongkuan Xu and Ruqi Zhang},
      year={2023},
      eprint={2302.13929},
      archivePrefix={arXiv},
      primaryClass={cs.LG},
      url={https://arxiv.org/abs/2302.13929}, 
}

@article{aboelhadid2018logistic,
  title={Logistic kernel density estimation: Statistical properties and optimal bandwidth selection},
  author={Aboelhadid, Mohamed and others},
  journal={International Journal of Computational Mathematical Sciences},
  year={2018},
  url={https://www.m-hikari.com/ijcms/ijcms-2018/5-8-2018/p/aboelhadidIJCMS5-8-2018.pdf}
}

@book{feller1971introduction,
  title={An Introduction to Probability Theory and Its Applications},
  author={Feller, William},
  volume={2},
  year={1971},
  publisher={Wiley}
}

@article{huber1964robust,
  title={Robust estimation of a location parameter},
  author={Huber, Peter J},
  journal={Annals of Mathematical Statistics},
  volume={35},
  number={1},
  pages={73--101},
  year={1964},
  publisher={Institute of Mathematical Statistics}
}

@article{maddison2017concrete,
  title={The Concrete Distribution: A Continuous Relaxation of Discrete Random Variables},
  author={Maddison, Chris J and Mnih, Andriy and Teh, Yee Whye},
  journal={arXiv preprint arXiv:1611.00712},
  year={2017}
}

@article{edwards1975theory,
  title={Theory of spin glasses},
  author={Edwards, S F and Anderson, P W},
  journal={Journal of Physics F: Metal Physics},
  volume={5},
  number={5},
  pages={965},
  year={1975},
  publisher={IOP Publishing},
  doi={10.1088/0305-4608/5/5/017}
}

@book{chaikin1995principles,
  title={Principles of Condensed Matter Physics},
  author={Chaikin, P M and Lubensky, T C},
  year={1995},
  publisher={Cambridge University Press},
  isbn={9780521319728},
  doi={10.1017/CBO9780511813467}
}

@book{robert1999monte,
  title={Monte Carlo Statistical Methods},
  author={Robert, Christian P. and Casella, George},
  year={1999},
  publisher={Springer},
  address={New York}
}

@article{tierney1994markov,
  title={Markov chains for exploring posterior distributions},
  author={Tierney, Luke},
  journal={The Annals of Statistics},
  volume={22},
  number={4},
  pages={1701--1728},
  year={1994},
  publisher={Institute of Mathematical Statistics}
}

@article{roberts1994simple,
  title={Simple conditions for the convergence of the Metropolis-Hastings algorithm},
  author={Roberts, Gareth O. and Tweedie, Richard L.},
  journal={Stochastic Processes and their Applications},
  volume={49},
  number={2},
  pages={207--216},
  year={1994},
  publisher={Elsevier}
}

@article{earl2005parallel,
  title={Parallel tempering: Theory, applications, and new perspectives},
  author={Earl, David J and Deem, Michael W},
  journal={Physical Chemistry Chemical Physics},
  volume={7},
  number={23},
  pages={3910--3916},
  year={2005},
  publisher={Royal Society of Chemistry}
}

@inproceedings{vinyals2015pointer,
  title={Pointer Networks},
  author={Vinyals, Oriol and Fortunato, Meire and Jaitly, Navdeep},
  booktitle={Advances in Neural Information Processing Systems},
  pages={2692--2700},
  year={2015},
  url={https://papers.nips.cc/paper/5866-pointer-networks}
}

@article{durmus2017nonasymptotic,
  title={Nonasymptotic convergence analysis for the unadjusted Langevin algorithm},
  author={Durmus, Alain and Moulines, Eric},
  journal={The Annals of Applied Probability},
  volume={27},
  number={3},
  pages={1551--1587},
  year={2017},
  publisher={Institute of Mathematical Statistics}
}

@article{bottou2018optimization,
  title={Optimization methods for large-scale machine learning},
  author={Bottou, L{\'e}on and Curtis, Frank E and Nocedal, Jorge},
  journal={Siam Review},
  volume={60},
  number={2},
  pages={223--311},
  year={2018},
  publisher={SIAM}
}

@inproceedings{dalalyan2017further,
  title={Further and stronger analogy between sampling and optimization: Langevin Monte Carlo and gradient descent},
  author={Dalalyan, Arnak},
  booktitle={Conference on Learning Theory},
  pages={678--689},
  year={2017},
  organization={PMLR}
}

@inproceedings{
mohanty2025entropy,
title={Entropy-Guided Sampling of Flat Modes in Discrete Spaces},
author={Pinaki Mohanty and Riddhiman Bhattacharya and Ruqi Zhang},
booktitle={NeurIPS 2025 Workshop on Structured Probabilistic Inference {\&} Generative Modeling},
year={2025},
url={https://openreview.net/forum?id=i6qthf42ad}
}

@InProceedings{bossek2020comprehensive,
author="Zaefferer, Martin
and Stork, J{\"o}rg
and Bartz-Beielstein, Thomas",
editor="Bartz-Beielstein, Thomas
and Branke, J{\"u}rgen
and Filipi{\v{c}}, Bogdan
and Smith, Jim",
title="Distance Measures for Permutations in Combinatorial Efficient Global Optimization",
booktitle="Parallel Problem Solving from Nature -- PPSN XIII",
year="2014",
publisher="Springer International Publishing",
address="Cham",
pages="373--383",
isbn="978-3-319-10762-2"
}

@misc{chehab2024practicaldiffusionpathsampling,
      title={A Practical Diffusion Path for Sampling}, 
      author={Omar Chehab and Anna Korba},
      year={2024},
      eprint={2406.14040},
      archivePrefix={arXiv},
      primaryClass={stat.ML},
      url={https://arxiv.org/abs/2406.14040}, 
}

@article{li2022circular,
  title={Circular Jaccard Distance-Based Multi-Solution Optimization for Traveling Salesman Problems},
  author={Li, Hui and Zhang, Mengyao and Zeng, Chenbo},
  journal={Mathematical Biosciences and Engineering},
  volume={19},
  number={5},
  pages={4458--4480},
  year={2022},
  publisher={AIMS Press}
}

@article{fagin2003comparing,
  title={Comparing top k lists},
  author={Fagin, Ronald and Kumar, Ravi and Sivakumar, Dandapani},
  journal={SIAM Journal on Discrete Mathematics},
  volume={17},
  number={1},
  pages={134--160},
  year={2003},
  publisher={SIAM}
}

@article{helgemo2021evaluating,
title = {Evolutionary algorithm to traveling salesman problems},
journal = {Computers \& Mathematics with Applications},
volume = {64},
number = {5},
pages = {788-797},
year = {2012},
note = {Advanced Technologies in Computer, Consumer and Control},
issn = {0898-1221},
doi = {https://doi.org/10.1016/j.camwa.2011.12.018},
url = {https://www.sciencedirect.com/science/article/pii/S089812211101073X},
author = {Yen-Far Liao and Dun-Han Yau and Chieh-Li Chen}
}

@article{George01091993,
author = {Edward I. George and Robert E. McCulloch},
title = {Variable Selection via Gibbs Sampling},
journal = {Journal of the American Statistical Association},
volume = {88},
number = {423},
pages = {881--889},
year = {1993},
publisher = {ASA Website},
doi = {10.1080/01621459.1993.10476353}
}

@article{tibshirani1996regression,
  title={Regression shrinkage and selection via the lasso},
  author={Tibshirani, Robert},
  journal={Journal of the Royal Statistical Society: Series B (Methodological)},
  volume={58},
  number={1},
  pages={267--288},
  year={1996},
  publisher={Wiley Online Library}
}

@article{park2008bayesian,
  title={The Bayesian lasso},
  author={Park, Trevor and Casella, George},
  journal={Journal of the American statistical association},
  volume={103},
  number={482},
  pages={681--686},
  year={2008},
  publisher={Taylor \& Francis}
}

@inproceedings{louizos2018learning,
  title={Learning sparse neural networks through L0 regularization},
  author={Louizos, Christos and Welling, Max and Kingma, Diederik P},
  booktitle={International Conference on Learning Representations},
  year={2018}
}

@inproceedings{NEURIPS2024_eb0b13cc,
 author = {Sahoo, Subham Sekhar and Arriola, Marianne and Schiff, Yair and Gokaslan, Aaron and Marroquin, Edgar and Chiu, Justin T and Rush, Alexander and Kuleshov, Volodymyr},
 booktitle = {Advances in Neural Information Processing Systems},
 editor = {A. Globerson and L. Mackey and D. Belgrave and A. Fan and U. Paquet and J. Tomczak and C. Zhang},
 pages = {130136--130184},
 publisher = {Curran Associates, Inc.},
 title = {Simple and Effective Masked Diffusion Language Models},
 url = {https://proceedings.neurips.cc/paper_files/paper/2024/file/eb0b13cc515724ab8015bc978fdde0ad-Paper-Conference.pdf},
 volume = {37},
 year = {2024}
}

@InProceedings{pmlr-v235-lou24a,
  title = 	 {Discrete Diffusion Modeling by Estimating the Ratios of the Data Distribution},
  author =       {Lou, Aaron and Meng, Chenlin and Ermon, Stefano},
  booktitle = 	 {Proceedings of the 41st International Conference on Machine Learning},
  pages = 	 {32819--32848},
  year = 	 {2024},
  editor = 	 {Salakhutdinov, Ruslan and Kolter, Zico and Heller, Katherine and Weller, Adrian and Oliver, Nuria and Scarlett, Jonathan and Berkenkamp, Felix},
  volume = 	 {235},
  series = 	 {Proceedings of Machine Learning Research},
  month = 	 {21--27 Jul},
  publisher =    {PMLR},
  url = 	 {https://proceedings.mlr.press/v235/lou24a.html}
}

@misc{blog,
  author       = {Buza, Krisztian},
  title        = {{BlogFeedback}},
  year         = {2014},
  howpublished = {UCI Machine Learning Repository},
  note         = {{DOI}: https://doi.org/10.24432/C58S3F}
}

@InProceedings{pmlr-v202-avdeyev23a,
  title = 	 {{D}irichlet Diffusion Score Model for Biological Sequence Generation},
  author =       {Avdeyev, Pavel and Shi, Chenlai and Tan, Yuhao and Dudnyk, Kseniia and Zhou, Jian},
  booktitle = 	 {Proceedings of the 40th International Conference on Machine Learning},
  pages = 	 {1276--1301},
  year = 	 {2023},
  editor = 	 {Krause, Andreas and Brunskill, Emma and Cho, Kyunghyun and Engelhardt, Barbara and Sabato, Sivan and Scarlett, Jonathan},
  volume = 	 {202},
  series = 	 {Proceedings of Machine Learning Research},
  month = 	 {23--29 Jul},
  publisher =    {PMLR},
  url = 	 {https://proceedings.mlr.press/v202/avdeyev23a.html}
}

@misc{hiv,
  author       = {Rgnvaldsson, Thorsteinn},
  title        = {{HIV-1 protease cleavage}},
  year         = {2015},
  howpublished = {UCI Machine Learning Repository},
  note         = {{DOI}: https://doi.org/10.24432/C5H03P}
}

@article{reinelt1991tsplib,
  author = {Reinelt, Gerhard},
  title = {{TSPLIB—A Traveling Salesman Problem Library}},
  journal = {ORSA Journal on Computing},
  volume = {3},
  number = {4},
  pages = {376--384},
  year = {1991},
  publisher = {INFORMS},
  doi = {10.1287/ijoc.3.4.376}
}

@book{appelgate2006traveling,
  title={The traveling salesman problem: A computational study},
  author={Applegate, David L and Bixby, Robert E and Chvatal, Vasek and Cook, William J},
  year={2006},
  publisher={Princeton University Press}
}
\bibliographystyle{apalike}
\newpage

\section*{Checklist}

\begin{enumerate}

  \item For all models and algorithms presented, check if you include:
  \begin{enumerate}
    \item A clear description of the mathematical setting, assumptions, algorithm, and/or model. [Yes]
    \item An analysis of the properties and complexity (time, space, sample size) of any algorithm. [Yes]
    \item (Optional) Anonymized source code, with specification of all dependencies, including external libraries. [Yes]
  \end{enumerate}

  \item For any theoretical claim, check if you include:
  \begin{enumerate}
    \item Statements of the full set of assumptions of all theoretical results. [Yes]
    \item Complete proofs of all theoretical results. [Yes]
    \item Clear explanations of any assumptions. [Yes]     
  \end{enumerate}

  \item For all figures and tables that present empirical results, check if you include:
  \begin{enumerate}
    \item The code, data, and instructions needed to reproduce the main experimental results (either in the supplemental material or as a URL). [Yes]
    \item All the training details (e.g., data splits, hyperparameters, how they were chosen). [Yes]
    \item A clear definition of the specific measure or statistics and error bars (e.g., with respect to the random seed after running experiments multiple times). [Yes]
    \item A description of the computing infrastructure used. (e.g., type of GPUs, internal cluster, or cloud provider). [Yes]
  \end{enumerate}

  \item If you are using existing assets (e.g., code, data, models) or curating/releasing new assets, check if you include:
  \begin{enumerate}
    \item Citations of the creator If your work uses existing assets. [Not Applicable]
    \item The license information of the assets, if applicable. [Not Applicable]
    \item New assets either in the supplemental material or as a URL, if applicable. [Not Applicable]
    \item Information about consent from data providers/curators. [Not Applicable]
    \item Discussion of sensible content if applicable, e.g., personally identifiable information or offensive content. [Not Applicable]
  \end{enumerate}

  \item If you used crowdsourcing or conducted research with human subjects, check if you include:
  \begin{enumerate}
    \item The full text of instructions given to participants and screenshots. [Not Applicable]
    \item Descriptions of potential participant risks, with links to Institutional Review Board (IRB) approvals if applicable. [Not Applicable]
    \item The estimated hourly wage paid to participants and the total amount spent on participant compensation. [Not Applicable]
  \end{enumerate}

\end{enumerate}

\clearpage
\appendix

\onecolumn
\aistatstitle{Appendix}
All experiments were run on a single RTX A6000.
 \section{Justification for Logistic Convolutional Kernel\label{sec:app_eta}}
\subsection{Intuition}
Our choice of using logistic convolution over Gaussian convolution stems from the need for a kernel with fatter tails, which ensures better exploration of multimodal energy landscapes. In high-dimensional problems, escaping local modes is essential, and the Gaussian kernel, with its thin tails, heavily centers the probability mass around the current mode, making mode escape less efficient~\citep{bishop2006pattern}. The logistic distribution, on the other hand, exhibits slower tail decay (\(e^{-|x|}\)) than Gaussian (\(e^{-x^2}\))~\citep{feller1971introduction}, enabling it to retain intermediate mass in disconnected regions of the landscape.  We specifically discard other heavy-tailed distributions like the Cauchy and Laplace distributions: the Cauchy distribution lacks finite moments, leading to numerical instabilities, and the Laplace distribution is non-differentiable at its mean, which can hinder gradient-based optimization~\citep{huber1964robust}.

\subsection{Practical Implementation}
Logistic noise can be generated in Python via the inverse transform sampling method, where the standard logistic distribution can be expressed as:
\[
F^{-1}(u) = \ln\left(\frac{u}{1 - u}\right),
\]
with \(u \sim \text{Uniform}(0, 1)\).

\subsection{Auxiliary Distribution Characteristics }

In Figure~\ref{fig:eta_effect}, we illustrate the auxiliary distributions for both the Gaussian kernel with a VP schedule and the Logistic kernel, alongside the target distribution, which in this case is \textit{Bernoulli}(0.7). For both kernels, as \(\sigma\) (Gaussian) and \(\eta\) (Logistic) increase, the coupling weakens, causing \(p(\bm{\theta}_a)\) to become increasingly independent of \(\bm{\theta}_a\). Notably, for the Gaussian kernel, when the signal strength diminishes (\(\alpha \rightarrow 0\)), \(\sigma \rightarrow 1\), leading the auxiliary distribution to converge to the standard normal distribution (\(p(\bm{\theta}_a) \rightarrow \mathcal{N}(0,1)\)).

While both kernels are capable of bridging the modes effectively, the Logistic kernel demonstrates a significant advantage due to its broader support. In contrast, the Gaussian kernel's support narrows sharply as its variance decreases, resulting in limited coverage of these discrete modes.

Additionally, the Logistic kernel provides a smooth and gradual transition as the scale parameter \(\eta\) changes, ensuring a consistent evolution of the auxiliary distribution. This consistency enables a balanced trade-off between exploration and exploitation. In comparison, the Gaussian kernel exhibits abrupt changes in the auxiliary distribution as the scale parameter \(\sigma\) shrinks (e.g., as \(\alpha\) transitions from 0.95 to 0.99). These abrupt changes render the Gaussian kernel highly sensitive to hyperparameters.

\begin{figure}[ht]
    \centering
    \includegraphics[width=0.45\textwidth]{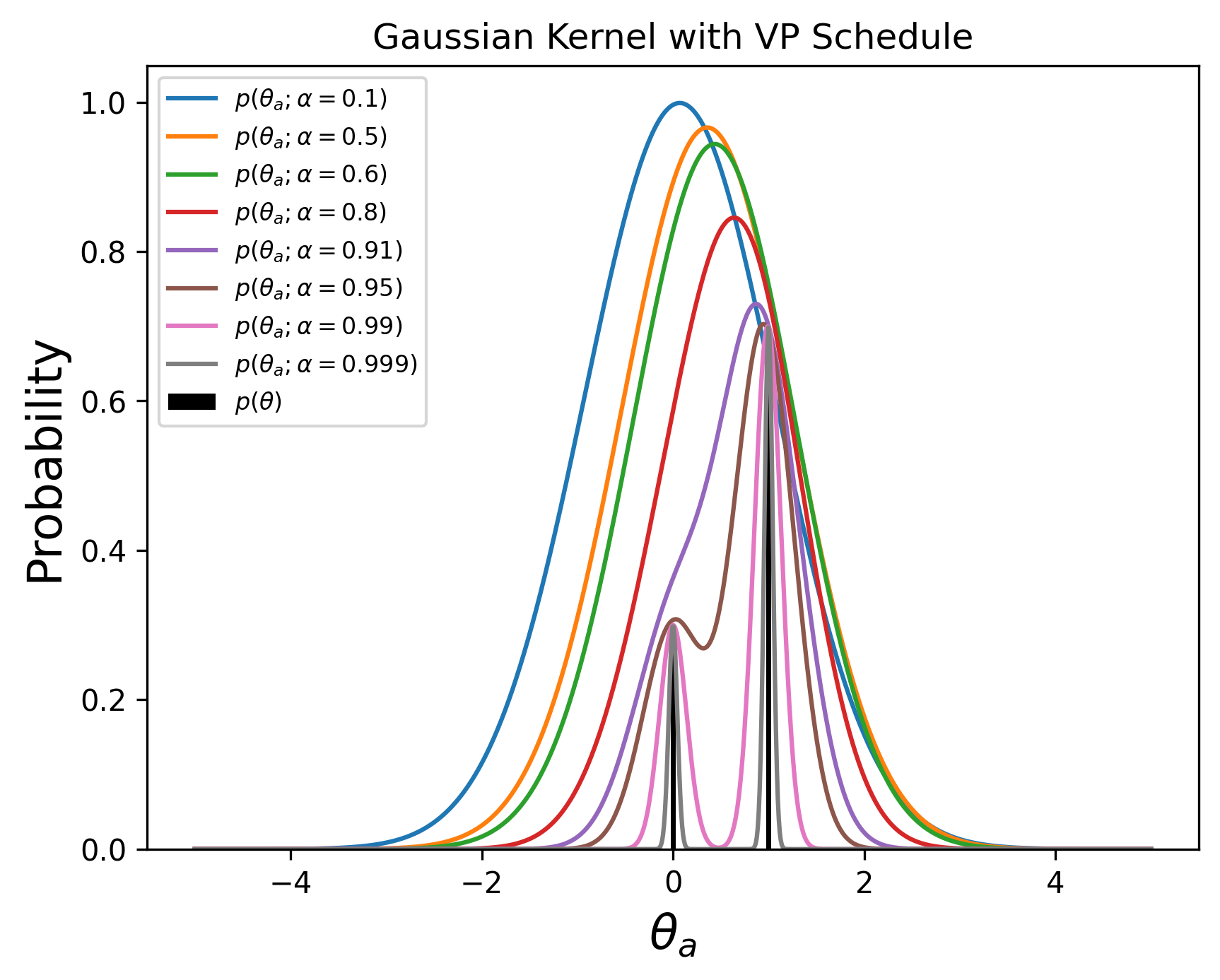}
    \includegraphics[width=0.45\textwidth]{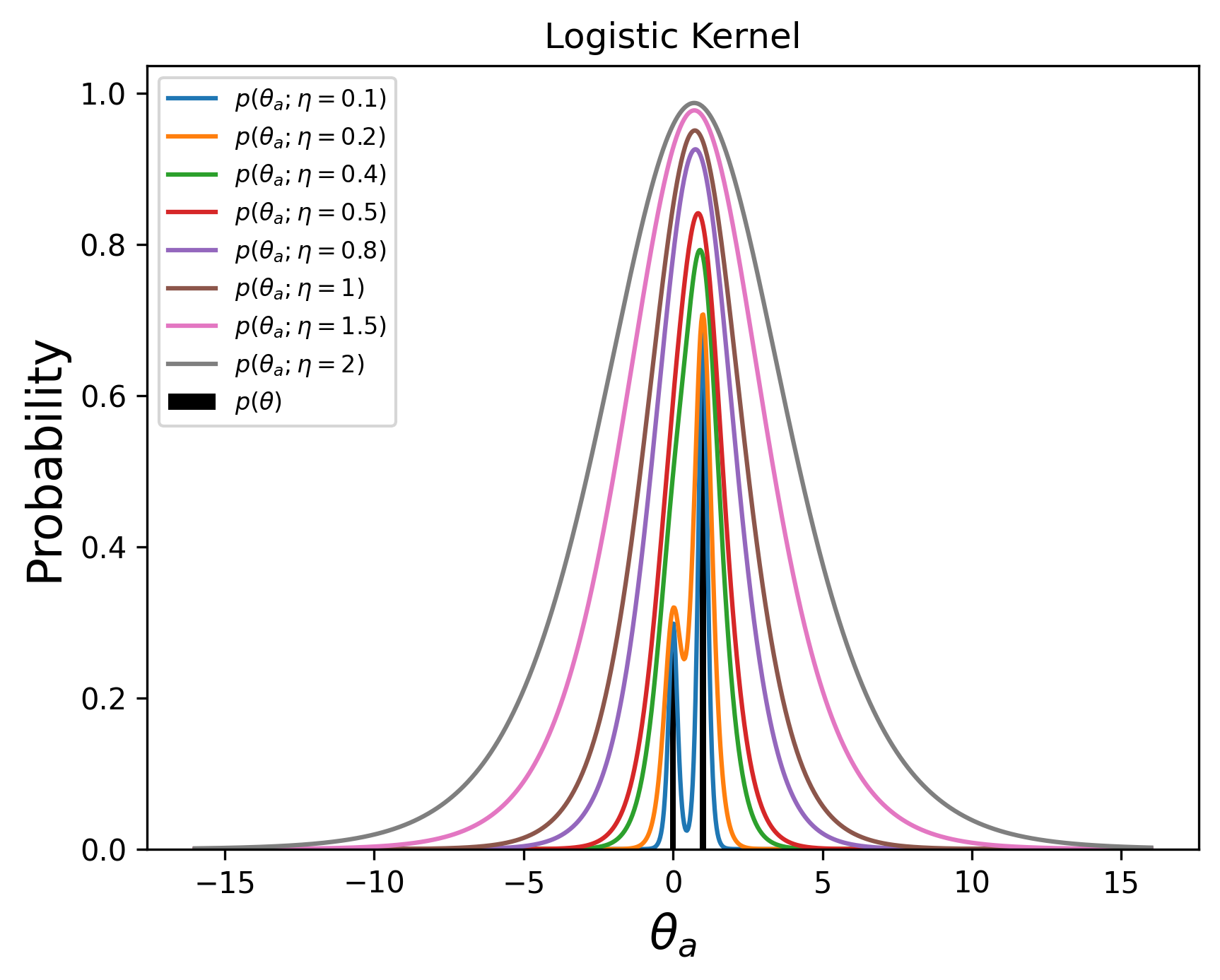}
    \caption{Mode Bridging in Bernoulli Distribution}
    \label{fig:eta_effect}
\end{figure} 

\subsection{Motivational Example: Mixture of Dirac Deltas}

We consider a symmetric {mixture of two Dirac delta distributions}. This setting is theoretically significant, as kernel-based methods have been shown to achieve geometric convergence rates specifically when the target measure is atomic(sum of Dirac deltas) \citep{pmlr-v161-khanna21a}.
Thus,
\[
p(x) = \frac{1}{2} \delta(x + \mu) + \frac{1}{2} \delta(x - \mu), \quad \mu > 0,
\]
where \(\mu\) controls the separation between the two modes. Convolving \(p(x)\) with a kernel \(k(x-x')\) produces the smoothed distribution:
\[
\tilde{p}(x) = \sum_{x' \in \{-\mu, \mu \} } p(x')k(x-x')
\]

In order to measure the mode bridging tendency of the kernels, we wish to compute the \textbf{intermediate mass} in the \(\epsilon\)-strip, defined as the probability mass within the region \(|x| < \epsilon\) under $\tilde{p}(x)$( See Figure \ref{fig:kernel} for intuition). 

\begin{figure}[ht]
    \centering
    \includegraphics[width=0.6\textwidth]{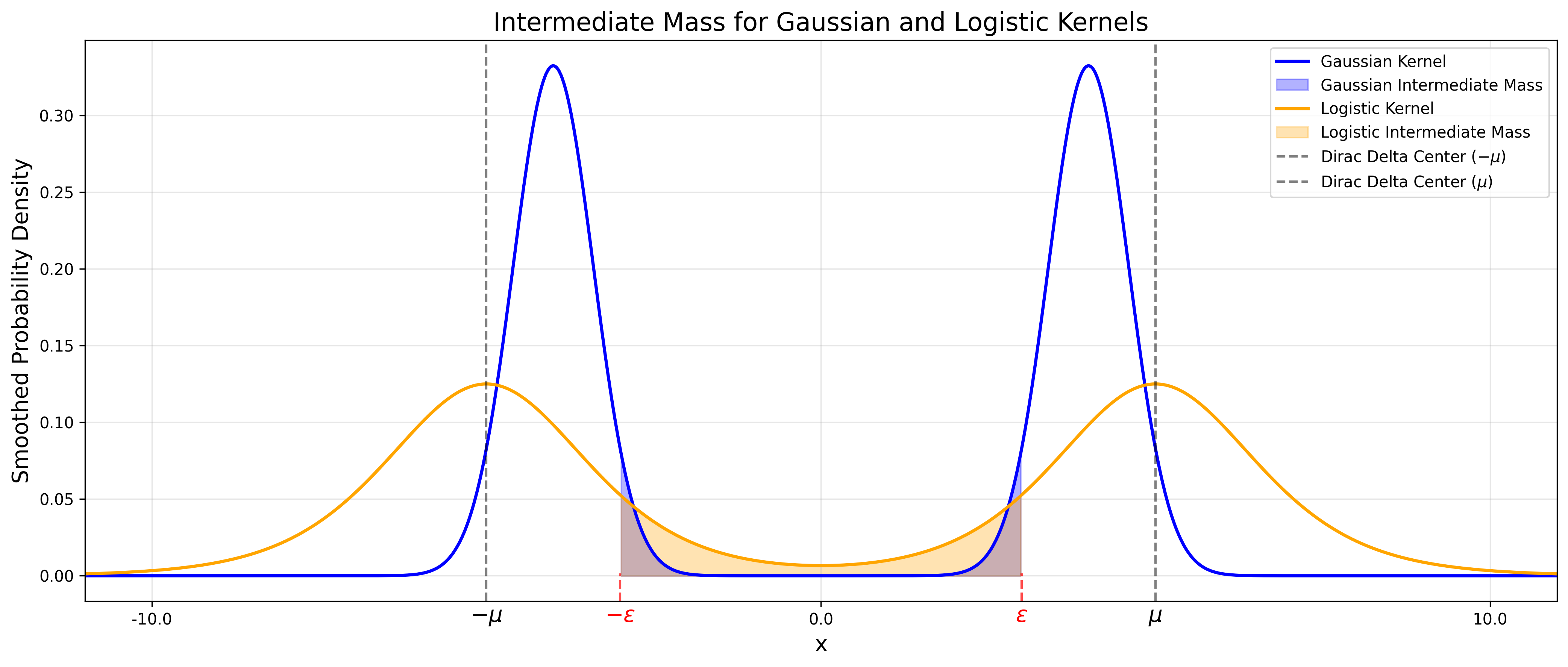}
    \caption{Intermediate Mass for Kernels}
    \label{fig:kernel}
\end{figure} 
Mathematically, 
\[
\tilde{I}(\epsilon) = \int_{-\epsilon}^\epsilon \tilde{p}(x) \, dx , \quad \epsilon > 0,
\]

\subsubsection*{Gaussian Kernel}
Under the {VP schedule} inspired by Diffusion Models \citep{sohl2015deep,song2019generative, ho2020denoising}, the Gaussian kernel is parameterized by \(\alpha\) and \(\sigma\), satisfying:
\[
\alpha^2 + \sigma^2 = 1, \quad \alpha > 0, \sigma >0,
\]
The Gaussian kernel is given by:
\[
k_{\text{G}}(x-x') = \frac{1}{\sqrt{2\pi\sigma^2} } e^{\frac{-(x-\alpha x')^2 }{ 2\sigma^2}}.
\]
The smoothed distribution becomes:
\[
\tilde{p}_{\text{G}}(x) = \frac{1}{2}\frac{1}{\sqrt{2\pi\sigma^2} } e^{\frac{-(x+\alpha \mu)^2 }{ 2\sigma^2}} + \frac{1}{2}\frac{1}{\sqrt{2\pi\sigma^2} } e^{\frac{-(x-\alpha \mu)^2 }{ 2\sigma^2}}
\]

\subsubsection*{Logistic Kernel}
The logistic kernel is parameterized by \(\eta > 0\) and is defined as:
\[
k_{\text{L}}(x-x') = \frac{1}{4\eta} \text{sech}^2\left(\frac{x-x'}{2\eta}\right),
\]
where \(\text{sech}(z) = \frac{2}{e^z + e^{-z}}\).

The smoothed distribution becomes:
\[
\tilde{p}_{\text{L}}(x) = \frac{1}{2} \frac{1}{4\eta} \text{sech}^2\left(\frac{x + \mu}{2\eta}\right) + \frac{1}{2} \frac{1}{4\eta} \text{sech}^2\left(\frac{x - \mu}{2\eta}\right).
\]

 Now near \(x = 0\),
\begin{align*}
\tilde{p}_{\text{G}}(0) &= \frac{1}{\sqrt{2\pi\sigma^2} } e^{\frac{-(\alpha \mu)^2 }{ 2\sigma^2}}\\
&=\frac{1}{\sqrt{2\pi\sigma^2} } e^{\frac{- \mu^2 }{2}\cdot\frac{1-\sigma^2}{\sigma^2}}
\end{align*}

\begin{align}
\tilde{p}_{\text{L}}(0) &=  \frac{1}{4\eta} \text{sech}^2\left(\frac{ \mu}{2\eta}\right)
\end{align}
For distant modes i.e. \(\mu >>> \epsilon\), the intermediate mass can be approximated as,

\[
\tilde{I}(\epsilon) \approx 2\epsilon.\tilde{p}_{}(0)
\]
Thus, 
\begin{align*}
\tilde{I}_{\text{G}}(\epsilon) &\approx 2\epsilon \cdot \frac{1}{\sqrt{2\pi\sigma^2} } e^{\frac{- \mu^2 }{2}\cdot\frac{1-\sigma^2}{\sigma^2}}\\
& =\sqrt{\frac{2}{\pi}}\frac{\epsilon}{\sigma}e^{-\frac{\mu^2}{(\frac{2\sigma^2}{1-\sigma^2})}}
\end{align*}

\begin{align*}
\tilde{I}_{\text{L}}(\epsilon) &\approx 2\epsilon \cdot \frac{1}{4\eta} \text{sech}^2\left(\frac{ \mu}{2\eta}\right)\\
& \approx 2\frac{\epsilon}{\eta} e^{-\frac{\mu}{\eta}}
\end{align*}

We notice, for $\mu \to \infty$, both kernels are highly sensitive to their respective parameters when the parameters are small, as the exponential decay dominates. The Logistic kernel is generally more robust to parameter variations compared to the Gaussian kernel, as its decay rate remains linear in $\mu$, while the Gaussian kernel exhibits a steep quadratic decay that magnifies sensitivity to small $\sigma^2$.

The Gaussian kernel, due to its rapid decay (\(e^{-\mu^2}\)), heavily concentrates probability mass near the modes, making it inefficient for escaping local modes. In contrast, the logistic kernel, with its slower tail decay (\(e^{-\mu}\)), retains significantly more intermediate mass, allowing it to bridge distant isolated modes for effectively. 

\subsection{Empirical Assessment}
See our discussion on Gaussian Kernel vs Logistic Kernel under Section \ref{sec:exp:as}, Figure \ref{fig:gaussvslog}.

\section{Conditional DLP \label{sec:app_DLP}}
By incorporating the coupling between the variables, we refine the DLP proposal by replacing \(\nabla U(\bm{\theta})\) with \(\nabla_{\bm{\theta}} U_{\eta}(\widetilde{\bm{\theta}})\). This adjustment results in the modified proposal:

\begin{equation}\label{eq:prop3}
    q_{\text{DMALA}}^{\text{joint}}(\bm{\widetilde{\theta}}^{(t)} \mid \bm{\widetilde{\theta}}^{(t-1)}) =\frac{1}{Z(\widetilde{\bm{\theta}}^{(t-1)})} \exp\left(\frac{1}{2} \nabla_{\bm{\theta}}U_{\eta}(\bm{\widetilde{\theta}}^{(t-1)})^{\top}(\bm{\theta}^{(t)} - \bm{\theta}^{(t-1)}) - \frac{1}{2\alpha}\|\bm{\theta}^{(t)} - \bm{\theta}^{(t-1)}\|^2\right)
\end{equation}

and corresponding normalization constant is $Z(\widetilde{\bm{\theta}}^{(t-1)})$. To further simplify, we exploit \emph{coordinate-wise factorization}, allowing us to express Eq. $
 q_{\text{DMALA}}^{\text{joint}}(\bm{\widetilde{\theta}}^{(t)} \mid \bm{\widetilde{\theta}}^{(t-1)}) = \prod_{i=1}^{d} q_{\text{DMALA}_i}^{\text{joint}}(\theta_i'|\widetilde{\bm{\theta}})$, where \( q_{\text{DMALA}_i}^{\text{joint}}(\widetilde{\bm{\theta}}_i'|\widetilde{\bm{\theta}}) \) is a categorical distribution:

\begin{align}
   \text{Categorical}\left(\text{Softmax}\left(\frac{1}{2} \nabla_{\theta} U_{\eta}(\widetilde{\bm{\theta}})_i (\theta_i' - \theta_i) - \frac{(\theta_i' - \theta_i)^2}{2\alpha}\right)\right).
   \label{eq:catprop2}
\end{align}

We use the keyword \emph{joint} specifically to denote that at the current joint position \(\widetilde{\bm{\theta}}\), the proposal distribution \( q_{\text{DMALA}}^{\text{joint}}(\cdot|\widetilde{\bm{\theta}}) \) generates the next joint position, keeping $\bm{\theta}_a^{}$ the same across both positions. 

 During the Metropolis-Hastings (MH) step, the conditional acceptance probability is calculated as:

\begin{equation}\label{eq:acce}
a_{\text{DMALA}}^{\text{joint}}(\bm{\widetilde{\theta}}^{(t)}| \bm{\widetilde{\theta}}^{(t-1)})=\min\left(1,  \frac{q_{\text{DMALA}}^{\text{joint}}(\bm{\widetilde{\theta}}^{(t-1)})|\bm{\widetilde{\theta}}^{(t)}_{})}{q_{\text{DMALA}}^{\text{joint}}(\bm{\widetilde{\theta}}^{(t)}| \bm{\widetilde{\theta}}^{(t-1)})}\cdot\frac{\pi(\bm{\widetilde{\theta}}^{(t)})}{\pi(\bm{\widetilde{\theta}}^{(t-1)})}\right).
\end{equation}
 where \( \bm{\widetilde{\theta}}^{(t)} = [{\bm{\theta}^{(t)}}^T, {\bm{\theta}_a^{(i-1)}}^T]^T \)and \( \bm{\widetilde{\theta}}^{(t-1)} = [{\bm{\theta}^{(t-1)}}^T, {\bm{\theta}_a^{(i-1)}}^T]^T \).

\section{Proof of Theorems\label{sec:app_proof}}

\subsection{Proof of Lemma~\ref{lemma:joint_posterior}}
\begin{proof}
Assume $\widetilde{\bm{\theta}}=[\bm{\theta}^T,\bm{\theta}_a^T]^T$ is sampled from the joint posterior distribution:
\begin{equation}
    p(\widetilde{\bm{\theta}})=p(\bm{\theta},\bm{\theta}_a)\propto\exp\left\{U(\bm{\theta})- 2\ln\left(\cosh\left(\frac{\bm{\theta}_a - \bm{\theta}}{2\eta}\right)\right)\right\}.
\end{equation}
Then the marginal distribution for $\bm{\theta}$ is:
\begin{equation}
\begin{aligned}
    p(\bm{\theta})&=\int_{}p(\bm{\theta},\bm{\theta}_a)d\bm{\theta}_a \\
    &=(4\eta)^{-d}Z^{-1}\int_{}\exp\left\{U(\bm{\theta})- 2\ln\left(\cosh\left(\frac{\bm{\theta}_a - \bm{\theta}}{2\eta}\right) \right) \right\} d\bm{\theta}_a \\
    &=Z^{-1}\exp(U(\bm{\theta}))(4\eta)^{-d}\int_{}\exp\left\{- 2\ln\left(\cosh\left(\frac{\bm{\theta}_a - \bm{\theta}}{2\eta}\right) \right) \right\}d\bm{\theta}_a \\
    &=Z^{-1}\exp(U(\bm{\theta})),
\end{aligned}
\end{equation}
where $Z=\sum_{\bm{\Theta}}\exp(U(\bm{\theta}))$ is the normalizing constant, and it is obtained by:
\begin{equation}
    \sum_{\bm{\Theta}}\int_{}\exp\left\{U(\bm{\theta})- 2\ln\left(\cosh\left(\frac{\bm{\theta}_a - \bm{\theta}}{2\eta}\right) \right) \right\}d\bm{\theta}_a=(4\eta)^{d}\sum_{\bm{\Theta}}\exp(U(\bm{\theta})):=(4\eta)^{d}Z.
\end{equation}
This verifies that the joint posterior distribution $p(\bm{\theta},\bm{\theta}_a)$ is mathematically well-defined\footnote{The exact form of the joint posterior is $p(\bm{\theta},\bm{\theta}_a)=(4\eta)^{-d}Z^{-1}\exp\left\{U(\bm{\theta})-2\ln\left(\cosh\left(\frac{\bm{\theta}_a - \bm{\theta}}{2\eta}\right) \right)\right\}$.}. Similarly, the marginal distribution for $\bm{\theta}_a$ is:
\begin{equation}
\begin{aligned}
p(\bm{\theta}_a)&=\sum_{\bm{\Theta}}p(\bm{\theta},\bm{\theta}_a) \\
    &\propto\sum_{\bm{\Theta}}\exp\left\{U(\bm{\theta})-2\ln\left(\cosh\left(\frac{\bm{\theta}_a - \bm{\theta}}{2\eta}\right) \right)\right\}
\end{aligned}
\end{equation}
\end{proof}

\subsection{Proof of Proposition~\ref{prop:gibbs-valid}}
\begin{proof}
To show irreducibility, we consider two states $\bm{\widetilde{\theta}}_1$ and $\bm{\widetilde{\theta}}_2$, such that, the Markov chain can move from the first to the second with positive probability. In particular we focus on one full gibbs sweep of Algorithm \ref{alg:1} consisting of: (a) updating $\bm{\theta}_a$ given the current $\bm{\theta}$, and (b) updating $\bm{\theta}$ given the new $\bm{\theta}_a$.
Because $p(\bm{\theta})>0$ for all $\bm{\theta}$ (by assumption), and the logistic conditional $p(\bm{\theta}_a|\bm{\theta})$ has a density that is positive everywhere on $\mathbb{R}^d$ , each sub-step has full support over its variable’s domain. In practical terms, no matter what values $(\bm{\theta}_1,\bm{\theta}_{a_1})$ the chain is in, the next $\bm{\theta}_a$-update can yield any $\bm{\theta}_{a}$ in $\mathbb{R}^d$ with some positive probability.Further, $\bm{\theta}_a$ is arbitrarily close to the desired $\bm{\theta}_{a_2}$. In the subsequent $\bm{\theta}$-update, it can then move $\bm{\theta}$ to (or near) $\bm{\theta}_2$ because $p(\bm{\theta} \mid \bm{\theta}_a) \propto p(\bm{\theta})p(\bm{\theta}_a|\bm{\theta})$, which remains positive for $\bm{\theta}=\bm{\theta}_2$ given $p(\bm{\theta}_2)>0$ and $p(\bm{\theta}_{a_2}|\bm{\theta}_2)>0$. Combining these two moves: starting at $(\bm{\theta}_1,\bm{\theta}_{a_1})$, the probability of first picking $\bm{\theta}_a$ near $\bm{\theta}_{a_2}$ and then picking $\bm{\theta}$ near $\bm{\theta}_2$ is positive. By making the neighborhood small, we see there is a non-zero probability of landing exactly in state $(\bm{\theta}_2,\bm{\theta}_{a_2})$ after one full sweep (or in a few sweeps, if we proceed gradually). Thus, $\bm{\widetilde{\theta}}_2$ is reachable with some positive probability.  Equivalently, the 2-step transition kernel  has a density that is everywhere positive on the state space i.e. $p(\widetilde{\bm{\theta}}) > 0$\citep{roberts1994simple, robert1999monte}. Therefore, every state can reach with every other state in some number of steps, making the chain one communicating class.

Since the joint target $p(\widetilde{\bm{\theta}})$ is a proper probability distribution (i.e. it integrates to 1 over the whole state space), it serves as an invariant (stationary) distribution for the Markov chain. An irreducible chain( shown above) that possesses a finite invariant measure  is positive recurrent\citep{meyn2009markov}. Positive recurrence implies that the expected return time to any state is finite, and in fact each state (or any set of states with nonzero stationary probability) will be visited infinitely often over an infinite chain trajectory\citep{nummelin1984general}. 

All in all, the Markov chain induced by Algorithm~\ref{alg:1} is irreducible and recurrent (in particular, positive recurrent).

\end{proof}

\subsection{Proof of Proposition~\ref{prop:db}}
\begin{proof}

To establish that HiSS satisfies detailed balance with respect to the  target distribution \(\pi(\bm{{\theta}})\), we aim to show that for every pair of states \(\bm{{\theta}}^{(i)}\) and \(\bm{{\theta}}^{(i-1)}\):
\[
\pi(\bm{{\theta}}^{(i-1)}) \kappa_{\text{HiSS}}(\bm{{\theta}}^{(i)} \mid \bm{{\theta}}^{(i-1)}) = \pi(\bm{{\theta}}^{(i)}) \kappa_{\text{HiSS}}(\bm{{\theta}}^{(i-1)} \mid \bm{{\theta}}^{(i)}),
\]
where \(\kappa_{\text{HiSS}}(\bm{{\theta}}^{(i)} \mid \bm{{\theta}}^{(i-1)})\) is the marginal  transition kernel for $G$ Gibbs sweep.

We rigorously define the marginal transition kernel for a single Gibbs sweep (denoted as $\kappa$) in Algorithm~\ref{alg:1} in \eqref{eq:marginal}. In particular, we discuss the process for the $i^{\text{th}}$ Gibbs sweep, where $i=1,2,3,\cdots, G$, from $\bm{{\theta}}^{(i-1)}$ to $\bm{{\theta}}^{(i)}$:

\begin{equation}\label{eq:marginal}
\kappa_{\text{}}(\bm{\theta}^{(i)} \mid \bm{\theta}^{(i-1)}) = \int p_{\text{MwG}}(\bm{\theta}_{\text{init}}^{(i)}, \bm{\theta}_a^{(i-1)} \mid \bm{\theta}^{(i-1)}) \cdot p_{\text{DMALA}}(\bm{\theta}^{(i)} \mid [{\bm{\theta}^{(i)}_{\text{init}}}^T, {\bm{\theta}_a^{(i-1)}}^T]^T ) \, d\bm{\theta}_a^{(i-1)}
\end{equation}
where $p_{\text{MwG}}^{\text{}}(\bm{{\theta}}^{(i)}_{\text{init}},  \bm{\theta}_a^{(i-1)} \mid \bm{{\theta}}^{(i-1)^{}})$ represents the local  transition kernel for the acceptance of the denoised proposal(Metropolis-within-Gibbs (MwG) style), and \(p_{\text{DMALA}}^{\text{}}(\bm{{\theta}}^{(i)} \mid \bm{{\theta}}^{(i)}_{\text{init}},  \bm{\theta}_a^{(i-1)})\) represents the $L$-stepped score-based refinement using DMALA. We also sometimes use this notation $\bm{\widetilde{\theta}}^{(i)}_{\text{init}}$ for intermediate state $(\bm{\theta}^{(i)}_{\text{init}}, \bm{\theta}_a^{(i-1)})$.

\subsubsection*{$p^{\text{}}_{\text{MwG}}(. \mid .)$}
The MwG step transitions  from \(\bm{{\theta}}^{(i-1)^{}}\) to \(\widetilde{\bm{{\theta}}}^{(i)}_{\text{init}}\):
\begin{align*}
p_{\text{MwG}}^{\text{}}(\widetilde{\bm{{\theta}}}^{(i)}_{\text{init}} \mid \bm{{\theta}}^{(i-1)^{}})
&= q_{\text{noise}}(\bm{\theta}_a^{(i-1)} \mid \bm{\theta}^{(i-1)}) \cdot q_{\text{denoise}}(\bm{\theta}'_{\text{init}} \mid \bm{\theta}_a^{(i-1)}) \cdot a_{\text{init}}(\bm{{\theta}}'_{\text{init}}|\bm{{\theta}}^{(i-1)}) \\
&\quad + \left(1-L(\bm{{\theta}}^{(i-1)^{}})\right)\delta(\bm{{\theta}}^{(i)}_{\text{init}})
\end{align*}

where $p_{\text{noise-denoise}}(\bm{{\theta}}'_{}|\bm{{\theta}}^{})=\int_{\bm{\theta}_a} q_{\text{noise}}(\bm{{\theta}}_a^{}|\bm{{\theta}}_{})q_{\text{denoise}}(\bm{{\theta}}'^{} |\bm{{\theta}}_a^{}) d\bm{\theta}_a$, $\delta_{}(.)$ is the Kronecker delta function and $L({\bm{\theta}}^{(i-1)})$ is the total acceptance probability away from the point ${\bm{\theta}}^{(i-1)}$ with 
\[L({\bm{\theta}}^{(i-1)})=\sum_{ \bm{\theta}'\in \bm{\Theta} }a_{\text{init}}(\bm{\theta'} \mid {\bm{\theta}}^{(i-1)})\cdot p_{\text{noise-denoise}}( \bm{\theta}' \vert \bm{\theta}^{(i-1)}) \quad \]

\subsubsection*{$p^{\text{}}_{\text{DMALA}}(. \mid .)$}
The overall L-stepped kernel mapping $\bm{\widetilde{\theta}}^{(i)}_{\text{init}}$ to $\text{}\bm{{\theta}}^{(i)}$ such that
${\bm{\theta}}^{(t=0)}={\bm{\theta}}_{\text{init}}^{(i)}, \bm{\theta}^{(t=L)} = \bm{\theta}^{(i)}$ is given by:
\begin{align*}
p_{\text{DMALA}}^{\text{}}(\bm{{\theta}}^{(i)} \mid \bm{\widetilde{\theta}}^{(i)}_{\text{init}}) &= \prod_{t=1}^L \left[q_{\text{DMALA}}^{\text{}}(\bm{{\theta}}^{(t)} \mid \bm{\widetilde{\theta}}^{(t-1)}) \cdot a_{\text{DMALA}}^{\text{}}(\bm{{\theta}}^{(t)}| \bm{\widetilde{\theta}}^{(t-1)}) + \left(1-L^{\text{}}(\bm{\widetilde{\theta}}^{(t-1)})\right)\delta_{}(\bm{{\theta}}^{(t)}_{}) \right]
\end{align*}

where by incorporating the coupling between the variables, we refine the DLP proposal by replacing \(\nabla U(\bm{\theta})\) with \(\nabla_{\bm{\theta}} U_{\eta}(\widetilde{\bm{\theta}})\). This adjustment results in the modified proposal from 
 \eqref{eq:prop3}. During the Metropolis-Hastings (MH) step, the conditional acceptance probability is calculated per \eqref{eq:acce}.
$L^{\text{}}(\bm{\widetilde{\theta}}^{(t-1)})$ is the total acceptance probability away from the point  $\bm{\widetilde{\theta}}^{(t-1)}$ with auxiliary variable fixed. Thus,  
$L^{\text{}}(\bm{\widetilde{\theta}}^{(t-1)})=\sum_{ \bm{\theta}'\in \bm{\Theta} }a_{\text{DMALA}}^{\text{}}([{\bm{\theta}'}^T, {\bm{\theta}_a^{(i-1)}}^T]^T \mid [{\bm{\theta}^{(t-1)}}^T, {\bm{\theta}_a^{(i-1)}}^T]^T )\cdot q_{\text{DMALA}}^{\text{}}( [{\bm{\theta}'}^T, {\bm{\theta}_a^{(i-1)}}^T]^T \mid [{\bm{\theta}^{(t-1)}}^T, {\bm{\theta}_a^{(i-1)}}^T]^T) $.

To show detailed balance for the marginal, we proceed in two steps. First, we show that both components of the HiSS transition : the Metropolis-within-Gibbs (MwG) proposal and the DMALA refinement, individually satisfy detailed balance with respect to $\pi$. Then, we show that their composition via a delayed acceptance framework preserves detailed balance in the marginal space over $\bm{\theta}$.

\paragraph{Step 1: Local detailed balance.}

\paragraph{1.1 MwG Step}
\begin{align*}
a_{\text{init}}(\bm{{\theta}}'_{\text{init}}|\bm{{\theta}}^{(i-1)})&= \min\left(1, \frac{\pi(\bm{{\theta}}'_{\text{init}})q_{\text{noise}}(\bm{{\theta}}_a^{(i-1)}|\bm{{\theta}}'_{\text{init}})q_{\text{denoise}}(\bm{{\theta}}^{(i-1)} \mid \bm{{\theta}}_a^{(i-1)})}{\pi(\bm{{\theta}}^{(i-1)})q_{\text{noise}}(\bm{{\theta}}_a^{(i-1)}|\bm{{\theta}}^{(i-1)})q_{\text{denoise}}(\bm{{\theta}}'_{\text{init}} \mid \bm{{\theta}}_a^{(i-1)})}\right)\\
&=\min\left(1, \frac{\pi(\bm{{\theta}}'_{\text{init}})p_{\text{noise-denoise}}(\bm{{\theta}}^{(i-1)}_{} \mid \bm{{\theta}}'_{\text{init}})}{\pi(\bm{{\theta}}^{(i-1)})p_{\text{noise-denoise}}(\bm{{\theta}}'_{\text{init}} \mid \bm{{\theta}}^{(i-1)})}\right)
\end{align*}

The above acceptance probability $a_{\text{init}}(\cdot|\cdot)$ explicitly matches the detailed balance condition for the marginal distribution $\pi(\bm{\theta})$. This is evident because the acceptance ratio is constructed precisely as the product of the target distribution and the marginalized noise-denoise kernels in both forward and reverse directions. Therefore, the Metropolis-within-Gibbs step preserves detailed balance with respect to the marginal distribution $\pi(\bm{\theta})$.

\paragraph{1.2 Denoising Step}

Each refinement step of DMALA aims to sample from the conditional distribution of $\bm{\theta}$ keeping $\bm{\theta}_a$ fixed. From Equation \eqref{eq:acce}, we notice, transition kernel respects detailed balance with respect to the joint distribution $\pi(\widetilde{\bm{\theta}})$. This guarantees that marginal transitions over $\bm{\theta}$ inherit detailed balance with respect to $\pi(\bm{\theta})$.\citep{robert1999monte, tierney1994markov}. Consequently, the \emph{induced} marginal proposal $q^{\text{marg}}_{\text{DMALA}}(\bm{\theta}' \mid \bm{\theta})$ satisfies:
    \[
    a_{\text{DMALA}}(\bm{{\theta}}^{(t)}|\bm{{\theta}}^{(t-1)}) = \min\left(1, \frac{\pi(\bm{{\theta}}^{(t)}) q^{\text{marg}}_{\text{DMALA}}(\bm{{\theta}}^{(t-1)} \mid \bm{{\theta}}^{(t)})}{\pi(\bm{{\theta}}^{(t-1)}) q^{\text{marg}}_{\text{DMALA}}(\bm{{\theta}}^{(t)} \mid \bm{{\theta}}^{(t-1)})}\right).
    \]

From above, we see, the DMALA kernel is Metropolis-adjusted and preserves detailed balance. Since each step in DMALA retains detailed balance, applying it $L$ times sequentially preserves detailed balance for the same stationary distribution $\pi$ for the entire Gibbs loop. 

\paragraph{Step 2: Composition via delayed acceptance.}

While each component satisfies detailed balance individually, their composition is not automatically reversible unless structured carefully. HiSS addresses this through a \emph{delayed-acceptance} Metropolis-Hastings construction \citep{banterle2015acceleratingmetropolishastingsalgorithmsdelayed}.

Although this is \emph{not} how Algorithm \ref{alg:1} is executed in practice, we now \emph{interpret} the overall transition between discrete states $\bm{\theta}$ and $\bm{\theta}'$ as a standard Metropolis-Hastings (MH) algorithm with an \emph{analytically induced effective proposal} $Q(\bm{\theta} \to \bm{\theta}')$, to encapsulate the two-stage process of sampling an auxiliary variable $\bm{\theta}_a$, proposing an intermediate state $\bm{\theta}_{\text{init}}$, and refining it into $\bm{\theta}'$ via $L$-step DMALA, and total acceptance probability $A(\bm{\theta} \to \bm{\theta}')$.  This interpretation is used purely for the purpose of proving detailed balance. 

The total acceptance probability then factors as:
\[
A(\bm{\theta} \to \bm{\theta}') = a_{\text{init}}(\bm{\theta}_{\text{init}} \mid \bm{\theta}) \cdot a_{\text{DMALA}}(\bm{\theta}' \mid \bm{\theta}_{\text{init}}),
\]
where $a_{\text{DMALA}}(\bm{\theta}' \mid \bm{\theta}_{\text{init}}) = \prod_{t=1}^{L} a_{\text{DMALA}}(\bm{\theta}^{(t)} \mid \bm{\theta}^{(t-1)})$.

In essence, we argue HiSS is a simple MH algorithm in disguise. The following decomposition solely serves as a condensed, conceptual MH abstraction to analyze the marginal chain:
\[
\kappa(\bm{\theta}' \mid \bm{\theta}) = Q(\bm{\theta} \to \bm{\theta}') A(\bm{\theta} \to \bm{\theta}') + \left[1 - \sum_u Q(\bm{\theta} \to u) A(\bm{\theta} \to u) \right] \delta(\bm{\theta}' - \bm{\theta}).
\]
By construction, the acceptance term is,
\[
A(\bm{\theta} \to \bm{\theta}') = \min\left(1, \frac{\pi(\bm{\theta}') Q(\bm{\theta}' \to \bm{\theta})}{\pi(\bm{\theta}) Q(\bm{\theta} \to \bm{\theta}')}\right)
\]
This factorization corresponds to the delayed-acceptance MH framework described in \cite[Section 2.2, Lemma 2]{banterle2015acceleratingmetropolishastingsalgorithmsdelayed}, where the acceptance is decomposed across multiple refinement steps.

Therefore, even though detailed balance does not necessarily hold jointly for the auxiliary-augmented chain, the marginal transition kernel $\kappa(\bm{\theta}' \mid \bm{\theta})$ satisfies detailed balance with respect to the target distribution $\pi(\bm{\theta})$:
\[
\pi(\bm{\theta}) \kappa(\bm{\theta}' \mid \bm{\theta}) = \pi(\bm{\theta}') \kappa(\bm{\theta} \mid \bm{\theta}').
\]

By induction, applying $G$ such Gibbs sweeps—each satisfying detailed balance—yields an overall kernel $\kappa_{\text{HiSS}}$ that also satisfies detailed balance with respect to $\pi$:
\[
\pi(\bm{\theta}^{(i-1)}) \kappa_{\text{HiSS}}(\bm{\theta}^{(i)} \mid \bm{\theta}^{(i-1)}) = \pi(\bm{\theta}^{(i)}) \kappa_{\text{HiSS}}(\bm{\theta}^{(i-1)} \mid \bm{\theta}^{(i)}).
\]

Thus, HiSS satisfies detailed balance with respect to the marginal distribution $\pi(\bm{\theta})$, ensuring that the Markov chain converges to the desired discrete target distribution.

\end{proof}

\begin{lemma}\label{lemma:2}
Under Assumptions \ref{assumption:5.1} and \ref{assumption:5.2}, and for step size $\alpha < \frac{2}{M}$ in Algorithm 1, for the $i^{th}$ Gibbs Sweep,  the $L$-step DMALA refinement kernel admits a uniform lower bound independent of the auxiliary variable $\bm{\theta}_a$. Specifically, for any starting state $\bm{\theta}^{(i)}_{\text{init}} \in \bm{\Theta}$ and any refined final  state  $\bm{\theta}^{(i)} \in \bm{\Theta}$,
{\footnotesize
    \begin{align*}
    p_{\text{DMALA}}(\bm{{\theta}}^{(i)} \mid \bm{\widetilde{\theta}}^{(i)}_{\text{init}}) \geq \mathcal{\nu}_i(A)\exp\Bigg\{ L\Bigg(( -\frac{M}{2} - \frac{1}{\alpha}+\frac{m}{4}) \text{diam}^2\bm{(\Theta)} - (\frac{1}{2} \|\nabla U(a)\|+\frac{3\sqrt{d}+1}{\eta})\text{diam}\bm{(\Theta)}  \Bigg) \Bigg\}
    \end{align*}
    }
where $\bm{\widetilde{\theta}}^{(i)}_{\mathrm{init}} = [\bm{\theta}_{\mathrm{init}}^{(i)}, \bm{\theta}_a^{(i-1)}]^T$ is the starting joint state and $A \subseteq \bm{\Theta}^L$ is a measurable set of length-$L$ trajectories, with $\nu_i(A)$ denoting a probability measure over those discrete trajectories.
\end{lemma}
\subsection{Proof of Lemma ~\ref{lemma:2}}
\begin{proof}
We follow a similar minorization proof style as of Lemma 5.3 from \citet{pynadath2024gradientbased}. The DMALA step transitions from \(\bm{\widetilde{\theta}}'_{\text{init}}\) to \(\bm{{\theta}}^{(i)}\) conditionally over \(L\) iterative steps, for some arbitrary Gibbs-Sweep, each incorporating DLP with \(\alpha\) as the step size for the $t^{th}$ refinement.

{\scriptsize
\begin{align*}
q_{\text{DMALA}}(\bm{{\theta}}^{(t)} \mid \bm{\widetilde{\theta}}^{(t-1)}) &\propto \exp\Bigg( 
\frac{1}{2} \nabla_{\bm{\theta}}U_{\eta}(\bm{\widetilde{\theta}}^{(t-1)})^{\top}(\bm{\theta}^{(t)} - \bm{\theta}^{(t-1)}) 
- \frac{1}{2\alpha}\|\bm{\theta}^{(t)} - \bm{\theta}^{(t-1)}\|^2 \Bigg) \\
&= \exp\Bigg( 
\frac{1}{2} \nabla U(\bm{{\theta}}^{(t-1)})^{\top}(\bm{\theta}^{(t)} - \bm{\theta}^{(t-1)}) 
+ \frac{1}{2\eta} \tanh\left(\frac{\bm{\theta}_a^{(t-1)}-\bm{\theta}^{(t-1)}}{2\eta}\right)^{\top}(\bm{\theta}^{(t)} - \bm{\theta}^{(t-1)})
\Bigg) \\
&= \exp \Bigg( 
\frac{1}{2}( -U(\bm{\theta}^{(t-1)}) + U(\bm{\theta}^{(t)}) ) - (\bm{\theta}^{(t-1)} - \bm{\theta}^{(t)})^{\top} 
\left(\frac{1}{2\alpha}I + \frac{1}{4}\int_{0}^{1} \nabla^2 U((1-s)\bm{\theta}^{(t-1)} + s\bm{\theta}^{(t)}) ds \right)\cdot\\
&(\bm{\theta}^{(t-1)} - \bm{\theta}^{(t)}) + \frac{1}{2\eta} \tanh\left(\frac{\bm{\theta}_a^{(t-1)}-\bm{\theta}^{(t-1)}}{2\eta}\right)^{\top}(\bm{\theta}^{(t)} - \bm{\theta}^{(t-1)})
\Bigg)
\end{align*}
}
The third line is true because we replace the linear gradient term $\nabla_{\bm\theta} U(\bm\theta)^\top (\bm\theta' - \bm\theta)$ with Taylor expansion with integral remainder.

Consequently, the modified normalizing constant becomes

 \scalebox{0.90}{
$\begin{aligned}
Z(\widetilde{\bm{\theta}}^{(t-1)}) &= \sum_{\bm{\theta'} \in \bm{\Theta}} 
\exp \Bigg(  \frac{1}{2}( -U(\bm{\theta}^{(t-1)}) + U(\bm{\theta'}) ) - (\bm{\theta}^{(t-1)}-\bm{\theta'}^{})^{\top} \left(\frac{1}{2\alpha}I + \frac{1}{4}\int_{0}^{1} \nabla^2 U((1-s)\bm{\theta}^{(t-1)} + s\bm{\theta'}^{}) ds \right)(\bm{\theta}^{(t-1)}-\bm{\theta'}^{})\\
&+\frac{1}{2\eta} \tanh\left(\frac{\bm{\theta}_a^{(t-1)}-\bm{\theta}^{(t-1)}}{2\eta}\right)^{\top}(\bm{\theta'}^{} - \bm{\theta}^{(t-1)})
\Bigg).
\end{aligned}$}

Recall, from Assumption \ref{assumption:5.1},U is $M$-gradient Lipschitz, we have 
\begin{align*}
   \frac{1}{\alpha} I +\frac{1}{2} \int_{0}^{1} \nabla^{2} U ( (1-s) \bm\theta + s \bm\theta' ) \, ds )   \ge  \left( \frac{1}{ \alpha}  -  \frac{M}{2} \right) I \\
\end{align*}
Since $\alpha < \frac{2}{M}$, the matrix $\left(\frac{1}{2\alpha} - \frac{ M}{2}\right) I $is positive definite.  \\

This implies, 
\[
 \scalebox{0.95}{
$\begin{aligned}
Z(\widetilde{\bm{\theta}}^{(t-1)}) &\leq \sum_{\bm{\theta'} \in \bm{\Theta}} 
\exp \Bigg(  \frac{1}{2}( -U(\bm{\theta}^{(t-1)}) + U(\bm{\theta'}) )+\frac{1}{2\eta} \tanh\left(\frac{\bm{\theta}_a^{(t-1)}-\bm{\theta}^{(t-1)}}{2\eta}\right)^{\top}(\bm{\theta'}^{} - \bm{\theta}^{(t-1)})\Bigg)\\
&\leq \exp\Bigg(\frac{-U(\bm{\theta}^{(t-1)})}{2}\Bigg)\sum_{\bm{\theta'} \in \bm{\Theta}} 
\exp \Bigg(  \frac{1}{2}( U(\bm{\theta'}) )+\frac{1}{2\eta} ||\tanh\left(\frac{\bm{\theta}_a^{(t-1)}-\bm{\theta}^{(t-1)}}{2\eta}\right)||\cdot||(\bm{\theta'}^{} - \bm{\theta}^{(t-1)})||\Bigg)\\
& \leq \exp\Bigg(\frac{-U(\bm{\theta}^{(t-1)})}{2} +\frac{\sqrt{d}}{{2\eta}}\text{diam}(\bm{\Theta}) \Bigg)\sum_{\bm{\theta'} \in \bm{\Theta}} 
\exp \Bigg(  \frac{1}{2}( U(\bm{\theta'}) ) \Bigg)
\end{aligned}$}
\]

Since Assumption~\ref{assumption:5.2} holds true in this setting, we have an $m>0$ such that for any $\bm{\theta} \in conv(\bm{\Theta})$
    \[-\nabla^2_{} U_{}(\bm{\theta}) \ge m\, I.\]
From this, one notes that,
\[
 \scalebox{0.95}{
$\begin{aligned}
Z(\widetilde{\bm{\theta}}^{(t-1)}) &\geq \sum_{\bm{\theta'} \in \bm{\Theta}} 
\exp \Bigg(  \frac{1}{2}( -U(\bm{\theta}^{(t-1)}) + U(\bm{\theta'}) )+\frac{1}{2\eta} \tanh\left(\frac{\bm{\theta}_a^{(t-1)}-\bm{\theta}^{(t-1)}}{2\eta}\right)^{\top}(\bm{\theta'}^{} - \bm{\theta}^{(t-1)})\Bigg)\cdot\\
&\exp\left\{-\frac{1}{2}\left(\frac{1}{ \alpha}- \frac{m}{2}\right)\, \text{diam($\bm{\Theta}$)}^2\right\}\\
&\geq \exp\left\{-\frac{1}{2}\left(\frac{1}{ \alpha}- \frac{m}{2}\right)\, \text{diam}^2(\bm{\Theta}) + \frac{-U(\bm{\theta}^{(t-1)})}{2} \right\} \cdot\\
&\sum_{\bm{\theta'} \in \bm{\Theta}} 
\exp \Bigg(  \frac{1}{2}( U(\bm{\theta'}) )-\frac{1}{2\eta} ||\tanh\left(\frac{\bm{\theta}_a^{(t-1)}-\bm{\theta}^{(t-1)}}{2\eta}\right)||\cdot||(\bm{\theta'}^{} - \bm{\theta}^{(t-1)})||\Bigg)\\
& \geq \exp\Bigg(\frac{-U(\bm{\theta}^{(t-1)})}{2} -\frac{\sqrt{d}}{2\eta} \text{diam}(\bm{\Theta}) -\frac{1}{2}\left(\frac{1}{ \alpha}- \frac{m}{2}\right)\, \text{diam}^2(\bm{\Theta}) \Bigg)\sum_{\bm{\theta'} \in \bm{\Theta}} \exp \Bigg(  \frac{U(\bm{\theta'})}{2}\Bigg)
\end{aligned}$}
\]

In other words,
{\scriptsize 
    \begin{align*}
        \exp\left(  -\frac{\sqrt{d}\text{diam}(\bm{\Theta})}{2\eta} -\frac{1}{2}\left(\frac{1}{ \alpha}- \frac{m}{2}\right)\, \text{diam}^2(\bm{\Theta}) \right) \le \frac{Z(\widetilde{\bm{\theta}}^{(t-1)})}{\exp(\frac{-U(\bm{\theta}^{(t-1)})}{2})\sum_{\bm{\theta'} \in \bm{\Theta}} \exp \Bigg(  \frac{U(\bm{\theta'})}{2} \Bigg)} \le \exp\Bigg(\frac{\sqrt{d}\text{diam}(\bm{\Theta})}{2\eta} \Bigg)
    \end{align*}
}
Consequently, 
    \begin{align*}
        \frac{\frac{Z_{}(\widetilde{\bm{\bm{\theta}}}^{(t-1)})}{\sum_{x \in \bm{\Theta}} \exp\left(  \frac{U(x)}{2}\right) \exp\left(-  \frac{U(\bm{\theta}^{(t-1})}{2}\right)}}{\frac{Z_{}(\widetilde{\bm{\theta}}^{(t)})}{\sum_{x \in \bm{\Theta}} \exp\left(  \frac{U(x)}{2}\right) \exp\left(-  \frac{U(\bm{\theta})^{(t)}}{2}\right)}} &\geq \frac{\exp\left( -\frac{\sqrt{d}\text{diam}(\bm{\Theta})}{2\eta} -\frac{(2-m\alpha)\text{diam}^2\bm{(\Theta)}}{4\alpha} \right)}{\exp\left(  \frac{\sqrt{d}\text{diam}(\bm{\Theta})}{2\eta} \right)}
    \end{align*}
This implies, for any two states $\widetilde{\bm{\theta}},\widetilde{\bm{\theta}'} $,
\begin{align*}
        \frac{Z_{}(\widetilde{\bm{\bm{\theta}}})}{Z_{}(\widetilde{\bm{\theta}'})} &\geq {\exp \left( \frac{1}{2} (-U(\bm{\theta})+U(\bm{\theta}'))\right)}{\exp\left( -\frac{\sqrt{d}\text{diam}(\bm{\Theta})}{\eta} -\frac{(2-m\alpha)\text{diam}^2\bm{(\Theta)}}{4\alpha} \right)}
\end{align*}
Thus,
{\scriptsize
\begin{align*}
q_{\text{DMALA}}(\bm{\widetilde{\theta}}^{(t)} \mid \bm{\widetilde{\theta}}^{(t-1)}) &\geq \frac{1}{Z(\widetilde{\bm{\theta}}^{(t-1)})} \exp\left(\frac{1}{2}\left\langle\nabla U(\bm{\theta}^{(t-1)}),\bm{\theta}^{(t)}-\bm{\theta}^{(t-1)}\right\rangle + \frac{1}{2\eta} \tanh\left(\frac{\bm{\theta}_a^{(t-1)}-\bm{\theta}^{(t-1)}}{2\eta}\right)^{\top}(\bm{\theta}^{(t)} - \bm{\theta}^{(t-1)})   \right)\\
&\geq \frac{1}{Z(\widetilde{\bm{\theta}}^{(t-1)})} \exp\left(\frac{1}{2}\left\langle\nabla U(\bm{\theta}^{(t-1)}),\bm{\theta}^{(t)}-\bm{\theta}^{(t-1)}\right\rangle - \frac{1}{2\eta} ||\tanh\left(\frac{\bm{\theta}_a^{(t-1)}-\bm{\theta}^{(t-1)}}{2\eta}\right)||\cdot ||\bm{\theta}^{(t)}  - \bm{\theta}^{(t-1)}||   \right)\\
&\geq \frac{1}{Z(\widetilde{\bm{\theta}}^{(t-1)})} \exp\left(\frac{1}{2}\left\langle\nabla U(\bm{\theta}^{(t-1)}),\bm{\theta}^{(t)}-\bm{\theta}^{(t-1)}\right\rangle - \frac{\sqrt{d}\text{diam}\bm{(\Theta)}}{2\eta}\right)\\
\end{align*}
}

We also note that
{\scriptsize
    \begin{align*}
        - \frac{1}{2} \left\langle\nabla U(\bm{\theta}),\bm{\theta}'-\bm{\theta}\right\rangle+\frac{1}{2\alpha}\|\bm{\theta}-\bm{\theta}'\|^2
        & = \frac{1}{2} \left\langle - \nabla U(\bm{\theta}) + \nabla U(a),\bm{\theta}'-\bm{\theta}\right\rangle+ \frac{1}{2} \left\langle - \nabla U(a),\bm{\theta}'-\bm{\theta}\right\rangle + \frac{1}{2\alpha}\|\bm{\theta}-\bm{\theta}'\|^2\\
        & \le \frac{1}{2} \left\langle - \nabla U(\bm{\theta}) + \nabla U(a),\bm{\theta}'-\bm{\theta}\right\rangle+ \frac{1}{2} \left\langle - \nabla U(a),\bm{\theta}'-\bm{\theta}\right\rangle + \frac{1}{2\alpha}diam(\bm{\Theta})^2 \\
        & \le \frac{1}{2} \left\| - \nabla U(\bm{\theta}) + \nabla U(a)\| \| \bm{\theta}'-\bm{\theta}\right\|+ \frac{1}{2} \left\|\nabla U(a) \| \| \bm{\theta}'-\bm{\theta}\right\| + \frac{1}{2\alpha}diam(\bm{\Theta})^2 \\
        & \le \frac{1}{2} \| - \nabla U(\bm{\theta}) + \nabla U(a)\|   diam(\bm{\Theta}) + \frac{1}{2} \|\nabla U(a) \| diam(\bm{\Theta}) + \frac{1}{2\alpha}diam(\bm{\Theta})^2 \\
        & \le \left(\frac{1}{2} M+\frac{1}{2\alpha}\right)\, diam(\bm{\Theta})^2+ \frac{1}{2} \|\nabla U(a)\|\, diam(\bm{\Theta}).
    \end{align*}
    }
Therefore,
{\small
\begin{align*}
q_{\text{DMALA}}(\bm{{\theta}}^{(t)} \mid \bm{\widetilde{\theta}}^{(t-1)}) &\geq \frac{1}{Z(\widetilde{\bm{\theta}}^{(t-1)})} \exp\left( (-\frac{M}{2}-\frac{1}{2\alpha}) \text{diam}^2\bm{(\Theta)}- \frac{1}{2} \|\nabla U(a)\|\text{diam}\bm{(\Theta)} - \frac{\sqrt{d}\text{diam}\bm{(\Theta)}}{2\eta} \right)\\
&\geq \frac{\exp\left( (-\frac{M}{2}-\frac{1}{2\alpha}) \text{diam}^2\bm{(\Theta)}- \frac{1}{2} \|\nabla U(a)\|\text{diam}\bm{(\Theta)} - \frac{\sqrt{d}\text{diam}\bm{(\Theta)}}{2\eta} \right)}{\exp\Bigg(\frac{-U(\bm{\theta}^{(t-1)})}{2} +\frac{\sqrt{d}\text{diam}\bm{(\Theta)}}{2\eta} \Bigg)\sum_{\bm{\theta'} \in \bm{\Theta}} 
\exp \Bigg(  \frac{1}{2}( U(\bm{\theta'}) ) \Bigg)}
\end{align*}
}

The acceptance ratio is given as:
{\footnotesize
\begin{align*}
    \rho_{\text{DMALA}}({\bm{\theta}}^{(t)} \mid \widetilde{\bm{\theta}}^{(t-1)}) 
    &= \frac{\pi(\widetilde{\bm{\theta}}^{(t)}) q_{\text{DMALA}}({\bm{\theta}}^{(t-1)} \mid \widetilde{\bm{\theta}}^{(t)})}
    {\pi(\widetilde{\bm{\theta}}^{(t-1)}) q_{\text{DMALA}}({\bm{\theta}}^{(t)} \mid \widetilde{\bm{\theta}}^{(t-1)})} \\
    &= 
    \exp\Bigg\{
        U(\bm{\theta}^{(t)}) - U(\bm{\theta}^{(t-1)}) 
        - 2 \ln \bigg(\cosh\left(\frac{\bm{\theta}_a^{(t-1)} - \bm{\theta}^{(t)}}{2\eta}\right)\bigg) \\
        &\quad + 2 \ln \bigg(\cosh\left(\frac{\bm{\theta}_a^{(t-1)} - \bm{\theta}^{(t-1)}}{2\eta}\right)\bigg) 
    \Bigg\} \cdot \frac{\tilde{Z}}{\tilde{Z}} \cdot \\
    &\exp\Bigg\{
        U(\bm{\theta}^{(t-1)}) - U(\bm{\theta}^{(t)}) 
        + \frac{1}{2\eta} \tanh\left(\frac{\bm{\theta}_a^{(t-1)} - \bm{\theta}^{(t)}}{2\eta}\right)^\top (\bm{\theta}^{(t-1)} - \bm{\theta}^{(t)}) \\
        &\quad - \frac{1}{2\eta} \tanh\left(\frac{\bm{\theta}_a^{(t-1)} - \bm{\theta}^{(t-1)}}{2\eta}\right)^\top (\bm{\theta}^{(t)} - \bm{\theta}^{(t-1)})
    \Bigg\} \cdot 
    \frac{Z(\widetilde{\bm{\theta}}^{(t-1)})}{Z(\widetilde{\bm{\theta}}^{(t)})} \\
    &= 
    \exp\Bigg\{
        - 2 \ln\left(
            \frac{\cosh\left(\frac{\bm{\theta}_a^{(t-1)} - \bm{\theta}^{(t)}}{2\eta}\right)}
                 {\cosh\left(\frac{\bm{\theta}_a^{(t-1)} - \bm{\theta}^{(t-1)}}{2\eta}\right)}
        \right) - \frac{1}{2\eta}(\bm{\theta}^{(t)} - \bm{\theta}^{(t-1)})^\top\\
        &\quad( \tanh\left(\frac{\bm{\theta}_a^{(t-1)} - \bm{\theta}^{(t)}}{2\eta}\right)
        + \tanh\left(\frac{\bm{\theta}_a^{(t-1)} - \bm{\theta}^{(t-1)}}{2\eta}\right) )\Bigg\}\cdot \frac{Z(\widetilde{\bm{\theta}}^{(t-1)})}{Z(\widetilde{\bm{\theta}}^{(t)})}\\
    &\geq
        \exp\Bigg\{
        - \frac{1}{\eta}||\bm{\theta}^{(t-1)}-\bm{\theta}^{(t)}|| - \frac{1}{2\eta}(\bm{\theta}^{(t)} - \bm{\theta}^{(t-1)})^\top\\
        &\quad( \tanh\left(\frac{\bm{\theta}_a^{(t-1)} - \bm{\theta}^{(t)}}{2\eta}\right)
        + \tanh\left(\frac{\bm{\theta}_a^{(t-1)} - \bm{\theta}^{(t-1)}}{2\eta}\right) )\Bigg\}\cdot \frac{Z(\widetilde{\bm{\theta}}^{(t-1)})}{Z(\widetilde{\bm{\theta}}^{(t)})}\\
    &\geq
        \exp\Bigg\{
        - \frac{1}{\eta}||\bm{\theta}^{(t-1)}-\bm{\theta}^{(t)}|| - \frac{\sqrt{d}}{\eta}||\bm{\theta}^{(t-1)}-\bm{\theta}^{(t)}||\Bigg\}\cdot \frac{Z(\widetilde{\bm{\theta}}^{(t-1)})}{Z(\widetilde{\bm{\theta}}^{(t)})}\\
    &\geq
        \exp\Bigg\{
        - \frac{(\sqrt{d}+1)}{\eta}\text{diam}(\bm{\Theta})\Bigg\}\cdot \frac{Z(\widetilde{\bm{\theta}}^{(t-1)})}{Z(\widetilde{\bm{\theta}}^{(t)})}
\end{align*}
}
where $\tilde{Z}$ is the normalizing constant for $\pi(\widetilde{\bm{\theta}})$.

with Acceptance Probability
\begin{align*}
 \mathcal{A}_{\text{DMALA}}({\bm{\theta}}^{(t)} \mid \widetilde{\bm{\theta}}^{(t-1)})&= \left(\mathcal{\rho}_{\text{DMALA}}({\bm{\theta}}^{(t)} \mid \widetilde{\bm{\theta}}^{(t-1)})\wedge 1\right)\\
 \end{align*}

The overall L-stepped kernel is then given by:
${\bm{\theta}}^{(t=0)}={\bm{\theta}}_{\text{init}}^{(i)}, \bm{\theta}^{(t=L)} = \bm{\theta}^{(i)}$
{\scriptsize
\begin{align*}
p_{\text{DMALA}}(\bm{{\theta}}^{(i)} \mid \bm{\widetilde{\theta}}^{(i)}_{\text{init}}) &= \prod_{t=1}^L \left[q_{\text{DMALA}}(\bm{{\theta}}^{(t)} \mid \bm{\widetilde{\theta}}^{(t-1)}) \cdot \mathcal{A}_{\text{DMALA}}(\bm{{\theta}}^{(t)}| \bm{\widetilde{\theta}}^{(t-1)}) + \left(1-L(\bm{\widetilde{\theta}}^{(t-1)})\right)\delta_{{\bm{\bm{\theta}}}}(\bm{{\theta}}^{(t)}_{}) \right]\\
&\geq \prod_{t=1}^L \left[q_{\text{DMALA}}(\bm{{\theta}}^{(t)} \mid \bm{\widetilde{\theta}}^{(t-1)}) \cdot \mathcal{A}_{\text{DMALA}}(\bm{{\theta}}^{(t)}| \bm{\widetilde{\theta}}^{(t-1)}) \right]\\
&\geq  \prod_{t=1}^L\left[ q_{\text{DMALA}}(\bm{{\theta}}^{(t)} \mid \bm{\widetilde{\theta}}^{(t-1)}) \cdot \exp\Bigg\{
        - \frac{(\sqrt{d}+1)}{\eta}\text{diam}(\bm{\Theta})\Bigg\}\cdot \frac{Z(\widetilde{\bm{\theta}}^{(t-1)})}{Z(\widetilde{\bm{\theta}}^{(t)})} \right]^{}\\
&\geq \prod_{t=1}^L \Bigg[ 
q_{\text{DMALA}}(\bm{{\theta}}^{(t)} \mid \bm{\widetilde{\theta}}^{(t-1)}) 
\cdot \exp\left( - \frac{(\sqrt{d}+1)}{\eta} \, \text{diam}(\bm{\Theta}) \right) \cdot 
\exp\left( \frac{1}{2} \left[ -U(\bm{\theta}^{(t-1)}) + U(\bm{\theta}^{(t)}) \right] \right) \cdot \\
&\quad \exp\left( -\frac{\sqrt{d} \, \text{diam}(\bm{\Theta})}{\eta} 
- \frac{(2 - m\alpha) \, \text{diam}^2(\bm{\Theta})}{4\alpha} \right) 
\Bigg]\\
&=  \prod_{t=1}^L \Bigg[ 
q_{\text{DMALA}}(\bm{{\theta}}^{(t)} \mid \bm{\widetilde{\theta}}^{(t-1)}) \cdot 
\exp\Bigg( 
- \frac{(2\sqrt{d}+1)}{\eta} \, \text{diam}(\bm{\Theta}) 
+ \frac{1}{2} \left( -U(\bm{\theta}^{(t-1)}) + U(\bm{\theta}^{(t)}) \right) \\
&\quad - \frac{(2 - m\alpha) \, \text{diam}^2(\bm{\Theta})}{4\alpha} 
\Bigg) \Bigg]\\
& \geq \prod_{t=1}^L \Bigg[ 
\frac{
\exp\left( 
-\left( \frac{M}{2} + \frac{1}{2\alpha} \right) \text{diam}^2(\bm{\Theta}) 
- \frac{1}{2} \|\nabla U(a)\| \cdot \text{diam}(\bm{\Theta}) 
- \frac{\sqrt{d} \cdot \text{diam}(\bm{\Theta})}{2\eta} 
\right)
}{
\exp\left( -\frac{1}{2} U(\bm{\theta}^{(t-1)}) + \frac{\sqrt{d} \cdot \text{diam}(\bm{\Theta})}{2\eta} \right)
\sum_{\bm{\theta'} \in \bm{\Theta}} 
\exp\left( \frac{1}{2} U(\bm{\theta'}) \right)
} \\
&\quad \cdot 
\exp\Bigg( 
- \frac{(2\sqrt{d} + 1)}{\eta} \cdot \text{diam}(\bm{\Theta}) 
+ \frac{1}{2} \left( U(\bm{\theta}^{(t)}) - U(\bm{\theta}^{(t-1)}) \right) 
- \frac{(2 - m\alpha)}{4\alpha} \cdot \text{diam}^2(\bm{\Theta}) 
\Bigg)
\Bigg]\\
&= \prod_{t=1}^L \Bigg[ 
\frac{
\exp\left( 
-\left( \frac{M}{2} + \frac{1}{2\alpha} \right) \text{diam}^2(\bm{\Theta}) 
- \frac{1}{2} \|\nabla U(a)\| \cdot \text{diam}(\bm{\Theta}) 
- \frac{\sqrt{d} \cdot \text{diam}(\bm{\Theta})}{\eta} 
\right)
}{
\sum_{\bm{\theta'} \in \bm{\Theta}} 
\exp\left( \frac{1}{2} U(\bm{\theta'}) \right)
} \\
&\quad \cdot 
\exp\Bigg(
- \frac{(2\sqrt{d}+1)}{\eta} \cdot \text{diam}(\bm{\Theta}) 
+ \frac{1}{2} U(\bm{\theta}^{(t)}) 
- \frac{(2 - m\alpha)}{4\alpha} \cdot \text{diam}^2(\bm{\Theta})
\Bigg)
\Bigg]\\
&=\frac{\exp(\sum_{t=1}^{L}\frac{U(\bm{\theta}^{(t-1)})}{2})}{\Bigg(\sum_{\bm{\theta'} \in \bm{\Theta}} \exp \Bigg( \frac{U(\bm{\theta'})}{2}  \Bigg)\Bigg)^{L} }\cdot\\
&\exp\Bigg\{ L\Bigg(( -\frac{M}{2} - \frac{1}{\alpha}+\frac{m}{4}) \text{diam}^2\bm{(\Theta)} - (\frac{1}{2} \|\nabla U(a)\|+\frac{3\sqrt{d}+1}{\eta})\text{diam}\bm{(\Theta)}  \Bigg) \Bigg\}\\
&=\mathcal{\nu}_i(A)\exp\Bigg\{ L\Bigg(( -\frac{M}{2} - \frac{1}{\alpha}+\frac{m}{4}) \text{diam}^2\bm{(\Theta)} - (\frac{1}{2} \|\nabla U(a)\|+\frac{3\sqrt{d}+1}{\eta})\text{diam}\bm{(\Theta)}  \Bigg) \Bigg\}
\end{align*}
}
Thus, we were able to lower bound the L-stepped score-based-denoising kernel without any dependence on $\bm{\theta}_a$.

\end{proof}

\subsection{Proof of Theorem~\ref{thm:hiss}}
\begin{proof}
 Our kernel of interest is the global marginal transition kernel after G gibbs sweeps i.e.
\begin{align*}
\kappa_{\text{(HiSS)}}(\bm{{\theta}}^{(G)} \mid \bm{{\theta}}^{(0)}) 
&= \prod_{i=1}^G \kappa(\bm{{\theta}}^{(i)} \mid \bm{{\theta}}^{(i-1)}) \\
& =\prod_{i=1}^G \int p_{\text{MwG}}(\bm{\theta}_{\text{init}}^{(i)}, \bm{\theta}_a^{(i-1)} \mid \bm{\theta}^{(i-1)}) \cdot p_{\text{DMALA}}(\bm{\theta}^{(i)} \mid [{\bm{\theta}^{(i)}_{\text{init}}}^T, {\bm{\theta}_a^{(i-1)}}^T]^T ) \, d\bm{\theta}_a^{(i-1)}
\end{align*}

In order to lower bound $\kappa_{\text{(HiSS)}}(\bm{{\theta}}^{(G)} \mid \bm{{\theta}}^{(0)})$, we derive a lower bound for $\int p_{\text{MwG}}(\bm{\theta}_{\text{init}}^{(i)}, \bm{\theta}_a^{(i-1)} \mid \bm{\theta}^{(i-1)})d\bm{\theta}_a^{(i-1)}$ and use Lemma \eqref{lemma:2}  to lower bound $p_{\text{DMALA}}(\bm{\theta}^{(i)} \mid [{\bm{\theta}^{(i)}_{\text{init}}}^T, {\bm{\theta}_a^{(i-1)}}^T]^T ) \,$ independent of the auxiliary variable.

By virtue of coordinatewise noising, we can say,
\begin{align*}
q_{\text{noise}}(\bm{\theta}_a \mid \bm{\theta}) \propto \frac{1}{(4\eta)^d} \prod_{i=1}^d \text{sech}^2\left( \frac{(\theta_a)_i - \theta_i}{2\eta} \right) \propto \frac{1}{(4\eta)^d} \, \text{sech}^2\left( \frac{\bm{\theta}_a - \bm{\theta}}{2\eta} \right) 
\end{align*}

Similarly, for denoising,
\begin{align*}
q_{\text{denoise}}(\bm{\theta} \mid \bm{\theta}_a) \propto \prod_{i=1}^d \exp \left( -2\ln  \left( \frac{ (\theta_a)_i - \theta_i}{2\eta} \right) \right)
\propto \text{sech}^2\left( \frac{\bm{\theta}_a - \bm{\theta}}{2\eta} \right)
\end{align*}

We know by definition, 
\begin{align*}
p_{\text{noise-denoise}}(\bm{{\theta}}'_{}|\bm{{\theta}}^{})&\propto\int_{\bm{\theta}_a} q_{\text{noise}}(\bm{{\theta}}_a^{}|\bm{{\theta}}_{})q_{\text{denoise}}(\bm{{\theta}}'^{} |\bm{{\theta}}_a^{}) d\bm{\theta}_a\\
&=\int_{\bm{\theta}_a}\frac{1}{(4\eta)^d} \text{sech}^2(\frac{\bm{\theta}-\bm{\theta}_a}{2\eta})\text{sech}^2(\frac{\bm{\theta}_a-\bm{\theta}'}{2\eta})d\bm{\theta}_a\\
&\geq \frac{1}{(4\eta)^d}\int_{\bm{\theta}_a} e^{- \frac{|\bm{\theta}-\bm{\theta}_a|}{\eta}}\cdot e^{- \frac{|\bm{\theta}_a-\bm{\theta}'|}{\eta}}d\bm{\theta}_a\\
&=\frac{1}{(4\eta)^d}\int_{\bm{\theta}_a} e^{- \frac{|\bm{\theta}-\bm{\theta}'|}{\eta}}\cdot e^{-2 \frac{|\bm{\theta}_a-\frac{\bm{\theta}'+\bm{\theta}}{2}|}{\eta}} d\bm{\theta}_a \\
&= \frac{1}{(4\eta)^d} e^{- \frac{|\bm{\theta}-\bm{\theta}'|}{\eta}}\int_{\bm{\theta}_a \in R^d} e^{-2 \frac{|\bm{\theta}_a-\frac{\bm{\theta}'+\bm{\theta}}{2}|}{\eta}} d\bm{\theta}_a \\
&= \frac{1}{(4\eta)^d} e^{- \frac{|\bm{\theta}-\bm{\theta}'|}{\eta}}(\eta)^d\\
&= \frac{1}{4^d} e^{- \frac{|\bm{\theta}-\bm{\theta}'|}{\eta}}
\end{align*}

\begin{align*}
Z_{\text{noise-denoise}}(\bm{\theta})&=\sum_{x \in \bm{\Theta}}p_{\text{noise-denoise}}(x|\bm{{\theta}}^{})\\
& \geq  \sum_{x \in \bm{\Theta}} \frac{1}{4^d} e^{- \frac{|\bm{\theta}-x|}{\eta}}\\
& \geq \frac{|\bm{\Theta}|}{4^d} e^{- \frac{\Delta(\bm{\Theta})}{\eta}}
\end{align*}

Similarly,
\begin{align*}
p_{\text{noise-denoise}}(\bm{{\theta}}'_{}|\bm{{\theta}}^{})&\propto \int_{\bm{\theta}_a} q_{\text{noise}}(\bm{{\theta}}_a^{}|\bm{{\theta}}_{})q_{\text{denoise}}(\bm{{\theta}}'^{} |\bm{{\theta}}_a^{}) d\bm{\theta}_a\\
&=\int_{\bm{\theta}_a}\frac{1}{(4\eta)^d} \text{sech}^2(\frac{\bm{\theta}-\bm{\theta}_a}{2\eta})\text{sech}^2(\frac{\bm{\theta}_a-\bm{\theta}'}{2\eta})d\bm{\theta}_a\\
&\leq \frac{1}{(4\eta)^d}\int_{\bm{\theta}_a} 4e^{- \frac{|\bm{\theta}-\bm{\theta}_a|}{\eta}}\cdot 4e^{- \frac{|\bm{\theta}_a-\bm{\theta}'|}{\eta}}d\bm{\theta}_a\\
&=\frac{16}{(4\eta)^d}\int_{\bm{\theta}_a} e^{- \frac{|\bm{\theta}-\bm{\theta}'|}{\eta}}\cdot e^{-2 \frac{|\bm{\theta}_a-\frac{\bm{\theta}'+\bm{\theta}}{2}|}{\eta}} d\bm{\theta}_a \\
&=\frac{16}{(4\eta)^d} e^{- \frac{|\bm{\theta}-\bm{\theta}'|}{\eta}}\int_{\bm{\theta}_a \in R^d} e^{-2 \frac{|\bm{\theta}_a-\frac{\bm{\theta}'+\bm{\theta}}{2}|}{\eta}} d\bm{\theta}_a \\
&= \frac{16}{(4\eta)^d} e^{- \frac{|\bm{\theta}-\bm{\theta}'|}{\eta}}(\eta)^d\\
&= \frac{16}{4^d} e^{- \frac{|\bm{\theta}-\bm{\theta}'|}{\eta}}
\end{align*}

\begin{align*}
Z_{\text{noise-denoise}}(\bm{\theta})&=\sum_{x \in \bm{\Theta}}p_{\text{noise-denoise}}(x|\bm{{\theta}}^{})\\
& \leq  \sum_{x \in \bm{\Theta}} \frac{16}{4^d} e^{- \frac{|\bm{\theta}-x|}{\eta}}\\
& \leq \frac{16}{4^d}{|\bm{\Theta}|}
\end{align*}

We also note that for any arbitrary $\bm{\theta}, \bm{\theta}' \in \bm{\Theta}$, the following holds true because of Assumption \ref{assumption:5.1}.
    \begin{align*}
        -\left\langle\nabla U(\bm{\theta}),\bm{\theta}'-\bm{\theta}\right\rangle
        & =  \left\langle - \nabla U(\bm{\theta}) + \nabla U(a),\bm{\theta}'-\bm{\theta}\right\rangle+  \left\langle - \nabla U(a),\bm{\theta}'-\bm{\theta}\right\rangle \\
        & \le  \left\langle - \nabla U(\bm{\theta}) + \nabla U(a),\bm{\theta}'-\bm{\theta}\right\rangle+  \left\langle - \nabla U(a),\bm{\theta}'-\bm{\theta}\right\rangle \\
        & \le  \left\| - \nabla U(\bm{\theta}) + \nabla U(a)\| \| \bm{\theta}'-\bm{\theta}\right\|+  \left\|\nabla U(a) \| \| \bm{\theta}'-\bm{\theta}\right\|  \\
        & \le  \| - \nabla U(\bm{\theta}) + \nabla U(a)\|   diam(\bm{\Theta}) +  \|\nabla U(a) \| diam(\bm{\Theta}) \\
        & \le \left( M\right)\, diam(\bm{\Theta})^2+ \|\nabla U(a)\|\, diam(\bm{\Theta}).
    \end{align*}

From \eqref{eq:MH:acc:rate} in Section \ref{sec:method}, we can see
\begin{align*}
a_{\text{init}}(\bm{{\theta}}'_{\text{init}}|\bm{{\theta}}^{(i-1)})&= \min\left(1, \frac{\pi(\bm{{\theta}}'_{\text{init}})q_{\text{noise}}(\bm{{\theta}}_a^{(i-1)}|\bm{{\theta}}'_{\text{init}})q_{\text{denoise}}(\bm{{\theta}}^{(i-1)} \mid \bm{{\theta}}_a^{(i-1)})}{\pi(\bm{{\theta}}^{(i-1)})q_{\text{noise}}(\bm{{\theta}}_a^{(i-1)}|\bm{{\theta}}^{(i-1)})q_{\text{denoise}}(\bm{{\theta}}'_{\text{init}} \mid \bm{{\theta}}_a^{(i-1)})}\right)\\
&= \min\left(1, \frac{\pi(\bm{{\theta}}'_{\text{init}})p_{\text{noise-denoise}}(\bm{{\theta}}^{(i-1)}|\bm{{\theta}}'_{\text{init}})}{\pi(\bm{{\theta}}^{(i-1)})p_{\text{noise-denoise}}(\bm{{\theta}}'_{\text{init}}|\bm{{\theta}}^{(i-1)})}\right)\\
&\geq\frac{\pi(\bm{{\theta}}'_{\text{init}})}{\pi(\bm{{\theta}}^{(i-1)})} \cdot \frac{p_{\text{noise-denoise}}(\bm{{\theta}}^{(i-1)}|\bm{{\theta}}'_{\text{init}})}{ p_{\text{noise-denoise}}(\bm{{\theta}}'_{\text{init}}|\bm{{\theta}}^{(i-1)})}\cdot\frac{{Z}}{{Z}}\cdot \frac{Z_{\text{noise-denoise}}(\bm{\theta}^{(i-1)})}{Z_{\text{noise-denoise}}( \bm{\theta}'_{\text{init}})} \\
& \geq \exp\left(U(\bm{\theta}'_{\text{init}}) - U(\bm{\theta}^{(i-1)}) \right)\frac{\frac{1}{4^d} e^{- \frac{|\bm{\theta}'_{\text{init}} - \bm{\theta}^{(i-1)}|}{\eta}}}{\frac{16}{4^d} e^{- \frac{|\bm{\theta}^{(i-1)}-\bm{\theta}'_{\text{init}}|}{\eta}}}\cdot\frac{\frac{|\bm{\Theta}|}{4^d} e^{- \frac{\Delta(\bm{\Theta})}{\eta}}}{{\frac{16|\bm{\Theta}|}{4^d}}}\\
& \geq  \frac{1}{2^8} \exp\left(\left\langle\nabla U(\bm{\theta}^{(i-1)}),\bm{\theta}'_{\text{init}} -\bm{\theta}^{(i-1)}\right\rangle\right)\cdot\exp \left(-\frac{\Delta(\bm{\Theta})}{\eta} \right)\\
&\geq \frac{1}{2^8}\exp\left( -Mdiam(\bm{\Theta})^2 - \|\nabla U(a)\|diam(\bm{\Theta}) - \frac{\Delta(\bm{\Theta})}{\eta} \right)
\end{align*}

Putting everything together,
\begin{align*}
\int p_{\text{MwG}}(\bm{\theta}_{\text{init}}^{(i)}, \bm{\theta}_a^{(i-1)} \mid \bm{\theta}^{(i-1)}) \, d\bm{\theta}_a^{(i-1)} 
&= \int \, q_{\text{noise}}(\bm{\theta}_a^{(i-1)} \mid \bm{\theta}^{(i-1)})q_{\text{denoise}}(\bm{\theta}'_{\text{init}} \mid \bm{\theta}_a^{(i-1)}) a_{\text{init}}(\bm{{\theta}}'_{\text{init}}|\bm{{\theta}}^{(i-1)}) \\
&\quad + \left(1-L(\bm{\widetilde{\theta}}^{(i-1)})\right)\delta(\bm{{\theta}}^{(i)}_{\text{init}}) \, d\bm{\theta}_a^{(i-1)} \\
&\geq \int \, q_{\text{noise}}(\bm{\theta}_a^{(i-1)} \mid \bm{\theta}^{(i-1)})q_{\text{denoise}}(\bm{\theta}'_{\text{init}} \mid \bm{\theta}_a^{(i-1)}) a_{\text{init}}(\bm{{\theta}}'_{\text{init}}|\bm{{\theta}}^{(i-1)}) d\bm{\theta}_a^{(i-1)} \\
&= a_{\text{init}}(\bm{{\theta}}'_{\text{init}}|\bm{{\theta}}^{(i-1)}) \int \, q_{\text{noise}}(\bm{\theta}_a^{(i-1)} \mid \bm{\theta}^{(i-1)})q_{\text{denoise}}(\bm{\theta}'_{\text{init}} \mid \bm{\theta}_a^{(i-1)}) d\bm{\theta}_a^{(i-1)} \\
&= a_{\text{init}}(\bm{{\theta}}'_{\text{init}}|\bm{{\theta}}^{(i-1)}) \frac{p_{\text{noise-denoise}}(\bm{{\theta}}'_{\text{init}}|\bm{{\theta}}^{(i-1)})}{Z_{\text{noise-denoise}}(\bm{\theta}^{(i-1)})} \\
&\geq \frac{1}{2^8} \cdot \frac{ \frac{1}{4^d} e^{- \frac{|\bm{\theta}^{(i-1)}-\bm{\theta}'_{\text{init}}|}{\eta}} }{ \frac{16}{4^d}|\bm{\Theta}| } \exp\left( -M\text{diam}^2(\bm{\Theta}) - \|\nabla U(a)\|\text{diam}(\bm{\Theta}) - \frac{\Delta(\bm{\Theta})}{\eta} \right) \\
&\geq \frac{1}{2^{12}|\bm{\Theta}|} \exp\left( -M\text{diam}^2(\bm{\Theta}) - \|\nabla U(a)\|\text{diam}(\bm{\Theta}) -2 \frac{\Delta(\bm{\Theta})}{\eta} \right)
\end{align*}
Thus,
\begin{equation}\label{eq:lb1}
\int p_{\text{MwG}}(\bm{\theta}_{\text{init}}^{(i)}, \bm{\theta}_a^{(i-1)} \mid \bm{\theta}^{(i-1)}) \, d\bm{\theta}_a^{(i-1)}\geq \frac{1}{2^{12}|\bm{\Theta}|} \exp\left( -M\text{diam}^2(\bm{\Theta}) - \|\nabla U(a)\|\text{diam}(\bm{\Theta}) -2 \frac{\Delta(\bm{\Theta})}{\eta} \right)
\end{equation}

Combining Equation \eqref{eq:lb1} and Lemma \ref{lemma:2}, we see,
\begin{align*}
\kappa_{\text{(HiSS)}}(\bm{{\theta}}^{(G)} \mid \bm{{\theta}}^{(0)}) 
&= \prod_{i=1}^G \kappa(\bm{{\theta}}^{(i)} \mid \bm{{\theta}}^{(i-1)}) \\
& =\prod_{i=1}^G \int p_{\text{MwG}}(\bm{\theta}_{\text{init}}^{(i)}, \bm{\theta}_a^{(i-1)} \mid \bm{\theta}^{(i-1)}) \cdot p_{\text{DMALA}}(\bm{\theta}^{(i)} \mid [{\bm{\theta}^{(i)}_{\text{init}}}^T, {\bm{\theta}_a^{(i-1)}}^T]^T ) \, d\bm{\theta}_a^{(i-1)}\\
& \geq \prod_{i=1}^G \exp\Bigg\{ ( -\frac{ML}{2} - \frac{L}{\alpha}+\frac{mL}{4}) \text{diam}^2\bm{(\Theta)} - (\frac{L}{2} \|\nabla U(a)\|+\frac{3L\sqrt{d}+L}{\eta})\text{diam}\bm{(\Theta)} \Bigg\}\cdot\\
&\int p_{\text{MwG}}(\bm{\theta}_{\text{init}}^{(i)}, \bm{\theta}_a^{(i-1)} \mid \bm{\theta}^{(i-1)}) d\bm{\theta}_a^{(i-1)}\\
&\geq \prod_{i=1}^G\Bigg[ \mathcal{\nu}_i(A)\exp\Bigg\{ ( -\frac{ML}{2} - \frac{L}{\alpha}+\frac{mL}{4}) \text{diam}^2\bm{(\Theta)} -(\frac{L}{2} \|\nabla U(a)\|+\frac{3L\sqrt{d}+L}{\eta})\text{diam}\bm{(\Theta)} \Bigg\}\cdot\\
&\frac{\exp\left( -M\text{diam}^2(\bm{\Theta}) - \|\nabla U(a)\|\text{diam}(\bm{\Theta}) -2 \frac{\Delta(\bm{\Theta})}{\eta} \right)}{2^{12}|\bm{\Theta}|}\Bigg]\\
&\geq \frac{1}{2^{12G}|\bm{\Theta}|^G}\cdot\prod_{i=1}^G\Bigg[ \mathcal{\nu}_i(A) \exp\Bigg\{ ( -\frac{ML}{2} - \frac{L}{\alpha}+\frac{mL}{4}) \text{diam}^2\bm{(\Theta)} - (\frac{L}{2} \|\nabla U(a)\|\\
&+\frac{3L\sqrt{d}+L}{\eta})\text{diam}\bm{(\Theta)} \Bigg\}\cdot\exp\left( -M\text{diam}^2(\bm{\Theta}) - \|\nabla U(a)\|\text{diam}(\bm{\Theta}) -\frac{2\sqrt{d}}{\eta} \text{diam}(\bm{\Theta}) \right)\Bigg]\\
&=\mathcal{\nu}(A')\cdot\exp\Bigg\{ ( -M(\frac{LG}{2}+G) - \frac{LG}{\alpha}+\frac{mLG}{4}) diam(\bm{\Theta})^2 + (\frac{3LG\sqrt{d}+LG-2G\sqrt{d}}{\eta} - \\
&(\frac{LG}{2}+G) \|\nabla U(a)\|)diam(\bm{\Theta})  \Bigg\}  \\
&=\mathcal{\nu}(A')\cdot\exp\Bigg\{ ( -M(\frac{LG}{2}+G) - \frac{LG}{\alpha}+\frac{mLG}{4}) diam(\bm{\Theta})^2 + (\frac{  G\sqrt{d}(3L-2)+LG}{\eta} - \\
&(\frac{LG}{2}+G) \|\nabla U(a)\| )diam(\bm{\Theta})  \Bigg\} 
\end{align*}
where $ \mathcal{\nu}(A') \in [0,1]$ is a valid probability measure over L trajectories across G Gibbs sweeps i.e. $A' \subseteq \bm{\Theta}^{L\times G}$.

\end{proof}
\begin{corollary}[Analytical Condition for Fast Mixing]
\label{cor:mixing}
Under the assumptions of Theorem \ref{thm:hiss2}, suppose the proposal scale \( \eta > 0 \), and the number of Gibbs sweeps \( G \ge 1 \), refinements per sweeps \( L \ge 1 \), and parameter space diameter \( \operatorname{diam}(\bm{\Theta}) > 0 \).
Also, let $m<4M$ and  $\alpha< \frac{2}{M}$. In particular, fast mixing (i.e., large \( \epsilon_\alpha \)) occurs whenever
\[
\|\nabla U(a)\| \leq 
\frac{ \sqrt{d}(3L - 2) + L }{
\eta \left( \frac{L}{2} + 1 \right) }.
\]
\end{corollary}

\begin{proof}
Let convergence factor be expressed as $$
\epsilon_\alpha = \exp\left\{
A \cdot \operatorname{diam}(\bm{\Theta})^2 + B \cdot \operatorname{diam}(\bm{\Theta})
\right\}
$$
where
\begin{align*}
A &= -M\left( \frac{LG}{2} + G \right) - \frac{LG}{\alpha} + \frac{mLG}{4}, \\
B &= \frac{G\sqrt{d}(3L - 2) + LG}{\eta} - \left( \frac{LG}{2} + G \right) \cdot \|\nabla U(a)\|.
\end{align*}

For faster mixing, the total variation distance must decay quickly. Thus,$\epsilon_\alpha\rightarrow1$, making $\left\{
A \cdot \operatorname{diam}(\bm{\Theta})^2 + B \cdot \operatorname{diam}(\bm{\Theta})
\right\}\rightarrow 0 $.

We know, \( \operatorname{diam}(\bm{\Theta}) > 0 \). Therefore,  $\operatorname{diam}(\bm{\Theta})=-\frac{B}{A}>0$.

Trivially, if $m< 4M$ and $\alpha< \frac{2}{M}$,then $A<0$.
$B$ becomes non-negative when,
$$\|\nabla U(a)\|<{\frac{\sqrt{d}(3L - 2) + L}{\eta(\frac{L}{2}+1)}} $$
For fixed values of $\eta$, $L$, and $d$, favorable local mixing is possible only if the flattest region of the energy landscape is not too steep i.e. ($\|\nabla U(a)\|\rightarrow0$). In other words, if the pull-back from the gradient is less than exploration strength of the sampler, chain mixes well. This inherently shows HiSS's success in disconnected energy regimes.

In contrast, a symmetric analysis of DMALA (\cite{pynadath2024gradientbased}) reveals that local mixing improves only when \(\|\nabla U(a)\| > 0\), which relies on the assumption that the energy landscape is very well connected.
\end{proof}

\section{Additional Experimental Results\label{sec:app_add}}
\subsection*{Coverage}\label{def:cov}
Let \( \bm{\Theta} \) denote the set of all possible discrete states with cardinality \( |\bm{\Theta}| \), and let:
\begin{itemize}
    \item \( S = \{\bm{\theta}_1, \bm{\theta}_2, \dots, \bm{\theta}_N\} \) be the set of samples generated by the MCMC sampler after \( N \) iterations.
    \item \( \mathcal{V}(S) \subseteq \bm{\Theta} \) represent the subset of states in \( \bm{\Theta} \) visited at least once by the sampler.
\end{itemize}

The \textbf{coverage} \( C \) of the MCMC sampler is defined as:
\[
C = \frac{|\mathcal{V}(S)|}{|\bm{\Theta}|},
\]
where:
\begin{itemize}
    \item \( |\mathcal{V}(S)| \) is the number of unique states visited by the sampler.
    \item \( |\bm{\Theta}| \) is the total number of discrete states.
\end{itemize}

\subsection*{Properties of Coverage}
\begin{itemize}
    \item \textbf{Range}: \( C \in [0, 1] \):
    \begin{itemize}
        \item \( C = 0 \): No states were visited, i.e., \( |\mathcal{V}(S)| = 0 \).
        \item \( C = 1 \): All possible states were visited, i.e., \( |\mathcal{V}(S)| = |\bm{\Theta}| \).
    \end{itemize}
    \item \textbf{Interpretation}:
    \begin{itemize}
        \item \( C \) represents the fraction of the state space explored by the sampler.
        \item A higher \( C \) indicates better exploration and diversity of the sampled states.
    \end{itemize}
\end{itemize}
 \subsection{4D Joint Bernoulli} \label{sec:app_bern}
For Section \ref{sec:exp}, we manually tuning the step size $\alpha$ to 0.2. For HiSS, we set $G = 5, L = 2,$ and $\eta=4$. To provide additional insights into the functionality of HiSS, we explore their behavior on the 4D Joint Bernoulli Distribution, which serves as the simplest low-dimensional case among our experiments. This aids in visualizing and understanding the sampling process.

\subsubsection*{Target Distribution}
The following represents the probability mass function (PMF) for the 4D Joint Bernoulli Distribution used in our test case. The distribution has 16 states with the corresponding probabilities:

\begin{figure}[h]
    \centering
    \begin{minipage}{0.4\textwidth}
\[
P_{\bm{\Theta}}(\bm{\theta}) =
\begin{cases} 
0.588204 & \text{if } \bm{\theta} = 0000, \\
5.882e-6 & \text{if } \bm{\theta} = 0001, \\
5.882e-6 & \text{if } \bm{\theta} = 0010, \\
5.882e-6 & \text{if } \bm{\theta} = 0011, \\
5.882e-6 & \text{if } \bm{\theta} = 0100, \\
5.882e-6 & \text{if } \bm{\theta} = 0101, \\
5.882e-6 & \text{if } \bm{\theta} = 0110, \\
5.882e-6 & \text{if } \bm{\theta} = 0111, \\
5.882e-6 & \text{if } \bm{\theta} = 1000, \\
5.882e-6 & \text{if } \bm{\theta} = 1001, \\
5.882e-6 & \text{if } \bm{\theta} = 1010, \\
5.882e-6 & \text{if } \bm{\theta} = 1011, \\
5.882e-6 & \text{if } \bm{\theta} = 1100, \\
5.882e-6 & \text{if } \bm{\theta} = 1101, \\
0.294102 & \text{if } \bm{\theta} = 1110, \\
0.117641 & \text{if } \bm{\theta} = 1111. \\
\end{cases}
\]
    \end{minipage}%
    \begin{minipage}{0.59\textwidth }
        \centering
        \includegraphics[width=\linewidth]{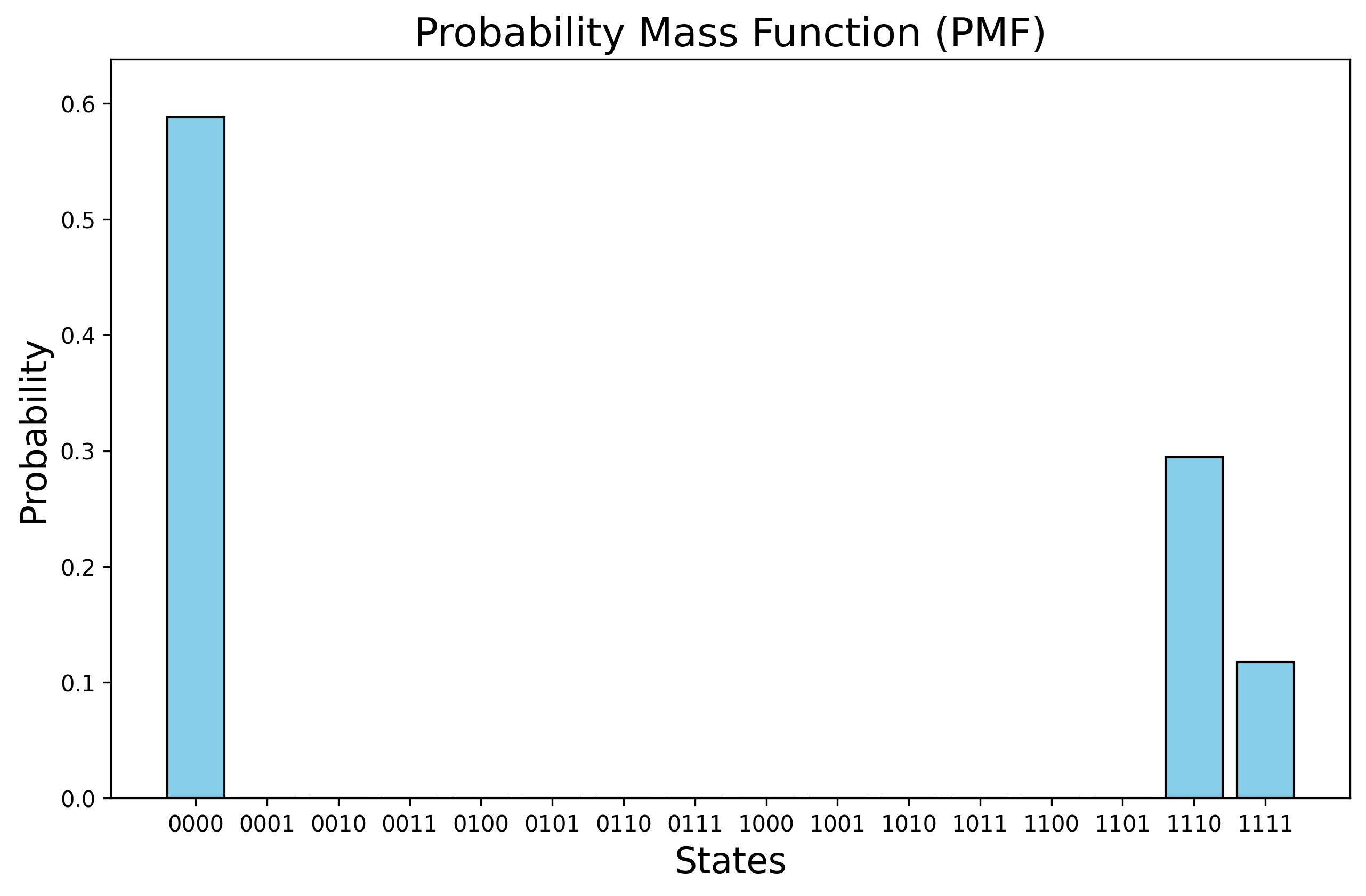}
        \caption{Target Distribution for 4D Joint Bernoulli}
        \label{fig:bern_target}
    \end{minipage}
\end{figure}

The state  $\bm{\theta} = 0000$  has the highest probability ( $P_{\bm{\Theta}}(0000) = 0.588204$ ), indicating it is the dominant mode. States  $\bm{\theta} = 1110$  and  $\bm{\theta} = 1111$  have moderate probabilities ( 0.294102  and  0.117641 , respectively), while all other states have extremely low probabilities ( $5.882 \times 10^{-6}$ ), highlighting a multimodal distribution with sharp peaks and a long tail of negligible probabilities.

\subsubsection*{Coverage Analysis}
As shown in Figure \ref{fig:bern_cov}, HiSS exhibits a clear upward trend in average coverage throughout the iterations, showcasing its superior ability to explore diverse modes effectively. This improvement can be attributed to the MwG sweep mechanism, which enhances its exploratory capacity. In contrast, other samplers, such as GWG, DMALA, ACS, and PT+DMALA, appear to stagnate early in the iterations, failing to escape initial regions and achieve broader coverage.

 \begin{figure}[ht]
    \centering
    \includegraphics[width=0.6\textwidth]{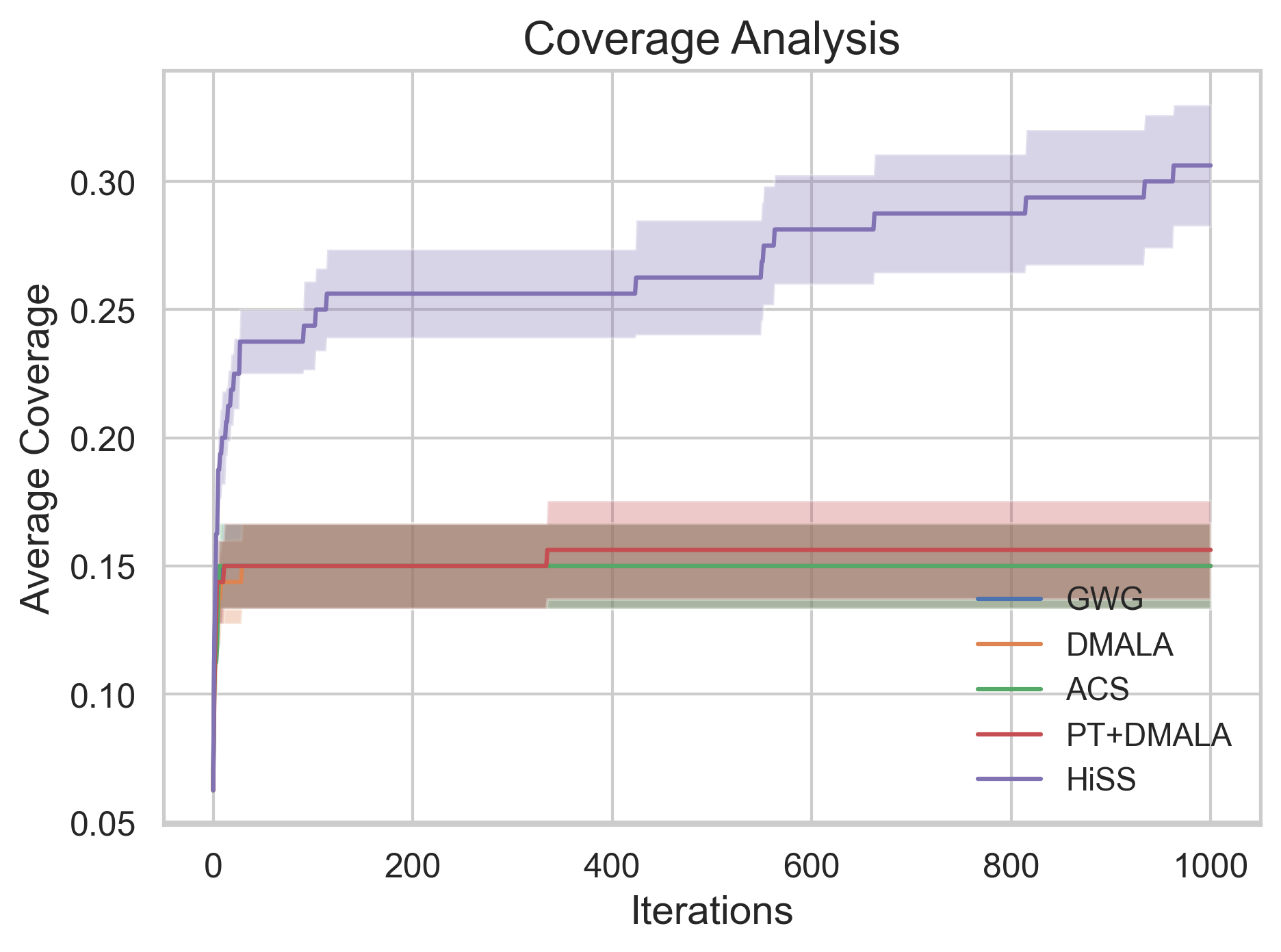}
    \caption{Coverage Analysis for 4D Bernoulli}
    \label{fig:bern_cov}
\end{figure} 

 \subsection{Ising Sampling}\label{sec:app_ising}
 \subsubsection*{Target Distribution}
For Section \ref{sec:exp}, for HiSS, we set $G=10, L=2, \alpha=0.2$, and $\eta=4$. For our setup, the interaction matrix $\mathbf{W}$ is essentially a cross-diagonal matrix. This is intentional, causing the Ising Model to be sparse, breaking the symmetry-like effect. $\mathbf{W}$'s construction is motivated by the study of frustrated and anisotropic systems\citep{edwards1975theory, chaikin1995principles}, where competing interactions and directional dependencies govern dynamics. For \( d = 9 \), the Ising model contains \( 2^9 = 512 \) discrete states. The probability distribution, visualized in Figure~\ref{fig:ising_dist}, reveals that 32 prominent states dominate the landscape, accounting for \( \frac{32}{512} = 6.25\% \) of the total state space. These high-probability states illustrate the multimodal nature of the model, where efficient sampling requires the ability to transition effectively between modes.

 \begin{figure}[ht]
    \centering
    \includegraphics[width=0.6\textwidth]{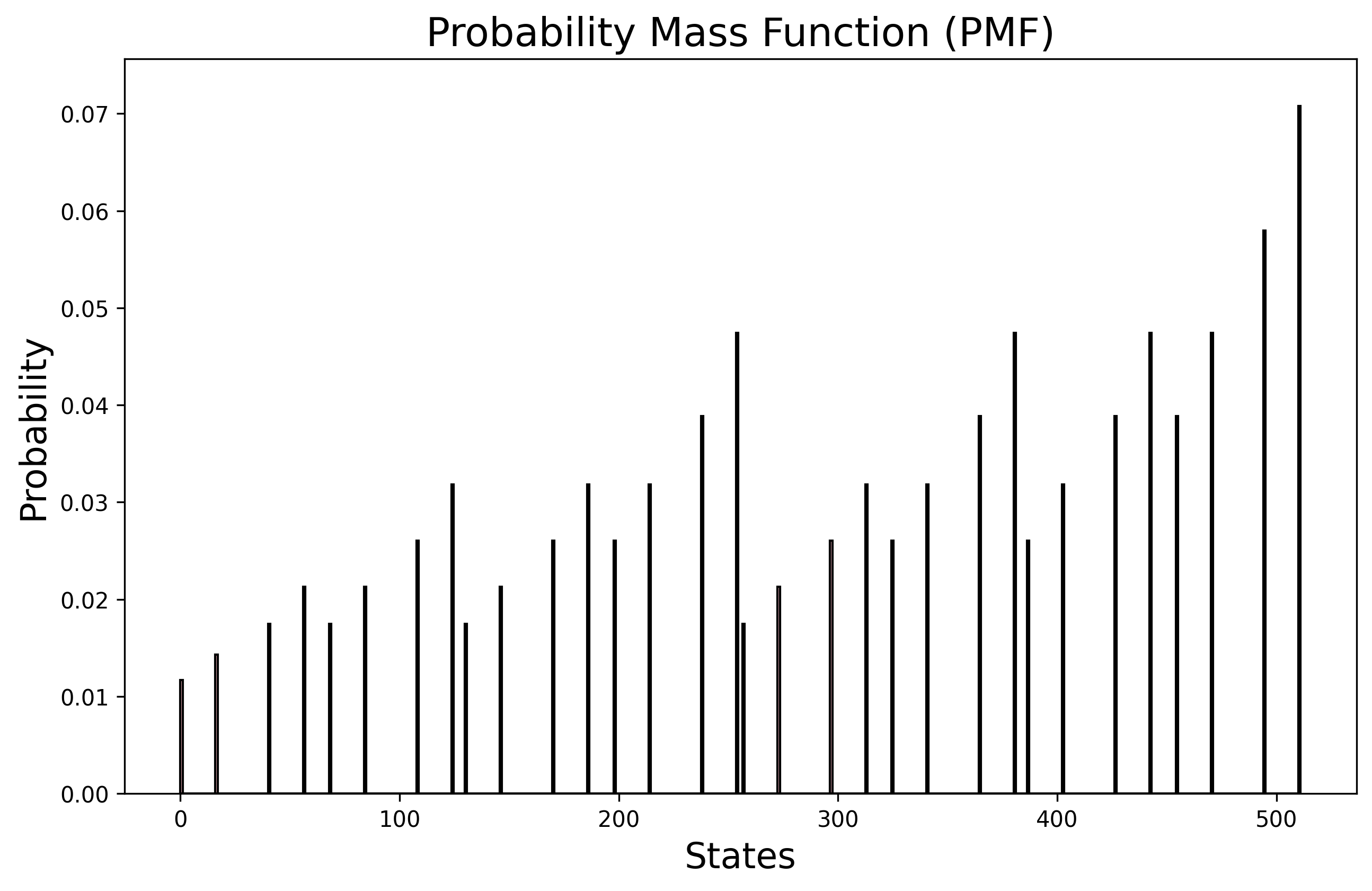}
    \caption{Ising Model Distribution}
    \label{fig:ising_dist}
\end{figure}

 \subsubsection*{Coverage Analysis}
In Figure~\ref{fig:ising_cov}, HiSS demonstrates excellent performance, rapidly converging to cover these prominent states. Its upward trajectory showcases its ability to explore the multimodal landscape efficiently, achieving the theoretical limit of 0.0625 well before other samplers. This success can be attributed to mechanisms like the MwG Gibbs sweep, which ensures good mixing and fast convergence. While PT+DMALA achieves comparable coverage in the long run, its inefficiency at earlier iterations highlights its limitations for tasks requiring faster convergence. In contrast, samplers like GWG, DMALA, and ACS struggle to escape initial regions of the probability landscape and fail to achieve sufficient mixing. These methods stall early, underscoring their inability to effectively explore the multimodal nature of the Ising model.

 \begin{figure}[ht]
    \centering
    \includegraphics[width=0.6\textwidth]{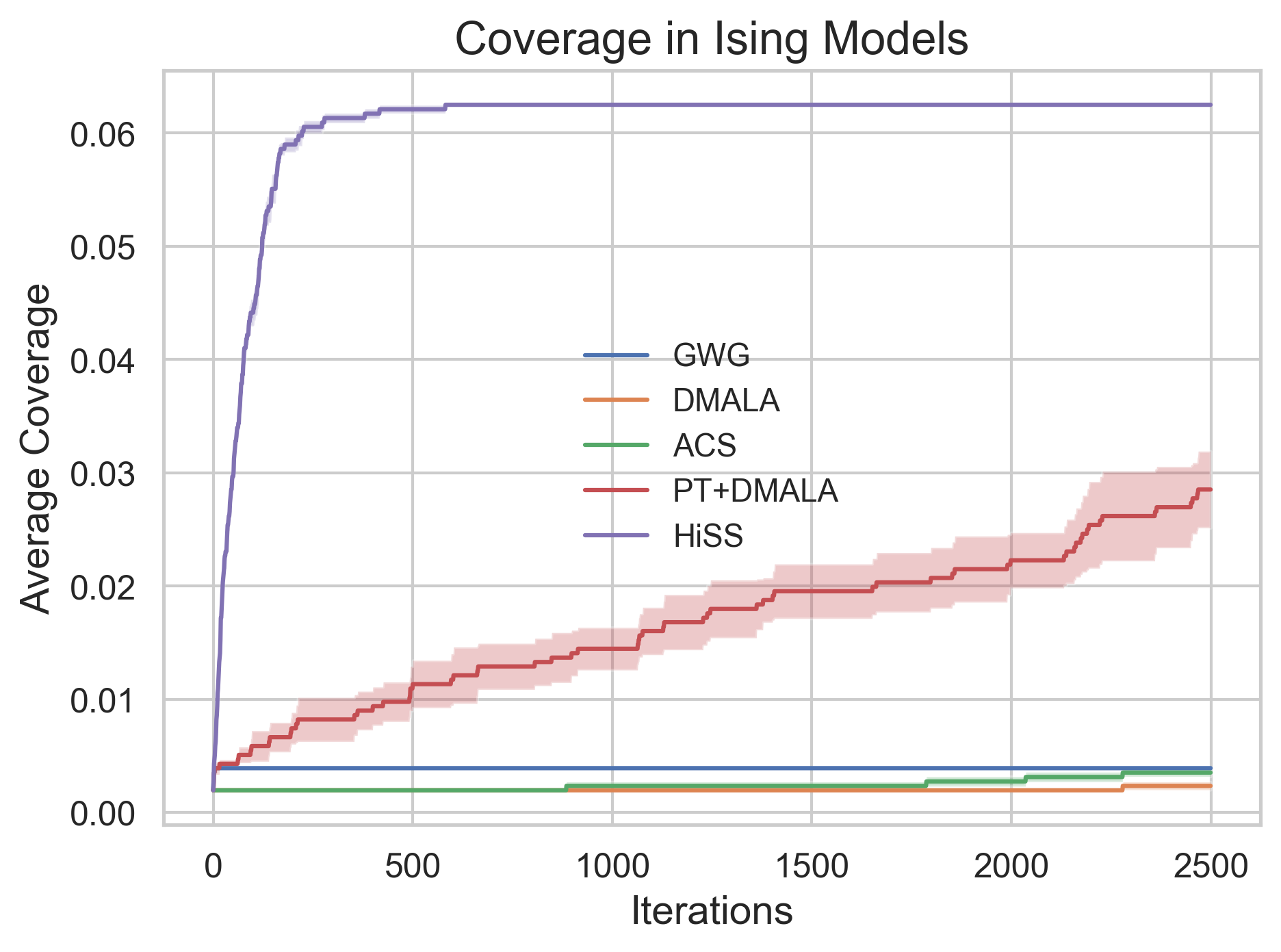}
    \caption{Coverage Analysis for Ising Model}
    \label{fig:ising_cov}
\end{figure}

\subsubsection*{Tuning $G$ and $L$\label{sec:add:gl_tune}}
Increasing G enhances global mixing by allowing more opportunities to transition between modes, ensuring broader coverage of the state space. However, excessive G without sufficient refinement can lead to a \emph{jumpy} process. Conversely, increasing L enables smoother and more informed transitions, improving local convergence. Yet, overly large L risks making transitions too deterministic, potentially trapping the chain in local modes. Thus, the optimal configuration depends on the energy landscape: smoother distributions benefit from higher L and moderate G, whereas highly disconnected, multimodal landscapes require a larger G with appropriately scaled $\eta$. Parameter selection should therefore be guided by the structure of the target distribution and the desired tradeoff between exploration and exploitation.

To study this phenomena, we fixed the product $G\times L$ and evaluated all integer factor combinations (e.g., $G=1, L=50; G=2, L=25; \cdots , G=50, L=1$) for a fixed $\alpha$ and $\eta$. We then plotted the Total Variational Distance versus runtime. The observed trend aligns with our theoretical intuition: higher G and lower L enhance convergence speed (due to improved mode mixing) but come with increased runtime overhead (due to more Metropolis-within-Gibbs steps). Conversely, lower G and higher L slow down convergence but reduce computational cost (as illustrated in Figure \ref{fig:ising_side_by_side}).
\begin{figure}[ht]
    \centering
    
    \begin{minipage}[t]{0.48\textwidth}
        \centering
        \includegraphics[width=\linewidth]{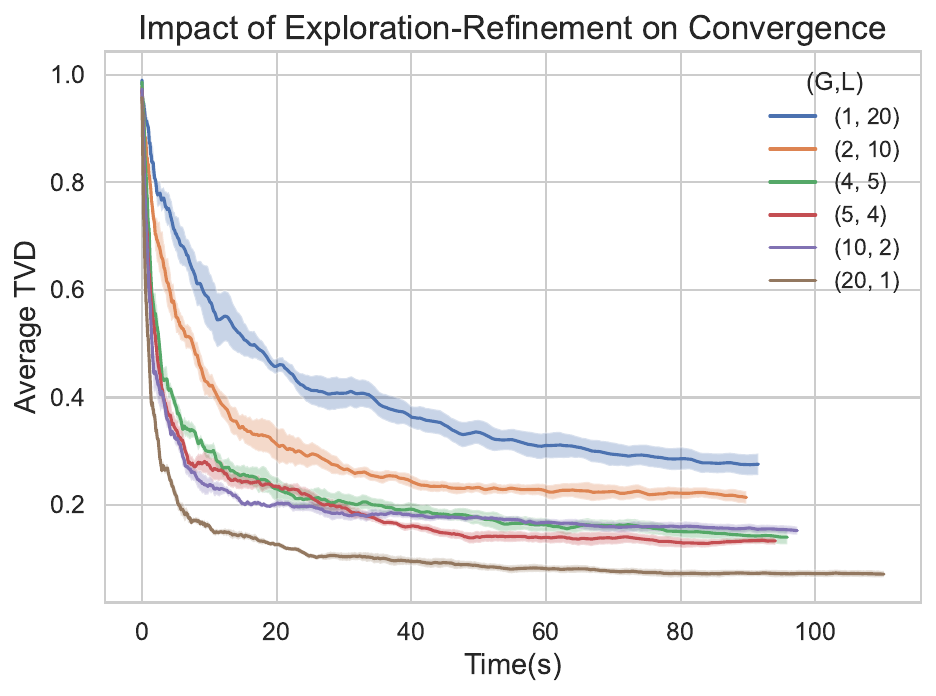}
        \label{fig:ising8}
    \end{minipage}
    \hfill
    \begin{minipage}[t]{0.48\textwidth}
        \centering
        \includegraphics[width=\linewidth]{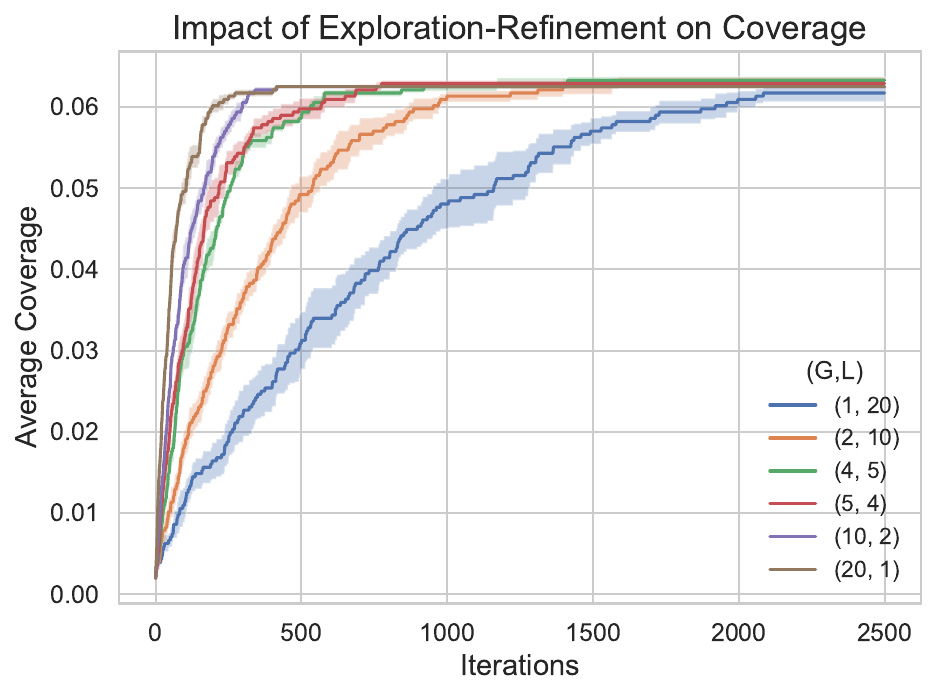}
        \label{fig:ising9}
    \end{minipage}

    \caption{Impact of Gibbs Sweeps and Refinement Iterations in Ising models.}
    \label{fig:ising_side_by_side}
\end{figure}

\subsubsection*{Criticality}
To investigate the limits of gradient-based sampling, we analyze 2D Ising model under Criticality($\beta_c=0.4407$) for $d=24\times24$, with $\alpha=\frac{\beta_c}{2}$ accounting for the double-counting of bonds in the quadratic form and $b=0$. Interestingly, we observed that HiSS and the baseline DMALA exhibit almost identical performance, with both converging to the theoretical internal energy ($\approx -1.4402$). This suggests that while critical systems suffer from `critical slowing down', they do not necessarily exhibit the disconnected energy landscape that traps gradient samplers. The gradients in the critical Ising model still provide a valid path for global exploration( Figure \ref{fig:ising_crit}).

 \begin{figure}[ht]
    \centering
    \includegraphics[width=0.6\textwidth]{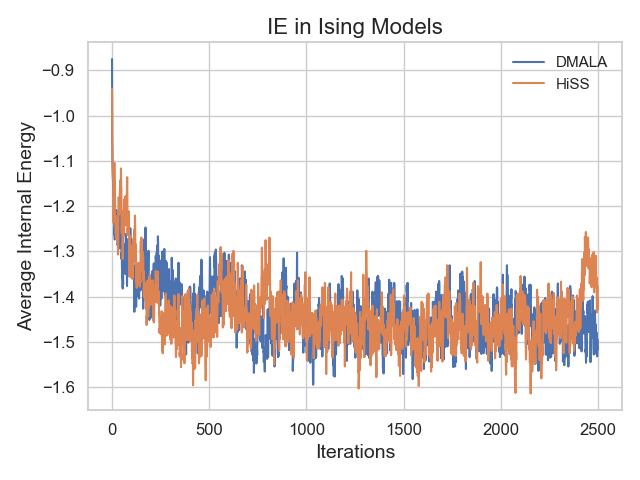}
    \caption{Criticality Ising Model}
    \label{fig:ising_crit}
\end{figure}

\subsection{Computational Complexity Analysis (Energy Evaluations)}
While Section \ref{sec:exp} provides a practical convergence analysis based on wall-clock runtime (which accounts for implementation overheads, memory access, and communication latency), it is also valuable to analyze the theoretical computational strain of each sampler in terms of Total Number of Function Evaluations (TNFE). We define the Total NFE as the aggregate number of energy function calls required to generate the full set of samples used in our results. In many scientific applications, the energy function $U(\bm\theta)$ is the computational bottleneck, making NFE a critical metric independent of hardware optimizations~\citep{doi:10.1126/science.aaw1147, song2021scorebased}. 

We define the cost of a single gradient-based transition (e.g., one step of DMALA or GWG) as $\mathcal{C}_{grad} = 4$ energy calls:

Gradient Calculation: Requires 2 evaluations (forward and backward passes, or function evaluations at $\bm\theta$ and new $\bm\theta$) to approximate or compute $\nabla_{\bm\theta} U(\bm\theta)$.
Metropolis-Hastings (MH) Correction: Requires 2 evaluations to compute the energy of the proposed state $U(\bm\theta')$ and the current state $U(\bm\theta)$ for the acceptance ratio.

For all experiments, we collect $N_S$ samples. Let $N$ be number of independent chains (batch size. We assume a fixed budget of refinement steps per sample across baselines to ensure parity.
\begin{itemize}
    \item Baseline Gradient Samplers (DMALA, GWG, ACS)For standard gradient-based samplers, the process is a sequential application of the kernel. With $S/GL$ refinement steps per sample:
$$
\text{NFE}_{base} = N\times N_S \times S \times \mathcal{C}_{grad}
$$

\item HiSS (Proposed Method)
HiSS introduces a hierarchical structure with an outer loop ($G$ sweeps) and an inner refinement loop ($L$ steps). The cost per sweep includes the MH step for the denoised proposal (2 evaluations) and the gradient refinement ($L \times \mathcal{C}_{grad}$).
$$
\text{NFE}_{HiSS} = N \times N_S \times G \times [2 + (L \times \mathcal{C}_{grad})]
$$

\item  Parallel Tempering (PT+DMALA)
Parallel Tempering incurs significantly higher computational costs due to the maintenance of multiple auxiliary chains at higher temperatures. For a batch of $T$ chains, each utilizing $K$ temperature levels, with swap attempts every $I$ steps, the cost is twofold:

Refinement Cost: All $K$ replicas must undergo Langevin dynamics.
Swap Cost: Every $I$ steps, energy differences between adjacent chains must be computed to satisfy the exchange criterion.

$$
\text{NFE}_{PT} = \underbrace{N_S \times N \times K \times S \times \mathcal{C}_{grad}}_{\text{Thermodynamic Refinement}} + \underbrace{\left\lfloor \frac{N_S \cdot S}{I} \right\rfloor \times N \times (K-1) \times 2}_{\text{Swap Communication}}
$$
\end{itemize}

For 4D Joint Bernoulli Sample, $N=10, N_S=10^3$, $S=10\quad(G=5, L=2)$, $K=5$ temperatures, $I=4$. 

\begin{table}[ht]
\centering
\normalsize
\renewcommand{\arraystretch}{0.1}
\caption{Number of Energy Evaluations in 4D Joint Bernoulli}
\resizebox{0.45\linewidth}{!}{
\begin{tabular}{|l|c|c|}
\hline
\textbf{Sampler} 
& \textbf{\# Energy Evaluations} \\
\hline
GWG    
& $4.0\times10^5 $ \\
\hline
DMALA   
& $4.0\times10^5$ \\
\hline
ACS   
& $4.0\times10^5 $  \\
\hline
PT+DMALA  
& ${2.2\times10^6}$ \\
\hline
HiSS  
& ${5.0\times10^5}$ \\
\hline
\end{tabular}
}
\label{tab:bern_energy-eval}
\end{table}

From Table \ref{tab:bern_energy-eval}, 
Conclusion:
While HiSS incurs a marginal increase in NFE compared to pure DMALA ($5 \times 10^5$ vs $4 \times 10^5$) to support the auxiliary variable mechanics;  PT requires $2.2 \times 10^6$ NFEs( a 4.4$\times$ with respect to HiSS increase in raw energy computations). This confirms that PT's "slower" wall-clock convergence (Figure \ref{fig:berneval}) is not just due to communication overhead, but due to the sheer volume of wasted computation on high-temperature auxiliary chains that are discarded during inference. HiSS achieves mode-hopping more efficiently by using a single continuous auxiliary variable rather than $K$ discrete replicas. TNFE values for Ising Models are reported in Table \ref{tab:ising_energy-eval}. 

\begin{table}[ht]
\centering
\normalsize
\renewcommand{\arraystretch}{0.1}
\caption{Number of Energy Evaluations in Ising Models}
\resizebox{0.45\linewidth}{!}{
\begin{tabular}{|l|c|c|}
\hline
\textbf{Sampler} 
& \textbf{\# Energy Evaluations} \\
\hline
GWG    
& $1.00\times10^6 $ \\
\hline
DMALA   
& $1.00\times10^6$ \\
\hline
ACS   
& $1.00\times10^6 $  \\
\hline
PT+DMALA  
& ${6.00\times10^6}$ \\
\hline
HiSS  
& ${1.25\times10^6}$ \\
\hline
\end{tabular}
}
\label{tab:ising_energy-eval}
\end{table}

 \subsection{Traveling Salesman Problem} \label{sec:app_tsp}
For results presented under Section \ref{sec:exp}, for HiSS, we set $G=10$, $L=4$, $\alpha=0.02$, and $\eta=2$.

In this section, we gauge to assess the impact of $\eta$ on the quality of the solutions sampled for HiSS. By employing $\alpha=10^{-4}$, $G=10$, and $L=4$, we plot the average PMC and Jaccard Similarity metrics, along with their respective standard deviations, as $\eta$ is progressively increased.
 \begin{figure}
    \centering
    \includegraphics[width=0.4\textwidth]{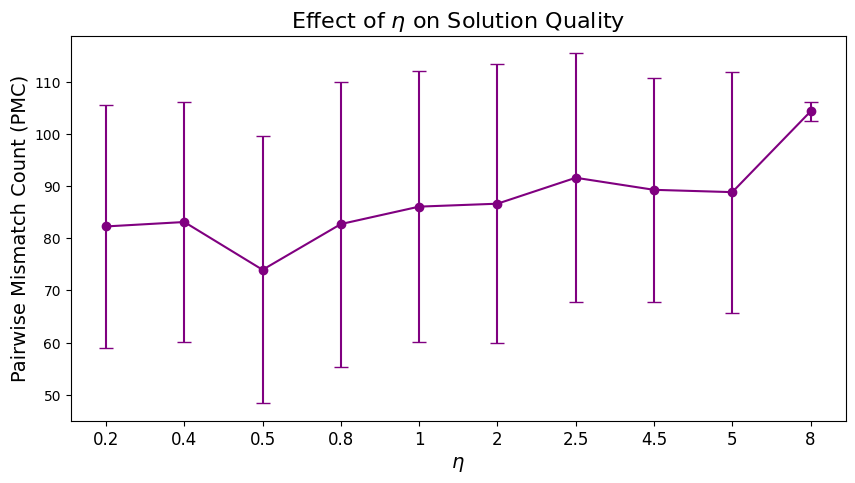}
\includegraphics[width=0.4\textwidth]{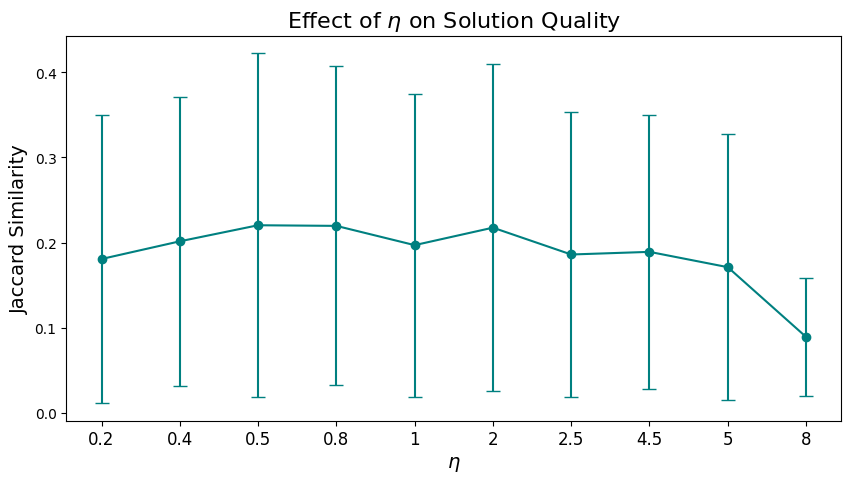}
    \caption{Impact of scale of logistic noise on solution quality.}
    \label{fig:optimalsol}
\end{figure} 

As evident from Figure \ref{fig:optimalsol}, the sample diversity improves as $\eta$ increases. This observation aligns with logical intuition, as larger $\eta$ enables the sampler to explore the state space more effectively.

 \subsection{Binary Bayesian Neural Networks\label{bnnn}}
For Section \ref{sec:exp}, for HiSS, we set the parameters $G=10$, $L=5$, and $\alpha=0.1$. As shown in Table \ref{tab:bayesbnn1}, for the smaller Breast Cancer dataset (with $\eta=0.005$), DMALA outperforms HiSS in terms of average training log-likelihood and RMSE. However, on the larger COMPAS dataset (with $\eta=2$),HIV dataset($\eta=4$), and Blog($\eta=5$) HiSS achieves the lowest training RMSE. For each dataset, $\eta$ values are chosen based on validation RMSE.

 \begin{table*}[ht]
    \centering
        \caption{Experiment results with binary Bayesian neural networks on different training datasets.}
    \resizebox{0.99\textwidth}{!}{
    \renewcommand\arraystretch{1.2}{
    \begin{tabular}{c|ccccc|ccccc}
    \hline \hline
        \multirow{2}{*}{Dataset}&  \multicolumn{5}{c|}{Average Training Log-likelihood ($\uparrow$)} & \multicolumn{5}{c}{Average Training Root-Mean Square Error ($\downarrow$)} \\ \cline{2-11}
         & GWG & DMALA  & ACS & PT+DMALA & HiSS & GWG & DMALA & ACS & PT+DMALA & HiSS \\ \hline
        Breast Cancer & -0.0708 \scriptsize{±0.0089} & \textbf{-0.0693} \scriptsize{±0.0025} & {-0.0938} \scriptsize{±0.0060} & -0.0721 \scriptsize{±0.0049} & \text{-0.0728} \scriptsize{±0.0028} & 0.0283 \scriptsize{±0.0032} & \textbf{0.0276} \scriptsize{±0.0011} & {0.0315} \scriptsize{±0.0018} & 0.0287 \scriptsize{±0.002} & \text{0.0291} \scriptsize{±0.0013} \\  
        COMPAS & -0.3130 \scriptsize{± 0.0068} & \textbf{-0.3121} \scriptsize{± 0.0027} & -0.3139 \scriptsize{± 0.0037} & {-0.3149} \scriptsize{± 0.0057} & \text{ -0.3697} \scriptsize{± 0.0073} & 0.2213 \scriptsize{±0.0019} & 0.2219 \scriptsize{±0.0020} & 0.2230 \scriptsize{±0.0020} & {0.2215} \scriptsize{± 0.0020} & \textbf{0.2172} \scriptsize{± 0.0024} \\
       HIV & -0.7453 \scriptsize{± 0.0750} 
    & -0.7746 \scriptsize{± 0.0000} 
    & -0.7746 \scriptsize{± 0.0000} 
    & -0.7746 \scriptsize{± 0.0000} 
    & \textbf{-0.4578} \scriptsize{± 0.0101} 
    & 0.7299 \scriptsize{± 0.1205} 
    & 0.7746 \scriptsize{± 0.0000} 
    & 0.7746 \scriptsize{± 0.0000} 
    & 0.7746 \scriptsize{± 0.0000} 
    & \textbf{0.2554} \scriptsize{± 0.0100} \\
    Blog & \textbf{-0.3458} \scriptsize{± 0.0127} 
    & -0.3476 \scriptsize{± 0.0000} 
    & -0.3476 \scriptsize{± 0.0000} 
    & -0.3476 \scriptsize{± 0.0000} 
    & {-0.4094} \scriptsize{± 0.0092} 
    &   0.3320\scriptsize{± 0.0453} 
    & 0.3476 \scriptsize{± 0.0000} 
    & 0.3476 \scriptsize{± 0.0000} 
    & 0.3476 \scriptsize{± 0.0000} 
    & \textbf{0.2043} \scriptsize{± 0.0088} \\
    \hline \hline
    \end{tabular}}}
    \label{tab:bayesbnn1}
\end{table*}

{Dataset Details}
\begin{itemize}
    \item \textbf{Breast Cancer}~\citep{breastcancer}: This dataset contains 569 instances of digitized fine needle aspirates (FNAs) of breast masses. The task involves predicting whether the instance is benign or malignant. For prediction, we use 30 real-valued attributes, and the dimensionality of sampling vector is 3,201.
    \item \textbf{COMPAS}~\citep{compas}: This dataset includes criminal records of 6,172 individuals arrested in Florida. The task is to predict whether an individual will re-offend within two years. We utilize 13 attributes for prediction. The dimensionality of sampling vector is 1,501.
    \item \textbf{HIV}~\citep{hiv} : This dataset contains 1,625 instances of octamers (8-amino-acid sequences). The binary classification task is to predict whether a sequence is a cleavage site for the HIV-1 protease enzyme. The input sequences are transformed into features via one-hot encoding across 20 standard amino acids, resulting in a 160-dimensional binary feature vector for each instance. This makes the dimensionality of the sampling vector 16,201. 
    \item \textbf{Blog}~\citep{blog} : This dataset containing 54,270 data points from blog posts. The raw HTML-documents of the blog posts were crawled and processed. The prediction task associated with the data is the prediction of the number of comments in the upcoming 24 hours. This makes the dimensionality of the sampling vector 28,201. 
\end{itemize}

To create a challenging, disconnected posterior landscape characterized by isolated modes, we introduce a sparsity inducing prior on the network weights for HIV and Blog Datasets. This approach is conceptually motivated by the classic Spike-and-Slab framework for Bayesian variable selection \citep{George01091993, koyejo2014prior}. In practice, we implement this as a Laplace prior, which is the Bayesian equivalent of the well-known L1/Lasso penalty \citep{tibshirani1996regression,park2008bayesian}. Applying such priors to encourage sparsity and prune connections is a highly active area of research, with recent applications to both continuous and binarized neural networks \citep{louizos2018learning}. This prior forces the BNN to find solutions where most weights are in a default \emph{off} state, creating deep energy wells at sparse configurations and high energy barriers between them.

\end{document}